\documentclass[11pt, a4paper, copyright, gdm]{google}
\usepackage[utf8]{inputenc} % allow utf-8 input
\usepackage[T1]{fontenc}    % use 8-bit T1 fonts
\usepackage{url}            % simple URL typesetting
\usepackage{booktabs}       % professional-quality tables
\usepackage{nicefrac}       % compact symbols for 1/2, etc.
\usepackage{microtype}      % microtypography
\usepackage{xcolor}         % colors
\usepackage{amsmath, amssymb, amsthm, amsfonts}
\usepackage{mathtools}
\usepackage{bm}
\usepackage[skins,breakable,theorems]{tcolorbox}
\usepackage{booktabs}
\usepackage[all]{nowidow}
\usepackage{subcaption}
\hypersetup{
    breaklinks=true,
    citecolor=cyan
}
\usepackage{titletoc} % Table of Contents for Appendix
\usepackage{algorithm}
\usepackage{algpseudocode}
\usepackage{xspace}
\usepackage{pgfplots}
\usepgfplotslibrary{groupplots}
\pgfplotsset{compat=1.18}
\usepackage{listings}
\definecolor{codegreen}{rgb}{0,0.6,0}
\definecolor{codegray}{rgb}{0.5,0.5,0.5}
\definecolor{codepurple}{rgb}{0.58,0,0.82}
\definecolor{backcolour}{rgb}{0.95,0.95,0.92}
\definecolor{colorHO}{HTML}{1F618D}   % Deep Royal Blue
\definecolor{colorFFT}{HTML}{B03A2E}  % Deep Crimson
\definecolor{colorDEFT}{HTML}{1E8449} % Emerald Green
\definecolor{colorBase}{HTML}{7F8C8D} % Slate Gray

\definecolor{coreColor}{HTML}{1E88E5}   % Vibrant Blue for Core (Stable/Foundational)
\definecolor{slackColor}{HTML}{00B159}  % Emerald Green for Slack (Plastic/Active)
\definecolor{danger}{HTML}{E53935}      % Crimson for Constraints/Cuts
\definecolor{layerBgColor}{HTML}{FFF9C4} % Soft Pastel Yellow for Layer Backgrounds

\usepackage{tikz}
\usetikzlibrary{shapes, shapes.misc, shapes.geometric,  arrows.meta, positioning, fit, calc, decorations.pathreplacing, shadows, backgrounds, patterns}
\usetikzlibrary{positioning, fit, backgrounds, decorations.pathreplacing, arrows.meta, calc}

\newcommand{\norm}[1]{\left\lVert#1\right\rVert}
\newcommand{\normshort}[1]{\lVert#1\rVert}
\newcommand{\abs}[1]{\left\lvert#1\right\rvert}
\newcommand{\inner}[2]{\left\langle#1, #2\right\rangle}

\newcommand{\bVect}{\boldsymbol{b}}

\newcommand{\eVect}{\boldsymbol{e}}
\newcommand{\gVect}{\boldsymbol{g}}

\newcommand{\sVect}{\boldsymbol{s}}
\newcommand{\uVect}{\boldsymbol{u}}
\newcommand{\vVect}{\boldsymbol{v}}
\newcommand{\wVect}{\boldsymbol{w}}
\newcommand{\xVect}{\boldsymbol{x}}
\newcommand{\yVect}{\boldsymbol{y}}
\newcommand{\zVect}{\boldsymbol{z}}
\newcommand{\zeroVect}{\boldsymbol{0}}

\newcommand{\muVect}{\boldsymbol{\mu}}
\newcommand{\sigmaVect}{\boldsymbol{\sigma}}
\newcommand{\betaVect}{\boldsymbol{\beta}}
\newcommand{\gammaVect}{\boldsymbol{\gamma}}
\newcommand{\thetaVect}{\boldsymbol{\theta}}
\newcommand{\calA}{\mathcal{A}}
\newcommand{\calB}{\mathcal{B}}
\newcommand{\calC}{\mathcal{C}}
\newcommand{\calD}{\mathcal{D}}
\newcommand{\calF}{\mathcal{F}}
\newcommand{\calH}{\mathcal{H}}
\newcommand{\calI}{\mathcal{I}}
\newcommand{\calJ}{\mathcal{J}}
\newcommand{\calL}{\mathcal{L}}
\newcommand{\calM}{\mathcal{M}}
\newcommand{\calN}{\mathcal{N}}
\newcommand{\calO}{\mathcal{O}}
\newcommand{\calS}{\mathcal{S}}

\newcommand{\calV}{\mathcal{V}}
\newcommand{\calX}{\mathcal{X}}
\newcommand{\R}{\mathbb{R}}
\newcommand{\E}{\mathbb{E}}
\newcommand{\normPDF}{\phi}
\newcommand{\normCDF}{\Phi}
\newcommand{\Amat}{\boldsymbol{A}}

\newcommand{\Cmat}{\boldsymbol{C}}
\newcommand{\Emat}{\boldsymbol{E}}

\newcommand{\Gmat}{\boldsymbol{G}}

\newcommand{\Imat}{\boldsymbol{I}}

\newcommand{\Kmat}{\boldsymbol{K}}
\newcommand{\Mmat}{\boldsymbol{M}}
\newcommand{\Wmat}{\boldsymbol{W}}

\newcommand{\Wall}{\boldsymbol{W}_{\text{joint}}}
\newcommand{\Rmat}{\boldsymbol{R}}

\newcommand{\Lambdamat}{\boldsymbol{\Lambda}}
\newcommand{\ReLU}{\texttt{ReLU}\xspace}

\DeclareMathOperator{\Var}{Var}
\DeclareMathOperator{\Cov}{Cov}
\DeclareMathOperator{\sign}{sign}

\DeclareMathOperator{\argmin}{argmin}
\DeclareMathOperator{\argmax}{argmax}
\DeclareMathOperator{\diag}{diag}

\newcommand{\emc}[1]{{\textbf{\color[rgb]{0,.3,.6}#1}}}
\theoremstyle{definition}
\newtheorem{definition}{Definition}[section]
\newtheorem{theorem}{Theorem}[section]
\newtheorem{axiom}{Axiom}
\newtheorem{lemma}[theorem]{Lemma}
\newtheorem{proposition}[theorem]{Proposition}
\newtheorem{corollary}[theorem]{Corollary}
\newtheorem{assumption}{Assumption}

\tcolorboxenvironment{theorem}{
    colback=yellow!5!white,
    colframe=yellow!50!black,
    boxrule=0.5mm,
    arc=2mm
}

\tcolorboxenvironment{lemma}{
    colback=yellow!5!white,
    colframe=yellow!50!black,
    boxrule=0.5mm,
    arc=2mm
}

\tcolorboxenvironment{proposition}{
    colback=yellow!5!white,
    colframe=yellow!50!black,
    boxrule=0.5mm,
    arc=2mm
}

\tcolorboxenvironment{corollary}{
    colback=yellow!5!white,
    colframe=yellow!50!black,
    boxrule=0.5mm,
    arc=2mm
}

\tcolorboxenvironment{assumption}{
    colback=blue!5!white,
    colframe=blue!50!black,
    boxrule=0.5mm,
    arc=2mm
}

\tcolorboxenvironment{axiom}{
    colback=red!5!white,
    colframe=red!50!black,
    boxrule=0.5mm,
    arc=2mm
}

\newtcolorbox{insight}[2][]{
    colback=green!5!white,
    colframe=green!50!black,
    title=\textbf{#2},
    #1
}

\newtcolorbox{alertbox}[2][]{
    colback=yellow!5!white,
    colframe=red!50!black,
    title=\textbf{#2},
    #1
}

\uselogo{} 

\title{Hilbert Operator for Progressive Encoding (HOPE)\\[1ex]
\normalfont\fontsize{14}{18}\selectfont A Mathematical Framework for Deconstructing Learned Representations in Deep Networks}

\correspondingauthor{hmobahi@google.com, peterbartlett@google.com}
\reportnumber{}
\author[1]{Hossein Mobahi}
\author[1,2]{Peter L. Bartlett}
\affil[1]{\thepa{}{}}
\affil[2]{University of California, Berkeley}

\begin{abstract}
Deep neural networks encode complex representations, but deconstructing this internal knowledge remains a challenge. Given the link between learning and compression, network compression offers a promising lens to analyze this knowledge. However, standard compression heuristics often suffer from scale symmetries and architectural biases. To resolve these, we introduce Hilbert Operator for Progressive Encoding (HOPE), a mathematical framework to gradually deconstruct the representations in trained network weights.

HOPE shifts network compression from the discrete domain into a Hilbert space of continuous functions. By modeling individual neurons as rank-1 Hilbert-Schmidt operators, HOPE unifies pruning and neuron merging as low-rank subspace projection. Extending this formulation, HOPE introduces macro block eviction to encompass multi-layer structures like entire residual pathways under the same unified metric. This unified approach enables unbiased architectural decisions across layers with different types and sizes. HOPE is a data-free and hyperparameter-free framework. We present proof-of-concept experiments in model compression and fine-tuning to highlight the practical potential of our theory.
\end{abstract}

\begin{document}

\maketitle

\section{Introduction}
\label{sec:intro}
While deep neural networks learn complex representations, deconstructing this knowledge from numerical weights remains challenging. In this work, we use model compression as a measurable proxy task to study these internal representations objectively. Given the fundamental link between compression and learning \cite{rissanen1978modeling, hinton1993keeping}, which has recently been reinforced by demonstrating that LLMs are general-purpose compressors \cite{Deletang2023} and amortized algorithmic predictors \cite{Genewein2026}, we believe compression provides a promising lens to study this issue objectively. Deconstructing opaque networks through capacity reduction has foundational roots \cite{mozer1989skeletonization}. Viewed through a modern information-theoretic lens, learning is essentially the systematic discarding of task-irrelevant \textbf{noise} to isolate generalizable \textbf{core} patterns \cite{tishby1999information, shwartz2017opening}.

Consequently, we posit that \textit{progressive compression} is an effective tool for achieving this separation. Because core invariants resist pruning significantly longer than malleable slack, iteratively reducing capacity with minimal distortion naturally peels away the periphery to expose the network's universal feature space \cite{he2026demystifyingpruningworksrepresentation, nguyen2026depth, wang2026understanding}. This post-training deconstruction mirrors the network's learning dynamics: the ``coarse-graining'' dimensionality reduction that gradient descent originally used to build them \cite{dandi2025computational}. Indeed, theoretical analyses confirm that networks learn by incrementally adding effective units to model increasingly complex functions \cite{Zhangetal2025}. While biologically inspired paradigms like \cite{behrouz2026languagemodelsneedsleep} achieve this core-slack segregation by expanding a network to consolidate memories, we demonstrate that this separation can be accomplished efficiently through data-free compression.

Despite compression's promise, the opacity of deep networks poses hurdles for objectively identifying and reducing capacity. Simple heuristics such as magnitude-based pruning (e.g. norm of raw weights) often fail because these magnitudes are typically optimization artifacts rather than importance indicators \cite{Scholl2021, Tanaka2020PruningNN, hooker2019what}, as also long observed in network sparsification \cite{lecun1989optimal, hassibi1993optimal, frankle2019lottery}.

To address the shortcomings of magnitude based notions of capacity, we propose to transition from the physical parameter space to the function space, specifically the underlying transformation a neuron applies to its input. This perspective treats the entire neuron, rather than individual weight matrices, as the atomic unit of the network. Note that this transformation must be analyzed over the relevant data manifold rather than the entire input space. For instance, characterizing a neuron's behavior within a dog-versus-cat image classification task requires isolating the data manifold specific to those two classes, while discarding irrelevant classes and non-natural images. While one might attempt to approximate the input manifold to each neuron empirically by passing a finite dataset through the network, such finite-dimensional approximations tether the evaluation to specific samples. This reliance on explicit data sets can render the resulting architecture brittle to distribution shifts and disproportionately degrade performance on long-tail features \cite{hooker2019what}. Furthermore, empirical approaches that rely on continuous activation matching or curvature approximations \cite{molchanov2016pruning, luo2017thinet} incur a severe computational penalty, as they require numerically evaluating neural activations across the dataset \textit{repeatedly during every iteration} of the progressive model encoding.

To avoid both the drawbacks of weight-based parameterization and the pitfalls of empirical data dependence, we introduce the Hilbert Operator for Progressive Encoding (HOPE) framework. The core philosophy of HOPE is that evaluation and analysis of neurons must occur in the function space, the resulting compression must be executed discretely on the network's parameters. By lifting the parameter vectors and matrices of individual neurons into continuous functions, HOPE models each neuron as a rank-1 Hilbert-Schmidt operator. This abstraction unifies pruning and merging under a single theoretical paradigm: an optimal low-rank projection within a functional tensor space, where distances are measured by $\norm{ \cdot}_{\calH}$. Once the optimal reductions (pruning or merging) are computed in this pure functional space, the framework projects the resulting continuous operators back onto the network parameters to execute the compression.

A key advantage of HOPE is that it operates entirely \textbf{data-free} for networks utilizing Batch Normalization (BN).\footnote{This data-free approach relies on global BN statistics. For architectures lacking BN, HOPE easily adapts: it requires only a simple, \textbf{one-time} calibration pass over a small data batch to capture the necessary pre-activation statistics.} While prior research has successfully exploited BN moving averages to generate synthetic images for data-free knowledge distillation \cite{yin2020dreaming, micaelli2019zero}, HOPE leverages this information in a different manner. Instead of generating synthetic spatial data or relying on massive real datasets to drive input signals, HOPE applies the Maximum Entropy principle directly to the empirical BN statistics embedded within the trained model checkpoint. This yields a continuous surrogate for the local input distribution at each neuron, allowing the framework to \textbf{analytically} evaluate the integrals required for Hilbert space norms $\norm{ \cdot}_{\calH}$. This fully continuous, analytical perspective largely mitigates scaling symmetries inherent in raw parameters, and thus enables unbiased capacity measurements across heterogeneous layers without requiring a single real or synthetic data sample.

The primary value of HOPE lies in establishing a rigorous, \textbf{hyperparameter free} (and also data free in presence of BN) mathematical framework for the progressive encoding of trained deep networks. We present proof-of-concept applications for model compression and fine-tuning experiments to empirically validate these theoretical capabilities, rather than to establish exhaustive, large-scale benchmarks.

\paragraph{Paper Organization.} We begin with a literature review in Section~\ref{sec:related_works} and formally define a neuron within our framework in Section~\ref{sec:neuron}. To enable data-free evaluation, Section~\ref{sec:neural_signal_distribution} constructs a surrogate distribution constrained by Batch Normalization (BN) statistics. In Section~\ref{sec:capacity}, we lift discrete neurons into continuous Hilbert-Schmidt operators. Section~\ref{sec:layer_transition_costs} then derives a cost functional $\calJ$ that captures the distortion induced by pruning or merging via subspace projection. Building on this, Section~\ref{sec:parent_neuron} derives the optimal parent neuron in function space for a given pair of merged neurons, detailing how this continuous representation is mapped back to the network's discrete parameters. Section~\ref{sec:block_eviction} extends this projection metric to the macro level to evict larger residual blocks. Section~\ref{sec:rate_distortion} introduces a rate-distortion-inspired objective that balances the distortion cost $\calJ$ against the resulting reduction in parameter count; this objective is used to greedily select the next optimal compression action. Section~\ref{sec:progressive_encoder} presents the encoding process as a loop over greedy selection of compression actions. Finally, Section~\ref{sec:experiments} provides a proof-of-concept evaluation of our framework for model compression and fine-tuning.

\section{Related Works}
\label{sec:related_works}
\noindent \textbf{Pruning and Parameter-Space Methods.} Network compression has historically relied on parameter-space pruning, evolving from early Taylor-expansion techniques \cite{lecun1989optimal, hassibi1993optimal} to modern unstructured \cite{han2015learning, frankle2019lottery} and structured \cite{he2017channel, wen2016learning} approaches. HOPE directly compares against structured baselines utilizing $L_1$-norms \cite{li2016pruning} and BN scaling \cite{liu2017learning}. To capture functional importance, methods ranging from early skeletonization via error derivatives \cite{mozer1989skeletonization} to modern empirical activation tracking \cite{molchanov2016pruning, molchanov2019importance, singh2020woodfisher, luo2017thinet} rely on dataset passes; however, this introduces computational bottlenecks and brittleness to distribution shifts \cite{hooker2019what}. Conversely, parameter-centric heuristics suffer from ``scale symmetry'' \cite{blalock2020state, renda2020comparing, badrinarayanan2015understanding, dinh2017sharp}, where overparameterized many-to-one mappings \cite{neyshabur2015search} mean raw magnitudes often reflect optimization artifacts rather than true importance \cite{Scholl2021, Tanaka2020PruningNN, hooker2019what}. Although methods like \cite{lee2021layer} can partially mitigate global scale symmetries, they remain bound to parameter-space magnitude heuristics, which are vulnerable to within-neuron scaling artifacts. By migrating evaluation to a continuous Hilbert space, HOPE circumvents both empirical data bottlenecks and parameter-space artifacts. This shift from structural to functional mappings aligns with the Platonic Representation Hypothesis \cite{huh2024platonic}, which posits that diverse networks converge to a shared statistical model of reality, treating weight spaces as mere shadows. HOPE operationalizes this intra-model, defining identity via continuous operators rather than superficial parameters.

\noindent \textbf{Neuron Alignment and Model Merging.} Some recent methods on merging networks include permutation invariances \cite{entezari2021role, ainsworth2022git}, Optimal Transport \cite{singh2020model}, alignment strategies \cite{tatro2020optimizing}, or feature zipping \cite{stoica2023zipit}. Other approaches perform data-free neuron merging by evaluating pairwise parameter similarities \cite{SrinivasBabu2015} or clustering weights via adaptive scalar hashing \cite{Yvinecetal2021}. However, these methods typically rely on combinatorial matching or parameter averaging. HOPE advances this paradigm by unifying pruning and merging under a single continuous operation: optimal low-rank projection in a functional tensor space, penalizing distortion via the Hilbert-Schmidt norm. This geometric approach mirrors using subspace embeddings to capture hierarchical and compositional representations \cite{Moreira2025}. Merging acts as the inverse to ``feature splitting'' \cite{bricken2023towards} or computation in superposition \cite{Hanni2024}, where overparameterized models fragment concepts across correlated sub-features, a phenomenon rooted in foundational vector space arithmetic \cite{mikolov2013efficient, pennington2014glove} and the Linear Representation Hypothesis \cite{park2023linear, engels2024notall}. By projecting a rank-2 neuron pair into an optimal rank-1 Hilbert-Schmidt parent, HOPE reconsolidates this distributed knowledge, reclaiming capacity while preserving the network's underlying linear geometry.

\noindent \textbf{Macro Architecture and Layer Pruning.} To reduce network's depth, standard methods employ stochastic depth \cite{huang2016deep}, heuristic layer dropping \cite{fan2019reducing}, Neural Architecture Search \cite{zoph2016neural, liu2018darts}, or dynamic routing gates \cite{veit2018convolutional, wang2018skipnet}. However, these approaches decouple macro-architectural decisions from granular feature selection, relying on separate optimization phases or custom hyperparameters. HOPE unifies these scales by formalizing block eviction as a macroscopic function subspace projection. Mapping this architectural deletion to the identical capacity cost $\calJ$ allows macro-reductions to compete against granular pruning and merging within a single, hyperparameter-free decision engine.

\noindent \textbf{Data-Free Compression and Maximum Entropy Surrogates.} Bypassing the original training data often involves inverting BN statistics to generate synthetic inputs \cite{lopes2017data, cai2020zeroq, nagel2019data, micaelli2019zero, yin2020dreaming} or conserving discrete synaptic flow at initialization \cite{Tanaka2020PruningNN}. Rather than generating explicit samples, HOPE elevates these empirical statistics via the Maximum Entropy principle \cite{jaynes1957information} to construct a continuous analytical surrogate, similar to some task-agnostic pruning of LLMs \cite{Ma2023}. Paralleling theoretical analyses of infinite-width networks and Gaussian Processes \cite{neal1996bayesian, lee2018deep, jacot2018neural, yang2019tensor}, this framework allows infinite-dimensional integrals to be resolved without a single forward pass.

\noindent \textbf{Parameter Plasticity, Continual Learning, and Representation Deconstruction.}
HOPE conceptually bridges progressive compression with transfer learning \cite{yosinski2014how, hu2022lora, houlsby2019parameter} and the Stability-Plasticity dilemma of continual learning \cite{grossberg1987competitive, kirkpatrick2017overcoming, zenke2017continual}. Drawing parallels to Complementary Learning Systems in cognitive neuroscience \cite{mcclelland1995why, kumaran2016what}, and aligning with the growing recognition across deep learning \cite{kong2026reasoning}, we hypothesize that the learned representation must be explicitly segregated into a universal core of invariants and a peripheral slack of malleable volume in order to allow learning without forgetting. While previous frameworks attempt to protect foundational knowledge against representational drift using computationally massive $\calO(N^3)$ orthogonal matrix projections \cite{saha2021gradient, zeng2019continuous, yang2025loranull, huggingface2026osf} or dataset-dependent quadratic penalties \cite{kirkpatrick2017overcoming}, HOPE provides a purely data-free alternative. By computing neuron capacity in $\calO(N)$ time, its progressive pruning and merging peel away slack and expose core abstractions, computationally mirroring how awake biological circuits rapidly decorrelate co-activated neurons to maintain network stability \cite{andrei2023rapid}. This perspective is corroborated by recent literature utilizing capacity reduction to deconstruct hierarchies. For example, targeted parameter removal has been used to explain robustness \cite{he2026demystifyingpruningworksrepresentation}, uncover heavy-tailed synaptic backbones \cite{nguyen2026depth}, and demonstrate hierarchical learning through progressive feature compression \cite{wang2026understanding}. Collectively, these works support the premise that HOPE's neuron capacity may efficiently identify foundations versus plastic slack, laying the groundwork for transfer learning and downstream algorithmic interpretability \cite{bau2017network, morcos2018on, olah2020zoom, neyshabur2017exploring}.

\section{The Neuron}
\label{sec:neuron}
Scale invariances are inherently relational and that is why isolated parameters fail to capture them. The minimal architectural unit where these symmetries fully manifest is the neuron in its entirety (encapsulating its incoming weights, BN parameters, non-linear activation, and outgoing weights). In this section we discuss how these within-neuron scale symmetries can be factored out, while deferring global scale symmetries to Section~\ref{sec:capacity}. For clarity, we derive our theoretical framework under the assumption of a fully connected architecture. However, this formalism seamlessly extends to convolutional networks (see Appendix~\ref{sec:conv_nets}). Indeed, the architectures evaluated in Section~\ref{sec:experiments} rely on this adaptation.

Consider a neuron $i$ with input weights $\wVect_{\text{raw},i} \in \R^n$ and output weights $\wVect_{\text{out},i} \in \R^c$. Let the neuron be subject to learnable affine BN parameters $(\gamma_i, \beta_i)$ and empirical moving dataset statistics $(\mu_i, \sigma_i^2)$, where $\mu_i \triangleq \E_{\calX}[\wVect_{\text{raw},i}^T \xVect]$ and $\sigma_i^2 \triangleq \Var_{\calX}(\wVect_{\text{raw},i}^T \xVect)$. To capture the signal reaching the non-linearity, we absorb the normalization operations into a set of \textit{effective} parameters:
\begin{equation}
\label{eq:eff_params}
\wVect_{\text{in},i}^{\text{eff}} \,\triangleq\, \gamma_i/\sqrt{\sigma_i^2 + \epsilon} \wVect_{\text{raw},i} \qquad , \qquad b_i \,\triangleq\, \beta_i - \gamma_i \mu_i/\sqrt{\sigma_i^2 + \epsilon} \,,
\end{equation}
where $\epsilon > 0$ is a small numerical stability constant. The neuron's end-to-end signal mapping, capturing its functional contribution to the subsequent layer for an input $\xVect$, is defined by the continuous function:
\begin{equation}
f_i(\xVect) = \Psi(y_i) \wVect_{\text{out},i} \qquad, \qquad y_i \triangleq  (\wVect_{\text{in},i}^{\text{eff}})^T \xVect + b_i \,,
\end{equation}
where $\Psi(\cdot)$ denotes a PH-1 activation function defined as below:
\begin{insight}{Positively Homogeneous of degree 1 (PH-1) Functions}
An activation function $\Psi: \R \rightarrow \R$ is Positively Homogeneous of degree 1 (PH-1) if it satisfies the scaling property $\Psi(c z) = c \Psi(z)$ for all $z \in \R$ and all scalars $c \geq 0$. Examples include \ReLU, \texttt{Leaky ReLU}, \texttt{PReLU}, and the linear functions.
\end{insight}

Throughout the HOPE framework, we formally designate this continuous function $f_i$ as the \textit{neuron}. Operating at this atomic level mitigates two primary sources of scaling symmetry: \textit{normalization invariance} (from BN) and \textit{re-parameterization invariance} (from cross-layer weight resharding). We detail the former below, and defer the latter to Section~\ref{sec:capacity}.

\subsection{Mitigating Normalization Invariance}
Normalization invariance arises from BN's standardizing mechanics. Scaling raw input weights $\wVect_{\text{raw},i}$ by a constant factor $\lambda > 0$ increases the pre-activation variance by $\lambda^2$. The subsequent BN layer divides by the standard deviation, canceling $\lambda$ before the non-linearity. Because downstream output remains unchanged, raw weight magnitudes can be deceptive.

HOPE mitigates this failure mode by evaluating capacity through effective parameters ($\wVect_{\text{in},i}^{\text{eff}}$ and $b_i$). While HOPE uses raw weights $\wVect_{\text{raw},i}$ alongside BN statistics to construct the data-constrained surrogate dataset $P_{\calX}$ (Section~\ref{sec:neural_signal_distribution}), their utility ends there. The framework then evaluates the neuron's continuous-functional impact on this surrogate.

Physically, the non-linear activation processes the normalized, shifted signal, not the raw parameter projection. Thus, when computing the continuous integral to compute Hilbert space norms or inner product, HOPE defines the pre-activation signal as $y_i = (\wVect_{\text{in},i}^{\text{eff}})^T \xVect + b_i$. Because $\wVect_{\text{in},i}^{\text{eff}}$ divides by the empirical standard deviation $\sigma_i$, any magnitude inflation is canceled. By evaluating the signal impacting the activation function, HOPE guarantees the capacity criterion reflects functional utility rather than scale artifacts.

\section{The Neural Signal Distribution}
\label{sec:neural_signal_distribution}
Since HOPE operates in a data-free regime, the true input distribution $P^*_{\calX}$ (encompassing both initial data and subsequent layer activations) is inaccessible. However, because we model neurons as continuous Hilbert-Schmidt operators, evaluating their inner products requires integrating over the data distribution. To resolve this, we invoke the Maximum Entropy principle \cite{jaynes1957information} to construct a Gaussian surrogate constrained by BN statistics. While a Gaussian approximation may seem overly idealistic, we explain below in two steps why it aligns closely with modern neural architectures.

\noindent \textbf{Step 1: Gaussianity of Pre-Activation.} While a neuron's true input distribution $P^*_{\calX}$ often lies on a complex, highly non-Gaussian manifold, each neurons observes its input $\xVect$ only through 1-dimensional linear projections $y = \sum_{j=1}^{n} w_j x_j$. As the fan-in dimension $n$ grows, by the Central Limit Theorem and the Diaconis-Freedman effect \cite{diaconis1984asymptotics}, these aggregated signals converge to a Gaussian distribution. Consequently, neurons remain oblivious to the complex data manifold; from their perspective the pre-activation $y_i$ is Gaussian. Although nonlinear activations (e.g., ReLU) disrupt this Gaussianity, subsequent high-dimensional linear transformations recursively smooth the signals back into Gaussian pre-activations across layers.

\noindent \textbf{Step 2: Gaussian Surrogate for the Input $\xVect$.}
While neurons are oblivious to the true shape of the data manifold and perceive their input as a Gaussian signal $y$. Based on this observation, for theoretical convenience, we substitute the complex true data with a tractable surrogate. For architectural consistency, the surrogate distribution must satisfy the same observational bottleneck: its 1D linear projections must remain Gaussian. By definition, if every linear combination of a random vector is Gaussian, the vector itself must be multivariate Gaussian. Thus, to construct a surrogate distribution aligned with a world where every linear observer (neuron) sees a Gaussian, that surrogate is necessarily a multivariate Gaussian, $P_{\calX} = \calN(\hat{\muVect}_x, \hat{\Sigma}_x)$. 

\begin{alertbox}{Architectural Context}
Throughout this work, we contextualize our framework within standard modern vision architectures, specifically the ResNet-50 (V1) architecture. The canonical computational block follows the sequence: convolution, followed by BN, followed by a ReLU activation $\text{Conv} \to \text{BN} \to \text{ReLU}$. Consequently, the input vector $\xVect$ presented to any internal convolutional layer is the output of a preceding ReLU activation.
\end{alertbox}

\begin{insight}{Post-ReLU Support Paradox}
Using a multivariate Gaussian surrogate $P_{\calX}$ whose support is the entire $\R^n$ might seem problematic, given that post-ReLU inputs are non-negative $\xVect \ge \zeroVect$. However, the purpose of $P_{\calX}$ is to model the inner product $\inner{f_i}{f_j}_{\calH}$, not the true input distribution. Because $\inner{f_i}{f_j}_{\calH} = \E_{\xVect \sim P_{\calX}} [\Psi(y_i) \Psi(y_j)] \cdot \inner{\wVect_{\text{out},i}}{\wVect_{\text{out},j}}_{\R^c}$, this integral reduces to a 2D subspace defined by the pre-activations $y_i$ and $y_j$. In high dimensions, the Central Limit Theorem and the Diaconis-Freedman effect ensure that projecting high-dimensional vectors into a low-dimensional subspace rapidly converges to a bivariate Gaussian distribution. Since the integration depends solely on this 2D slice, relaxing the ambient non-negativity constraint yields a tractable surrogate while maintaining an accurate asymptotic approximation of the bivariate distribution.
\end{insight}

\noindent \textbf{Optimal Surrogate.} To define the optimal parameters $(\hat{\muVect}_x, \hat{\Sigma}_x)$ of the unknown data distribution, we can incorporate the two empirical constraints $\xVect \in \R^n$: $\E[\wVect_{\text{raw}, i}^T \xVect] = \mu_i$ and $\Var(\wVect_{\text{raw}, i}^T \xVect) = \sigma_i^2$ imposed by BN for any $i$. We define a shared surrogate for each layer. Since the empirical means $\muVect_{\text{BN}} \in \R^n$ represent the 1D shadow of the dataset's center cast through the raw weights, to find the best mean for a layer we compute the optimal least-squares approximation of the dataset's center using the Moore-Penrose pseudo-inverse $\Wmat_{\text{raw}}^+$, which leads to $\hat{\muVect}_x = \Wmat_{\text{raw}}^+ \muVect_{\text{BN}}$. In underdetermined scenarios (e.g., when compressing layers where $n > c$), this pseudo-inverse yields the minimum-norm solution, which sets the unobserved orthogonal components of the data manifold to a zero mean. To find the optimal covariance matrix $\hat{\Sigma}_x$ for the layer, we maximize the differential entropy of the multivariate Gaussian, $H(\xVect) \propto \log \det(\Sigma_x)$ subject to the BN variance constraints $\wVect_{\text{raw}, i}^T \Sigma_x \wVect_{\text{raw}, i} = \sigma_i^2$ for $i \in \{1, \dots, c\}$. 
\textit{Conceptually}, applying Lagrange multipliers yields $\hat{\Sigma}_x = \left( \sum_{i=1}^c \lambda_i \wVect_{\text{raw}, i} \wVect_{\text{raw}, i}^T \right)^{-1}$, where $\lambda_i$ are optimized to satisfy variance equality constraints. However, inverting this covariance matrix is computationally expensive. Fortunately, the framework bypasses the computationally expensive need to compute and invert the full high-dimensional joint covariance matrix entirely: micro-operations evaluate pairwise merges restricted to a rank-2 subspace, allowing a closed-form solution via a pairwise neural kernel (Appendix~\ref{sec:cross-kernel}), while macro block eviction evaluates destruction using the $L_1$ cumulative distance of surviving capacities (Section~\ref{sec:block_eviction}).

\section{A Hilbert Functional Perspective on Neurons}
\label{sec:capacity}
To mitigate parameter shape bias and facilitate the identification of dead neurons, we transition from discrete parameter analysis to a continuous function formulation. By embedding each neuron into a Hilbert space\footnote{See Appendix~\ref{sec:hilbert} for a brief introduction to Hilbert spaces.} $\calH$ and treating it as a rank-1 Hilbert-Schmidt operator, our framework evaluates the actual function the neuron computes, effectively \textit{abstracting away its physical matrix shape}. Furthermore, by integrating this continuous function over the analytically derived surrogate $P_{\calX}$, HOPE identifies dead neurons via a closed-form expectation, bypassing the need for computationally expensive empirical forward passes over a dataset. Together, these features shift the evaluation criterion from raw parameter counts to functional capacity, quantified as the norm of the neuron's function in $\calH$.

We model each neuron as a rank-one \textbf{Hilbert operator} $f_i \in \calH$. This way, a neuron's identity is not defined \textbf{by its evaluation on a single input point $\xVect$}, but its continuous behavior over \textbf{the entire surrogate distribution} $\calX$. We quantify the capacity of the neuron by $\norm{f_i}_\calH$.

We now present this formally. We define our space of neural functions as $\calH \triangleq L_2(\calX, P_{\calX}; \R^c)$, the space of square-integrable functions mapping $\calX$ to the $c$-dimensional output space. Define $g_i : \calX \rightarrow \R$ as $g_i(\xVect) \triangleq \Psi \left( (\wVect_{\text{in},i}^{\text{eff}})^T \xVect + b_i \right)$. We embed $g_i$ as an element of the scalar Hilbert space $\calH_{\text{in}} \triangleq L_2(\calX, P_{\calX}; \R)$. Furthermore, the scalar activation of a neuron is sent to the next layer by scaling with the finite-dimensional output weight vector $\wVect_{\text{out}, i} \in \calH_{\text{out}} \triangleq \R^c$. Since the output across all $c$ dimensions is confined to the one-dimensional subspace spanned by $\wVect_{\text{out}, i}$, this entire continuous landscape is embedded exclusively along a single vector direction. By taking the tensor product of the input function and the output vector, we construct a linear mapping across these spaces: $\calH \cong \calH_{\text{in}} \otimes \calH_{\text{out}}$. Thus the vector-valued function $f_i : \calX \rightarrow \calH_{\text{out}}$, i.e. the neuron, is an element within this tensor product space $f_i \triangleq g_i \otimes \wVect_{\text{out}, i}$. Since this element is constructed from the outer product of one input function and one output vector, each individual neuron $f_i$ is a rank-1 Hilbert-Schmidt operator. See Figure~\ref{fig:tensor_synthesis} for a visualization. This tensor formulation is fundamental for defining the merging operation in HOPE as we will discuss in Sections~\ref{sec:layer_transition_costs} and \ref{sec:parent_params}.

\noindent \textbf{Hilbert-Schmidt Inner Product and Capacity.} Because our neurons are defined as rank-1 operators residing in the tensor $\calH \cong \calH_{\text{in}} \otimes \calH_{\text{out}}$, we must evaluate their geometric relationship using the inner product defined on this composite space: $\inner{f_i}{f_j}_{\calH} = \inner{g_i \otimes \wVect_{\text{out},i}}{g_j \otimes \wVect_{\text{out},j}}_{\calH} = \inner{g_i}{g_j}_{\calH_{\text{in}}} \cdot \inner{\wVect_{\text{out},i}}{\wVect_{\text{out},j}}_{\calH_{\text{out}}} = \E_{\xVect \sim P_{\calX}} [ \Psi(y_i) \Psi(y_j) ] \cdot \inner{\wVect_{\text{out},i}}{\wVect_{\text{out},j}}_{\R^c}$.
We define the \emc{capacity} of a neuron as its Hilbert-Schmidt norm $\normshort{f_i}_{\calH} = \sqrt{\inner{f_i}{f_i}_{\calH}}$, which we will later use to decide what neuron to prune and which macro block to evict.

\begin{figure}[tb]
    \centering
    \scalebox{1.0}{
    \begin{tikzpicture}[font=\small]
        
        % ==========================================
        % PANEL 1: The Input Phase (Landscape)
        % ==========================================
        \begin{scope}[shift={(0,0)}]
            % Titles
            \node[align=center] at (2.5, 5.5) {\textbf{Input}};
            \node[align=center, text=black!70] at (2.5, 5.0) {Function};
            \node[align=center] at (2.5, 4.5) {$g_i \in \calH_{\text{in}}$};
            
            \begin{axis}[
                width=6cm, height=6cm,
                view={35}{35},
                hide axis,
                colormap/viridis,
                domain=-2:2,
                y domain=-2:2,
                samples=25,
                zmin=0, zmax=1.2
            ]
            % A continuous scalar field (Gaussian used for visual appeal)
            \addplot3[surf, faceted color=black!60, opacity=0.9] {exp(-x^2-y^2)};
            \end{axis}
        \end{scope}

        % Tensor Product Operator
        \node at (5.2, 2.5) {\Large $\otimes$};

        % ==========================================
        % PANEL 2: The Output Phase (Vector)
        % ==========================================
        \begin{scope}[shift={(5.5, 0)}]
            % Titles
            \node[align=center] at (2.5, 5.5) {\textbf{Output}};
            \node[align=center, text=black!70] at (2.5, 5.0) {Vector};
            \node[align=center] at (2.5, 4.5) {$\wVect_{\text{out}, i} \in \calH_{\text{out}}$};
            
            \begin{axis}[
                width=6cm, height=6cm,
                view={35}{25},
                axis lines=center,
                ticks=none,
                xmin=0, xmax=2,
                ymin=0, ymax=2,
                zmin=0, zmax=2,
                enlargelimits=0.1
            ]
            % Axis labels
            \node at (2,0,0) [right] {$x_1$};
            \node at (0,2,0) [right] {$x_2$};
            \node at (0,0,2) [above] {$x_3$};
            
            % Save 3D coordinates to the 2D canvas
            \coordinate (Origin2) at (0,0,0);
            \coordinate (Vector2) at (1.0, 0.8, 1.2);
            \end{axis}
            
            % Draw the fixed hardware weight vector OUTSIDE the axis to guarantee top layer
            \draw[->, thick, red!80!black] (Origin2) -- (Vector2) node[above left, text=black] {$\wVect_{\text{out}, i}$};
        \end{scope}

        % Equals Sign
        \node at (10.6, 2.5) {\Large $=$};

        % ==========================================
        % PANEL 3: Tensor Synthesis (Rank-1 Subspace)
        % ==========================================
        \begin{scope}[shift={(10.8, 0)}]
            % Titles
            \node[align=center] at (2.5, 5.5) {\textbf{Tensor}};
            \node[align=center, text=black!70] at (2.5, 5.0) {Rank-1 Operator};
            \node[align=center] at (2.5, 4.5) {$f_i \in \calH$};
            
            \begin{axis}[
                width=6cm, height=6cm,
                view={35}{25},
                axis lines=center,
                ticks=none,
                xmin=0, xmax=2,
                ymin=0, ymax=2,
                zmin=0, zmax=2,
                enlargelimits=0.1
            ]
            % 1. The 1D subspace (infinite line/track)
            \draw[dashed, thick, gray] (0,0,0) -- (2.0, 1.6, 2.4) node[pos=0.8, right, align=left] {};
            
            % 2. The functional landscape projected onto the line (Enlarged Colored Beads)
            \addplot3 [
                domain=0.1:1.4,
                samples=30,
                samples y=1,
                only marks,
                mark=*,
                mark size=3pt,
                scatter,
                point meta={exp(-15*(x-0.7)^2)}, % Simulates the Gaussian bump from Panel 1
                colormap/viridis
            ] (
                {x*1.5}, 
                {x*1.2}, 
                {x*1.8}
            );

            % Axis labels
            \node at (2,0,0) [right] {$x_1$};
            \node at (0,2,0) [right] {$x_2$};
            \node at (0,0,2) [above] {$x_3$};
            
            % Explanatory text inside the plot
            \node[align=center, font=\tiny, text=black!80] at (1.0, 0.8, 2.2) {Image of landscape\\restricted to line};

            % Save 3D coordinates to the 2D canvas
            \coordinate (Origin3) at (0,0,0);
            \coordinate (Vector3) at (1.0, 0.8, 1.2);
            \end{axis}
            
            % 3. The Weight Vector drawn OUTSIDE the axis to guarantee it sits on top of the beads
            \draw[->, thick, red!80!black] (Origin3) -- (Vector3) node[above left, text=red!80!black] {$\wVect_{\text{out}, i}$};
        \end{scope}
        
    \end{tikzpicture}
    }
    \caption{\small Visualizing the rank-1 tensor of a neuron. \textbf{Left:} The input phase computes a continuous (infinite-dimensional) scalar landscape $g_i(\xVect)$. \textbf{Middle:} The output phase defines a single, finite-dimensional weight vector $\wVect_{\text{out}, i}$. \textbf{Right:} The tensor product binds them. The entire function's landscape (represented by the colored points) is mapped along the 1D subspace spanned by $\wVect_{\text{out}, i}$, which enforces the definition of a rank-1 operator.}
    \label{fig:tensor_synthesis}
\end{figure}
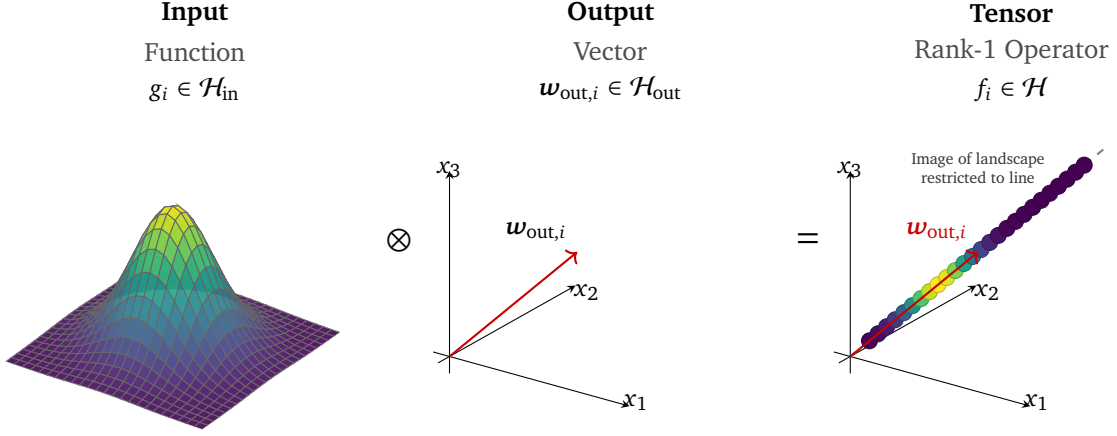

\noindent \textbf{Kernel Formulation.}
We define the kernel of two neurons $i,j$ as $K(i, j) \triangleq \E_{\xVect \sim P_{\calX}} [ \Psi(y_i) \Psi(y_j) ]$. Under this, the capacity of a neuron simplifies to $\normshort{f_i}_{\calH} = \normshort{\wVect_{\text{out},i}}_2 \cdot \sqrt{K(i, i)}$ and the inner product of two neurons becomes $\inner{f_i}{f_j}_{\calH} = \inner{\wVect_{\text{out},i}}{\wVect_{\text{out},j}}_{\R^c} K(i,j)$. Closed form expression for these kernels are provided at the end of this section, and their derivation is presented in Appendix~\ref{sec:kernel} when $\Psi$ is the $\ReLU$ activation function.

\noindent \textbf{Neuron Scale Invariance.} In networks with PH-1 activations, scaling $\wVect_{\text{in},i}^{\text{eff}}$ and $b_i$ of a neuron by $\lambda>0$ and $\wVect_{\text{out},i}$ by $1/\lambda$ alters weight magnitudes without changing the downstream function. This symmetry confounds raw magnitude-based criteria. HOPE mitigates this issue by defining capacity as the Hilbert norm: $\normshort{f_i}_{\calH} = \normshort{\wVect_{\text{out},i}}_2 \sqrt{K(i, i)}$. Because positive homogeneity scales the kernel $K$ by $\lambda$, the opposing factors cancel. Consequently, HOPE guarantees an invariant capacity score regardless of weight resharding.

\noindent \textbf{Neuron Shape Invariance.} 
By definition, the neuron's functional capacity $\|f_i\|_{\calH} = \|\wVect_{\text{out},i}\|_2 \sqrt{K(i, i)}$ depends on the input space $\calX$ solely through the kernel term $K(i,i) \triangleq \E_{\xVect \sim P_{\calX}} [ \Psi^2(y_i) ]$. Rather than counting discrete incoming parameters $n$ or computing weight magnitudes such as $\|\wVect_{\text{in},i}^\text{eff}\|$ or $\|\wVect_{\text{raw},i}\|$, this formulation abstracts away the physical dimensionality of the input tensor $\xVect \in \R^n$.\footnote{While increasing the fan-in $n$ naturally inflates the pre-activation variance $\Var_{\calX}(\wVect_{\text{raw},i}^T \xVect)$, this scaling artifact is then mitigated by the absorbed BN parameters. Because  $\wVect_{\text{in},i}^{\text{eff}} \triangleq (\gamma_i/\sqrt{\sigma_i^2 + \epsilon}) \wVect_{\text{raw},i}$, the variance of the pre-activation signal $y_i$ is bounded entirely by the learned scale $\gamma_i^2$. Consequently, the expected activation energy $\E_{\xVect \sim P_{\calX}} [ \Psi^2(y_i) ]$ remains decoupled from the physical width of the input tensor.} Furthermore, when paired with the layer-wise magnitude neutrality axiom, this function-based approach ensures the neuron's evaluation is entirely invariant to both its input and output dimensions.\footnote{Although capacity $\|f_i\|_\calH$ scales with the output dimension $c$ via $\|\wVect_{\text{out},i}\|_2$, HOPE mitigates this downstream. As shown in Section~\ref{sec:layer_transition_costs}, compression is governed by the distortion cost $\calJ$. Derived axiomatically, $\calJ$ normalizes a neuron's capacity against the layer's total capacity. Since all neurons in a layer share the same output space $\R^c$, this emergent normalization factors out $c$.}

\noindent \textbf{Neuron Merging via Hilbert-Schmidt Projection.} Similar to pruning, merging also reduces the network's neuron count by one, but it can provide a higher-fidelity reduction when the selected pair have a strong cosine similarity in $\calH$. We define the merger through an optimal Hilbert-Schmidt projection. Since neurons $i,j$ are each rank-1 operators, their joint contribution $[f_i, f_j]$ spans a 2-dimensional subspace in $\calH$ consisting of operators of rank at most 2. Merging them into a single parent neuron is defined as finding the optimal rank-1 approximation of this tensor subspace.

\begin{insight}{Self-Kernel of \ReLU Neurons}
Let $\normPDF$ and $\normCDF$ be the standard Normal PDF and CDF respectively. Then:
\begin{equation}
    K(i, i) = (\gamma_i^2 + \beta_i^2)\normCDF\left(\frac{\beta_i}{|\gamma_i|}\right) + \beta_i|\gamma_i|\normPDF\left(\frac{\beta_i}{|\gamma_i|}\right)
\end{equation}
\end{insight}

\begin{alertbox}{Cross-Kernel of \ReLU Neurons}
For brevity, let:
\begin{equation}
\label{eq:rho_hat_maintext}
\rho_{\text{eff}} \triangleq \frac{\inner{\wVect_{\text{in},i}^{\text{eff}}}{\wVect_{\text{in},j}^{\text{eff}}}}{\norm{\wVect_{\text{in},i}^{\text{eff}}}_2 \norm{\wVect_{\text{in},j}^{\text{eff}}}_2} \quad,\quad \kappa \triangleq \left( \frac{\rho_{\text{eff}}}{1 - \rho_{\text{eff}}^2} \right) \left( \frac{|\gamma_i|}{\norm{\wVect_{\text{in},i}^{\text{eff}}}_2} \right) \left( \frac{|\gamma_j|}{\norm{\wVect_{\text{in},j}^{\text{eff}}}_2} \right) \quad,\quad
\hat{\rho}_{ij} \triangleq \frac{2\kappa}{1 + \sqrt{1 + 4\kappa^2}} \,.
\end{equation}
Then the cross-kernel has the form\footnote{While the exact cross-kernel can be derived analytically, it requires evaluating a bivariate normal CDF for every neuron pair, which is computationally prohibitive for large networks. Instead, we approximate the kernel by assuming zero bias $\beta_i, \beta_j \approx 0$. This isolates the angular alignment, $\hat{\rho}_{ij}$, as the primary driver of redundancy and avoids costly CDF evaluations. Full derivations of both the exact and approximate kernels are in Appendix~\ref{sec:kernel}.}:
\begin{equation}
    K(i, j) \approx \frac{1}{\pi} \left( \sqrt{1 - \hat{\rho}_{ij}^2} + (\pi - \arccos \hat{\rho}_{ij}) \hat{\rho}_{ij} \right) \sqrt{K(i, i) K(j, j)}
\end{equation}
\end{alertbox}

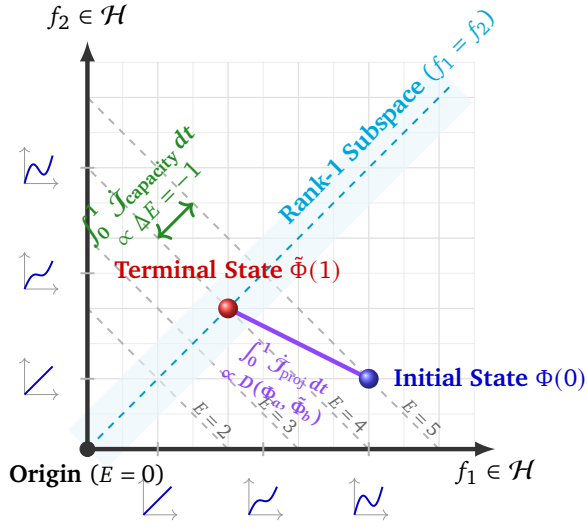
\begin{figure}[tb]
\centering
\begin{minipage}{0.50\textwidth}
\centering
\resizebox{1.0\linewidth}{!}{
\scalebox{1.0}{
\begin{tikzpicture}[
	% Premium styles
    axis_line/.style={-latex, line width=1.8pt, draw=black!80},
    proj_line/.style={line width=1.8pt, draw={rgb,255:red,143; green,70; blue,255}, shorten >=2pt},
    L1_contour/.style={dashed, gray!60, thick},
    % Subspace styles
    subspace_beam/.style={line width=22pt, cyan!8, opacity=0.6},
    subspace_core/.style={dashed, thick, cyan!80!black},
    % Typography and label styles
    state_label/.style={text opacity=1, rounded corners=3pt, inner sep=4pt, font=\small},
    label_bg/.style={text opacity=1}, % Added for readability over grid
    % Thumbnail macro
    miniplot/.style={scale=0.40, every node/.style={transform shape}}
]
    % FUNCTION CARD THUMBNAILS 
    \newcommand{\thumbPureLinear}{
        \begin{tikzpicture}[miniplot]
            \draw[->, gray!80, very thin] (0,0) -- (1.2,0);
            \draw[->, gray!80, very thin] (0,0) -- (0,1.3);
            \draw[blue!80!black, thick, domain=0:1, samples=50] plot (\x, {\x});
        \end{tikzpicture}
    }
    \newcommand{\thumbHalfRipple}{
        \begin{tikzpicture}[miniplot]
            \draw[->, gray!80, very thin] (0,0) -- (1.2,0);
            \draw[->, gray!80, very thin] (0,0) -- (0,1.3);
            \draw[blue!80!black, thick, domain=0:1, samples=50] plot (\x, {\x + 0.2*sin(\x*360)});
        \end{tikzpicture}
    }
    \newcommand{\thumbHeavyRipple}{
        \begin{tikzpicture}[miniplot]
            \draw[->, gray!80, very thin] (0,0) -- (1.2,0);
            \draw[->, gray!80, very thin] (0,0) -- (0,1.3);
            \draw[blue!80!black, thick, domain=0:1, samples=100] plot (\x, {\x + 0.4*sin(\x*360)});
        \end{tikzpicture}
    }

    % THE HILBERT SPACE CANVAS

    % Background architectural grid
    \draw[step=0.5, gray!15, very thin] (0,0) grid (5.5,5.5);
    \draw[step=1.0, gray!25, thin] (0,0) grid (5.5,5.5);

    % L1 Capacity Contours (Constant E = ||f1|| + ||f2||)
    \draw[L1_contour] (5,0) -- (0,5) node[pos=0.07, sloped, above, font=\scriptsize, text=gray!80!black, label_bg, inner sep=1pt] {$E = 5$};
    \draw[L1_contour] (4,0) -- (0,4) node[pos=0.10, sloped, above, font=\scriptsize, text=gray!80!black, label_bg, inner sep=1pt] {$E = 4$};
    \draw[L1_contour] (3,0) -- (0,3) node[pos=0.13, sloped, above, font=\scriptsize, text=gray!80!black, label_bg, inner sep=1pt] {$E = 3$};
    \draw[L1_contour] (2,0) -- (0,2) node[pos=0.16, sloped, above, font=\scriptsize, text=gray!80!black, label_bg, inner sep=1pt] {$E = 2$};

    % The Glowing Rank-1 Subspace
    \draw[subspace_beam] (0,0) -- (5.2,5.2);
    \draw[subspace_core] (0,0) -- (5.2,5.2);
    \node[cyan!90!black, rotate=45, above, state_label] at (4.5, 4.5) {\textbf{Rank-1 Subspace} ($f_1=f_2$)};

    % Axes
    \draw[axis_line] (0,0) -- (5.8,0) node[below, font=\bfseries] {$f_1 \in \calH$};
    \draw[axis_line] (0,0) -- (0,5.8) node[above, font=\bfseries] {$f_2 \in \calH$};

    % X-axis Function Cards
    \draw[thick, gray!60] (1, 0.1) -- (1, -0.1) node[below=5pt] {\thumbPureLinear};
    \draw[thick, gray!60] (2.5, 0.1) -- (2.5, -0.1) node[below=5pt] {\thumbHalfRipple};
    \draw[thick, gray!60] (4, 0.1) -- (4, -0.1) node[below=5pt] {\thumbHeavyRipple};

    % Y-axis Function Cards
    \draw[thick, gray!60] (0.1, 1) -- (-0.1, 1) node[left=5pt] {\thumbPureLinear};
    \draw[thick, gray!60] (0.1, 2.5) -- (-0.1, 2.5) node[left=5pt] {\thumbHalfRipple};
    \draw[thick, gray!60] (0.1, 4) -- (-0.1, 4) node[left=5pt] {\thumbHeavyRipple};

    % THE STRUCTURAL TRANSITION PATHS

    % Coordinates for the transition
    \coordinate (PhiA) at (4.0, 1.0); % Initial state E=5
    \coordinate (PhiB) at (2.0, 2.0); % Target state E=4

    % J_proj Path (Blue, Straight Euclidean)
\draw[proj_line] (PhiA) -- (PhiB) 
        node[midway, below, sloped, state_label, label_bg, align=center] {
            {\scriptsize \color[RGB]{143, 70, 255}\boldmath $\int_0^1 \dot{\calJ}_{\text{proj}}\,dt$} \\ 
            {\scriptsize \color[RGB]{143, 70, 255}\boldmath $\propto D(\Phi_a,\tilde{\Phi}_b)$}
        };

	% THE CAPACITY COST VISUALIZATION
    
    % J_capacity (Detached background environmental drop)
    \draw[<->, line width=1.2pt, draw={rgb,255:red,34; green,139; blue,34}] (1.5, 3.5) -- (1.0, 3.0) 
        node[midway, sloped, above=2pt, state_label, label_bg, text={rgb,255:red,34; green,139; blue,34}, align=center] {
            \textbf{\boldmath $\int_0^1 \dot{\calJ}_{\text{capacity}}\, dt$} \\ 
            $\propto \Delta E = -1$
        };

    % STATE NODES

    % Initial State Sphere
    \shade[shading=ball, ball color=blue!70] (PhiA) circle (4pt);
    \node[state_label, label_bg, align=center, right=6pt] at (PhiA) {\textbf{\color{blue!80!black}Initial State $\Phi(0)$}};

    % Terminal State Sphere
    \shade[shading=ball, ball color=red!80] (PhiB) circle (4pt);
    \node[state_label, label_bg, align=center, above=6pt] at (PhiB) {\textbf{\color{red!80!black}Terminal State $\tilde{\Phi}(1)$}};

    % Origin Singularity
    \filldraw[black!80] (0,0) circle (3pt);
    \node[state_label, label_bg, below=2pt] at (0,0) {\textbf{Origin} ($E=0$)};

\end{tikzpicture}
}
} % End resizebox
\end{minipage}\hfill % \hfill pushes the two minipages apart
\begin{minipage}{0.50\textwidth}
\caption{\textbf{Visualization of Transition Costs of Merging in $\calH^2$.} The axes schematically represent the space of continuous functions $\calH$, illustrated as a smooth transition from linear to sinusoidal. During a merge operation, the layer transitions from an initial state $\Phi(0)=\Phi_a$ to a pre-deletion target $\Phi(1)=\tilde{\Phi}_b$ where $f_1 = f_2$. For the sake of illustration, suppose $E(t)\approx E_b$ and $c(\Phi(t))=c_0$ for $0 \leq t \leq 1$. Then $\int_0^1 \dot{\calJ}_{\text{capacity}} \,dt \approx -\frac{c_0}{E_b} \int_0^1 \dot{E}(t) \,dt = -\frac{c_0}{E_b}(E_b - E_a) = -\frac{c_0}{E_b} \Delta E \propto -\Delta E$, yielding a penalty proportional to $-\Delta E$. However, $\int_0^1 \dot{\calJ}_{\text{proj}} \,dt \approx \frac{c_0}{E_b} \int_0^1 \dot{s}(t) \,dt = \frac{c_0}{E_b}(s(1)-s(0)) \propto D(\Phi_a,\tilde{\Phi}_b)$, the Euclidean projection distance shown as the straight purple line.}

\label{fig:tuple_projection}
\end{minipage}
\end{figure}

\section{Layer Transition Costs}
\label{sec:layer_transition_costs}
Scale symmetries also manifest globally: shallow layers processing high-variance data often yield neurons with larger capacities than deeper layers operating on compressed latent representations \cite{hanin2018start, Tanaka2020PruningNN}. For a compression method to account for this bias, it must evaluate individual neurons within the context of their entire layer \cite{lee2021layer}. We formalize this global context via a \textit{layer state} $\Phi \triangleq (f_1, f_2, \dots, f_{N})$, where $N$ is the number of active neurons in the layer. A single compression step maps an initial state $\Phi_a$ to a reduced state $\Phi_b$, where $|\Phi_b|=N-1$. The entire compression process is thus a chain of discrete state transitions across various layers, guided by a cost $\calJ(\Phi_a, \Phi_b) > 0$ that quantifies the resulting model distortion. This section focuses on deriving this $\calJ$.

\subsection{Continuous-Time Relaxation}
Despite the conceptual clarity of these transitions, their discrete nature (where the architecture hops from state $\Phi_a$ to $\Phi_b$ via pruning or merging) presents a significant barrier to mathematical analysis. To bridge the gap between abstract analysis and algorithmic execution, we proceed in two steps. First, we perform a \textbf{continuous relaxation}: instead of an instantaneous jump, we define a continuous deformation $\Phi(t)$ over $t \in [0,1]$, interpolating between $\Phi(0) = \Phi_a$ and a pre-deletion target $\Phi(1) = \tilde{\Phi}_b$ (Figure~\ref{fig:tuple_projection}). This allows us to use differential equations to express the infinitesimal cost of shrinking a layer's capacity. Second, to compute the total transition cost, we integrate this differential cost with respect to $t$. Our objective is to resolve this integral into an expression that depends only on the physically realizable endpoints ($\Phi_a$ and $\tilde{\Phi}_b$), bypassing the need to evaluate fictitious intermediate states along the continuous path. However, since this integral generally lacks a closed-form solution, and numerical evaluation incurs computationally prohibitive runtime overhead, we instead derive a \textbf{closed-form upper bound} on the analytically intractable integral.

\noindent \textbf{Layer Capacity.} To develop a layer cost $\calJ$, we first extend the single-neuron capacity, $\|f_i\|$, to define a layer capacity $E(\Phi)$ for state $\Phi$, where $E(f_1)=\|f_1\|$. A natural requirement is that $E(\Phi)$ remains invariant to arbitrary neuron partitioning. Assuming $E(\Phi)$ is a symmetric, separable, and homogeneous functional of individual capacities, this condition uniquely determines $E(\Phi) = \sum_{k=1}^N \normshort{f_k}_{\calH}$ (by Lemma~\ref{prop:unique_l1}). For some intuition, suppose that $E(\Phi) = \left(\sum \normshort{f_k}_{\calH}^p\right)^{1/p}$. Partitioning a neuron $f_0$ into $M$ fractions $f_0/M$ yields $M^{(1-p)/p} \normshort{f_0}_\calH$. Capacity invariance for any $M$ requires $(1-p)/p = 0$, yielding $p=1$.

\noindent \textbf{Axiomatic Cost Formulation.} To ensure a well-posed definition of $\calJ$, we introduce the following natural axioms: \textit{1. Magnitude Neutrality:} $\calJ$ must be scale invariant: $\forall k > 0 \,;\, \calJ(k\, \Phi_a , k\, \Phi_b) = \calJ(\Phi_a , \Phi_b)$. \textit{2. Connectivity Preservation:} $\calJ$ must establish an asymptotic barrier preventing layer extinction: $\lim_{E(\Phi_b) \rightarrow 0^+} \calJ = \infty$. \textit{3. Infinitesimal Capacity Dependence:} $\calJ$ must be additive along continuous paths and be driven by the reduction in layer capacity:
$\calJ(\Phi_a, \Phi_b) = \int_{0}^{1} -\xi(\Phi(t)) \dot{E}(t) dt$, where $\dot{E}(t) \triangleq dE(\Phi(t))/dt$ and $\xi(\Phi(t)) > 0$ is a state-dependent density function. While Axioms 1 and 2 define boundaries of the theory, Axiom 3 acts as an idealized analytical tool modeling a continuous capacity drain $\dot{E}(t) < 0$. This allows us to deduce the fundamental shape of the cost function. Under these premises, we can prove (By Theorem~\ref{theorem:path_log_form}) that $\calJ$ must obey $\calJ_{\text{capacity}}(\Phi_a, \Phi_b) = \int_{0}^{1} -c(\Phi(t)) \frac{\dot{E}(t)}{E(\Phi(t))} dt$, where $\dot{E}(t) < 0$ (due to capacity reduction) and $c(\Phi(t)) > 0$ is a scale-invariant factor (i.e., $c(k \Phi) = c(\Phi)$ for any $\Phi \in \calH^N$ and $k > 0$).

\noindent \textbf{Piecewise Constant $c(\Phi(.))$.} To bridge continuous theory with discrete execution, we restrict $c(\Phi)$ to remain constant along any discrete state transition $\Phi_a \rightarrow \Phi_b$, e.g. $c(\Phi(t)) = c(\Phi_a)$ for $t \in [0,1]$. This allows us to factor $c(\Phi)$ out of the integral for both $\calJ_{\text{capacity}}$ and all subsequently derived cost functionals; for $\calJ_{\text{capacity}}$, this directly yields the analytical solution $\calJ_{\text{capacity}} = c(\Phi_a) \ln(\frac{E_a}{E_b})$. Upon reaching the terminal state, physically removing extinguished neurons causes $c(\Phi(t))$ to snap to a new value $c(\Phi_b)$. Consequently, $c(\Phi(t))$ acts as a globally piecewise constant function that remains locally constant during any integration step. While $\calJ_{\text{capacity}}$ is not yet the final objective used in our optimizer, confirming that it satisfies Axioms 1 and 2 ensures we are on track, while its derivation via integration inherently satisfies the idealized capacity dependence assumption.

\subsection{Bounding the Projection Cost}
\label{sec:bound_proj_cost}
While $\dot{\calJ}(t)$ is driven by the relative capacity reduction $-\dot{E}/E$, our framework needs to minimize projection error\footnote{Transitioning from $\calJ_{\text{capacity}}$ to $\calJ_{\text{proj}}$ ensures sensitivity to feature alignment. For instance, merging two orthogonal neurons introduces a severe subspace projection error while $\calJ_{\text{capacity}}$ evaluates this catastrophic alignment loss identically to a merge between two collinear (hence redundant) neurons due to their equivalent linear capacity reductions. Shifting to $\calJ_{\text{proj}}$ reorients the optimization objective from macroscopic reduction to minimizing distortion within the network's internal mapping.} (Section~\ref{sec:capacity}). We bridge the two by calibrating along an orthogonal trajectory where $ds = -dE$ translates the abstract capacity loss $-\dot{E}$ into a geometric speed $\dot{s}$. Here $s(t) = \int_{0}^{t} \normshort{\dot{\Phi}(\tau)}_{\calH^N} d\tau$ is the arc-length swept by $\Phi(t)$ through the space $\calH^N$. Because $\calH^N$ is isotropic, this substitution generalizes to any arbitrary deformation path, yielding $\dot{\calJ}_{\text{proj}}(t) = c(\Phi(t)) \frac{\dot{s}(t)}{E(\Phi(t))}$ (Definition~\ref{def:inf_weber}). This substitution shifts $\calJ$ from pure capacity loss to any distance traversed, meaning the strict $\dot{E}(t) < 0$ assumption from the idealized model is no longer required along the physical path.

The compression algorithm executes discrete leaps (e.g., snapping neurons $f_i$ and $f_j$ to a shared parent $f_p$). Evaluating the cost of this transition conceptually requires integrating $\dot{\calJ}_{\text{proj}} = c \cdot \dot{s} / E$ over the jump path. However, because runtime integration is computationally prohibitive, we seek a fast, closed-form proxy. Since underestimating this integral risks destructive jumps (e.g., removing orthogonal features) and breaching layer depletion barriers before the continuous cost can diverge, we derive a \textit{closed-form upper bound} to enforce cautious greedy optimization. We construct this bound by exploiting the inverse relationship between $\dot{\calJ}_{\text{proj}}$ and $E$ in $\dot{\calJ}_{\text{proj}} = c \cdot \dot{s} / E$.

For any arbitrary deformation path connecting $\Phi_a$ to $\Phi_b$, we can establish an upper bound on the integral cost by replacing the dynamic capacity $E(t)$ with a constant minimum, $E_{\min}$, allowing us to pull the denominator outside the integral. This yields a bounded fractional cost where the numerator is the path's total arc length, $\int_{0}^{1} \dot{s}(t) dt$. Because infinitely many curves in $\calH^N$ connect the two states, this establishes a family of valid upper bounds. To tighten this proxy cost, we minimize the numerator by selecting the path with the shortest arc length: the straight-line trajectory in $\calH^N$. This evaluates to the traversed Euclidean distance $D(\Phi_a, \tilde{\Phi}_b) \triangleq \normshort{\Phi_a - \tilde{\Phi}_b}_{\calH^N} = ( \sum_{k=1}^N \normshort{f_k^{(a)} - \tilde{f}_k^{(b)}}_{\calH}^2 )^\frac{1}{2}$.

Next, to complete this bound, we must safely approximate the denominator's minimum $E_{\min}$ along this chosen straight-line path. Because the straight-line geometric path acts as a secant across the space of functions (abandoning the strict $\dot{E}(t) < 0$ assumption), the capacity $E(t)$ can temporarily dip below the pre-deletion target $E(\tilde{\Phi}_b)$. To safely absorb this without breaking the integral bound, we introduce a safety buffer by evaluating the denominator at the true terminal state $E(\Phi_b)$. For highly correlated neuron pairs, $E(t) \geq E(\Phi_b)$ throughout the straight-line transition (Lemma~\ref{lemma:discontinuity_capacity_bound}).

Substituting the minimized numerator $D$ and evaluating the constant denominator as $E(\Phi_b)$ yields the final bound $\calJ_{\text{proj}}(\Phi_a, \Phi_b) \le c(\Phi_a) \frac{D(\Phi_a, \tilde{\Phi}_b)}{E(\Phi_b)} \equiv \calJ_{\text{bound}}(\Phi_a, \Phi_b)$ (see Theorem~\ref{thm:discrete_subspace}). Here, $\tilde{\Phi}_b \in \calH^N$ is the \textit{pre-deletion} target at $t=1$ (e.g., a duplicated parent $[f_p, f_p]$) but the $N$-dimensional structure remains intact. Conversely, $\Phi_b \in \calH^{N-1}$ is the true \textit{terminal state}: the layer after the extinguished neuron is dropped. This separation ensures no dimensional mismatch in the arguments of $D$, while the denominator only relies on the $(N-1)$-dimensional post-deletion capacity $E(\Phi_b)$.

\paragraph{Axiomatic Consistency of the Bounded Proxy.} While the continuous functional $\calJ_{\text{capacity}}$ was derived from our foundational axioms, the subsequent derivation of $\calJ_{\text{bound}}$ alters the underlying differential equation. Specifically, to bypass the expensive runtime integration, we introduced a surrogate curve and bounded the capacity denominator. Because these approximations manipulate the original differential equation, it is no longer guaranteed \textit{a priori} that the resulting closed-form proxy inherits the axiomatic properties of its continuous predecessor. However, we can prove that $\calJ_{\text{bound}}$ (and consequently $\calJ_{\text{final}}$) still  preserves the foundational axioms of Magnitude Neutrality and Connectivity Preservation (by Proposition~\ref{prop:bound_axioms}). However, the Infinitesimal Capacity Dependence assumption acts primarily as an analytical tool rather than a fundamental necessity, and is intentionally relaxed. Specifically, $\calJ_{\text{capacity}}$ relies on integration over a path characterized by a monotonic capacity drain $\dot{E}(t) < 0$. However, deriving the closed-form $\calJ_{\text{bound}}$ abandons this path integration in favor of a straight-line approximation evaluated at endpoints. Because this straight-line projection cuts directly across $\calH^N$, the intermediate capacity along the path may temporarily fluctuate, violating the assumption of monotonic decrease required by the original differential equation. Consequently, $\calJ_{\text{bound}}$ knowingly sacrifices the path-additivity required by the modeling assumption. This relaxation is necessary to translate abstract continuous theory into an efficient $\calO(1)$ evaluation of discrete state transitions.

\begin{insight}{Practical Notes}
\noindent \textbf{The Correlation Constraint.} The assumption $E(t) \geq E(\Phi_b)$ holds only for highly correlated neurons, but this poses no practical limitation. Because the projection error $D(\Phi_a, \Phi_b)$ vanishes for collinear candidates, the greedy optimizer naturally minimizes $\calJ_{\text{bound}}$ by actively selecting highly correlated pairs, inherently satisfying this requirement.

\vspace{0.15in}
\noindent \textbf{Locality of the Projection Error.} Evaluating $\calJ$ across a wide layer might seem computationally intractable. However, for reductions modifying only a small subset of neurons $\calS$ (e.g., pruning or merging), the cost restricts entirely to the perturbed subspace: $\calJ_{\text{bound}}(\Phi_a, \Phi_b) = c(\Phi_a) \sqrt{\sum_{k \in \calS} \normshort{f_k^{(a)} - f_k^{(b)}}_{\calH}^2}/E(\Phi_b)$ (Corollary~\ref{cor:locality}). This isolates the computation from the total architectural width, guaranteeing $\calO(1)$ execution time.
\end{insight}

\noindent\textbf{Choice of $c(\Phi)=N$.} We previously specified $c(\Phi)$ to be piecewise constant; we now propose a more specific definition: setting $c(\Phi)=N$ for each continuous piece. This is to avoid unfair removal of critical diversity from wide layers by the global optimizer before addressing obvious redundancies in narrow bottlenecks, which may occur as capacity $E(\Phi)$ intrinsically scales with layer width. To mitigate this width bias, we normalize $\calJ$ using the \textit{average feature capacity}\footnote{Consider a mean-field assumption where each active neuron contributes an average capacity $\bar{e}$. The incremental cost of pruning a single neuron evaluates to $\calJ_{\text{prune}} \approx \frac{N \cdot \bar{e}}{N \cdot \bar{e}} = 1$. This normalization renders the penalty invariant to the instantaneous layer width. Without this dynamic coupling (e.g., if $c$ were anchored to $N_{\text{initial}}$), the incremental cost would artificially explode as the live capacity shrinks, forcing an artificial uniformity that prevents the optimizer from fully extinguishing noisy, redundant blocks.} $E(\Phi) / N$, which can be implemented by setting $c(\Phi)=N$. Substituting this $c(\Phi)$ into $\calJ_{\text{bound}}$ yields the final cost $\calJ_{\text{final}} \,\triangleq\, \frac{N \cdot D(\Phi_a, \tilde{\Phi}_b)}{E(\Phi_b)}$. We can instantiate the pruning and merging costs as special cases of $\calJ_{\text{final}}$.

\subsection{Final Pruning and Merging Costs}
Pruning a neuron $f_i$ corresponds to projecting its rank-1 operator down to the null operator $\zeroVect$. Because the perturbed subspace only contains this single neuron $\calS = \{i\}$ and its terminal state is $\zeroVect$, the projection error simplifies to $D = \sqrt{\normshort{f_i - \zeroVect}_{\calH}^2} = \normshort{f_i}_{\calH}$. By evaluating the terminal capacity as $E(\Phi_b) = E_a - \normshort{f_i}_{\calH}$ we get $\calJ_{\text{prune}} = \frac{N \cdot \normshort{f_i}_{\calH}}{E_a - \normshort{f_i}_{\calH}}$. 

Merging a neuron pair is slightly more involved. For neurons $i$ and $j$, their joint operator $[f_i, f_j]$ spans a rank-2 subspace in $\calH$. Because $f_i$ and $f_j$ are vector-valued functions, their joint operator is matrix-valued, denoted as $\Wall \triangleq [f_i, f_j]$. Merging compresses this into a rank-1 approximation $\Wall'$. Classic unconstrained rank truncation (Eckart-Young-Mirsky) prescribes a rank-one basis $f_b$ and independent scaling factors $\alpha, \beta \in \R$ by solving $\min_{f_b \in \calH, \alpha, \beta \in \R} \normshort{\Wall - \Wall'}^2_{\calH}$, where $\Wall' = [\alpha f_b, \beta f_b]$. 

However, because a physical neuron must produce a single unified output, we must restrict the valid replacement pair to $\Wall' = [f_p, f_p]$. This enforces the constraint $\alpha = \beta = 1$, yielding the constrained objective $\min_{f_p \in \calH} \normshort{\Wall - \Wall'}^2_{\calH}$ and rendering standard unconstrained projections inapplicable. 

Deferring the derivation of the optimal parent $f_p$ to Section~\ref{sec:parent_params}, we first establish the objective functional itself. Since the distance $D$ is the expected Frobenius projection error under the \emc{Hilbert-Schmidt norm}, we expand it as follows:
\begin{equation}
D^2(\Phi_a, \Phi_b) = \normshort{\Wall - \Wall'}^2_{\calH} = \E_{\xVect \sim P_{\calX}} \left[ \normshort{\Wall(\xVect) - \Wall'(\xVect)}_F^2 \right] = \normshort{f_i - f_p}_{\calH}^2 + \normshort{f_j - f_p}_{\calH}^2 \nonumber \,.
\end{equation}
The terminal capacity $E_b$ updates by swapping the eliminated children for the new parent: $E_b = E_a - \normshort{f_i}_{\calH} - \normshort{f_j}_{\calH} + \normshort{f_p}_{\calH}$. Substituting $D$ and $E_b$ yields the final merging cost.

\begin{insight}{Pruning and Merging Costs}
\begin{equation}
\calJ_{\text{prune}} = \frac{N \, \normshort{f_i}_{\calH}}{E_a - \normshort{f_i}_{\calH}} \qquad,\qquad \calJ_{\text{merge}} = \frac{N \, \sqrt{\normshort{f_i - f_p}_{\calH}^2 + \normshort{f_j - f_p}_{\calH}^2}}{E_a - \normshort{f_i}_{\calH} - \normshort{f_j}_{\calH} + \normshort{f_p}_{\calH}} \,.
\end{equation}
\end{insight}

\section{Generating the Parent Neuron}
\label{sec:parent_neuron}

\subsection{The Parent Neuron in Hilbert Space}
\label{sec:parent_params}
We determine the optimal parent neuron $f_p^*$ by minimizing $\calJ_{\text{merge}}(f_p)$ subject to $f_p\in \calN$, where $\calN$ denotes the space of realizable neurons:
\begin{equation}
\label{eq:realizable_neurons}
\calN \triangleq \{ f \,|\, f(\xVect) = \wVect_{\text{out}} \Psi(\tilde{\wVect}_{\text{in}} \cdot \tilde{\xVect}) \} \subset \calH \,.
\end{equation}
Here $\tilde{\xVect} = [\xVect, 1]^T$ and $\tilde{\wVect}_{\text{in}} = [\wVect_{\text{in}}^{\text{eff}}, b]^T$ denote the augmented inputs and weights. Any non-zero function $f \in \calH$ can be decomposed into a scalar magnitude $s > 0$ and a direction $\psi \in \calH$, such that $f = s\psi$ and $\|\psi\|_{\calH} = 1$. Applying this to the parent neuron $f_p$ allows us to reformulate the search for $f_p^*$ as the following nested optimization problem:
\begin{equation}
\label{eq:nested}
    \min_{s \in \R^+} \min_{ \psi \in \calN} \frac{\sqrt{\|f_p - f_i\|_{\calH}^2 + \|f_p - f_j\|_{\calH}^2}}{E_a - \|f_i\|_{\calH} - \|f_j\|_{\calH} + \|f_p\|_{\calH}} \qquad \mbox{s.t} \qquad f_p = s \, \psi \quad,\quad \normshort{\psi}_{\calH} = 1 \quad,\quad s>0
\end{equation}

\subsubsection{Optimal Direction}
\label{sec:exact_optimal_direction}
We first focus on the inner optimization problem of (\ref{eq:nested}). Expanding the squared numerator of the cost functional reveals that for a fixed magnitude $s > 0$, minimizing the cost in $\psi$ is equivalent to maximizing the alignment $\inner{\psi}{f_i + f_j}_{\calH}$ in $\psi$. To solve the latter, we enforce the realizability $\psi \in \calN$ and unit-norm $\normshort{\psi}_{\calH} = 1$ constraints by decoupling the input and output parameters, yielding the parametric form:
\begin{equation}
\label{eq:psi_valid_exact}
\psi = \frac{\Psi(\uVect \cdot \tilde{\xVect})}{\sqrt{K(\uVect,\uVect)}} \vVect \,.
\end{equation}
Substituting this parametric form into the unconstrained alignment objective and distributing the Hilbert inner product via the kernel identity isolates the output direction $\vVect$. By the Cauchy-Schwarz inequality, the optimal $\vVect^*$ must align with $\sum_{k \in \{i,j\}} K(\uVect^*, \tilde{\wVect}^k_{\text{in}}) \wVect^k_{\text{out}}$. Substituting this optimal $\vVect^*$ back into the objective simplifies the alignment inner product to the Euclidean norm of that sum, yielding the final objective for the optimal $\uVect^*$ (Theorem~\ref{thm:optimal_parent}):
\begin{equation}
\label{eq:opt_u_exact_main}
\vVect^* = \frac{\sum_{k \in \{i,j\}} K(\uVect^*, \tilde{\wVect}^k_{\text{in}}) \wVect^k_{\text{out}}}{\normshort{\sum_{k \in \{i,j\}} K(\uVect^*, \tilde{\wVect}^k_{\text{in}}) \wVect^k_{\text{out}}}} \,\,,\,\, \uVect^* = \argmax_{\normshort{\uVect}=1} \frac{\normshort{\sum_{k \in \{i,j\}} K(\uVect, \tilde{\wVect}^k_{\text{in}}) \wVect^k_{\text{out}}}}{\sqrt{K(\uVect,\uVect)}} \, \mbox{s.t.} \, K(\uVect, \uVect) > 0
\end{equation}
The above optimization\footnote{By the PH-1 property of $\Psi$, the mapping $\uVect \mapsto K(\uVect, \uVect)$ is homogeneous, implying that the objective is invariant to the transformation $\uVect \leftarrow c \uVect$ for any $c > 0$. We arbitrarily enforce $\|\uVect\|=1$ to keep the problem well-posed.} in $\uVect$ generally lacks a closed-form solution due to the non-linear nature of the kernel $K$. To maintain computational tractability, we introduce an approximation scheme that reduces the objective to an eigenvalue problem. Our approximation assumes that for any unit vector $\xVect$ and non-zero $\yVect$, the kernel factors as $K(\xVect,\yVect) = \normshort{\yVect} k\left( \inner{\xVect}{\frac{\yVect}{\normshort{\yVect}}} \right)$ for some angular function $k : [-1,1] \rightarrow \R$ bounded by $k(\rho) \leq 1$. Additionally, we require $k(1) = k'(1)$ and $k(1) > 0$. These conditions naturally hold for all PH-1 functions (piecewise linear with a single knot at the origin), e.g., ReLU, Leaky-ReLU; see Propositions \ref{prop:separability_kernel} to \ref{prop:kernel_positivity}. For highly correlated neuron pairings, the optimal parent direction $\uVect$ aligns closely with its children, pushing their cosine similarity $\rho \triangleq \inner{\uVect}{\frac{\tilde{\wVect}_{\text{in}}}{\normshort{\tilde{\wVect}_{\text{in}}}}}$ toward $1$. Expanding $k(\rho)$ to first order around $\rho = 1$ and applying the $k(1) = k'(1)$ identity cancels the constant terms, yielding the linear approximation $k(\rho) \approx \rho k(1)$. While this degrades for unaligned vectors, the phase-check provided later in the section corrects anti-alignment by flipping the sign of $\uVect$, ensuring the optimization trajectory remains safely within this linear domain.

Applying the linear approximation to the numerator of the objective, and defining the constant matrix $\Amat \triangleq \wVect^i_{\text{out}} (\tilde{\wVect}^i_{\text{in}})^T + \wVect^j_{\text{out}} (\tilde{\wVect}^j_{\text{in}})^T$, the summation factors neatly: $ \sum_{k \in \{i,j\}} K(\uVect, \tilde{\wVect}^k_{\text{in}}) \, \wVect^k_{\text{out}} \approx k(1) \Amat \uVect$. Conversely, the denominator requires no approximation; because $\normshort{\uVect}=1$, self-alignment $\rho=1$ evaluates to $K(\uVect, \uVect) = k(1)$. Substituting these into the original optimization problem gives: $ \widehat{\uVect} = \arg\max_{\uVect} \frac{\normshort{k(1) \Amat \uVect}}{\sqrt{k(1)}} \qquad\mbox{s.t.}\qquad \normshort{\uVect}=1 $. Because $k(1) > 0$, the scalars pull out. Dropping these constants and squaring the strictly non-negative objective simplifies the unconstrained problem to a standard quadratic form: $ \widehat{\uVect} = \arg\max_{\uVect} \, \uVect^T \Amat^T \Amat \uVect \qquad\mbox{s.t.}\qquad \normshort{\uVect}=1$. The optimal direction $\widehat{\uVect}$ is simply the principal eigenvector of $\Amat^T \Amat$. While explicitly constructing this ambient matrix is computationally prohibitive, $\Amat$ is fundamentally rank-2. Restricting the eigendecomposition to this rank-2 subspace bypasses the ambient dimension entirely, yielding the principal eigenvector via a fast closed-form solution.

\noindent\textbf{Determining the Sign of $\uVect$.} For PH-1 activations, the kernel is sign-sensitive $K(\uVect, \tilde{\wVect}_{\text{in}}) \neq K(-\uVect, \tilde{\wVect}_{\text{in}})$. However, because our linearization approximation relies on the leading eigenvector $\widehat{\uVect}$ of $\Amat^T \Amat$, we only recover the solution up to a sign ambiguity. We resolve this by evaluating both candidate polarities $\pm\widehat{\uVect}$ in the exact, non-linearized objective (\ref{eq:opt_u_exact_main}):
\begin{equation}
\uVect_{\text{correct}} = \argmax_{\uVect \in \{\widehat{\uVect},-\widehat{\uVect}\}} \frac{\norm{\sum_{k \in \{i,j\}} K(\uVect, \tilde{\wVect}^k_{\text{in}}) \wVect^k_{\text{out}}}}{\sqrt{K(\uVect,\uVect)}} \,.
\end{equation}

\subsubsection{Optimal Scale}
Recall from (\ref{eq:psi_valid_exact})
that the unit-norm direction $\psi \in \calN$ is parameterized by unit vectors $\uVect$ and $\vVect$ as $\psi = \frac{\Psi(\uVect \cdot \tilde{\xVect})}{\sqrt{K(\uVect,\uVect)}} \vVect$. Substituting $f = s\psi$ into the merging cost \eqref{eq:nested} and defining constants $a \triangleq \|f_i\|_{\calH}^2 + \|f_j\|_{\calH}^2$, $b \triangleq \inner{\psi}{f_i+f_j}_{\calH}$, and $E_{\text{rem}} \triangleq E_a - \|f_i\|_{\calH} - \|f_j\|_{\calH}$, minimizing the squared objective reduces to the 1D problem $s^* = \argmin_{s > 0} \frac{2s^2 - 2bs + a}{(s+E_{\text{rem}})^2}$. Setting the derivative with respect to $s$ to zero yields the unique minimizer $s^* = \frac{a + b E_{\text{rem}}}{2 E_{\text{rem}} + b}$. This solution is also stable. By definition, the residual capacity $E_{\text{rem}} \ge 0$, and the prior phase-check ensures the alignment in function space $b > 0$. Thus, the denominator is strictly positive, guaranteeing a unique global minimum in the positive domain (simplifying cleanly to $s^* = a/b$ in the event of a total layer collapse where $E_{\text{rem}}=0$). Once the optimal scale $s^*>0$ is determined, the parent neuron is fully characterized as shown below.

\begin{insight}{Optimal Parent Neuron}
\begin{equation}
\label{eq:fp_star}
f_p^*(\tilde{\xVect}) = s^* \psi^*(\tilde{\xVect}) \quad,\quad \psi^*(\tilde{\xVect}) = \frac{\Psi(\uVect_{\text{c}} \cdot \tilde{\xVect})}{\sqrt{K(\uVect_{\text{c}},\uVect_{\text{c}})}} \vVect^* \quad,\quad s^* = \frac{\|f_i\|_{\calH}^2 + \|f_j\|_{\calH}^2 + E_{\text{rem}} \, \inner{\psi^*}{f_i+f_j}_{\calH}}{2 E_{\text{rem}} + \inner{\psi^*}{f_i+f_j}_{\calH}}
\end{equation}

\begin{equation}
\uVect_{\text{c}} = \argmax_{\uVect \in \{\widehat{\uVect},-\widehat{\uVect}\}} \frac{\normshort{\sum_{k \in \{i,j\}} K(\uVect, \tilde{\wVect}^k_{\text{in}}) \wVect^k_{\text{out}}}}{\sqrt{K(\uVect,\uVect)}} \quad,\quad \vVect^* = \frac{\sum_{k \in \{i,j\}} K(\uVect_{\text{c}}, \tilde{\wVect}^k_{\text{in}}) \wVect^k_{\text{out}}}{\normshort{\sum_{k \in \{i,j\}} K(\uVect_{\text{c}}, \tilde{\wVect}^k_{\text{in}}) \wVect^k_{\text{out}}}}
\end{equation}

\begin{equation}
\widehat{\uVect} = \arg\max_{\normshort{\uVect}=1} \, \norm{\left( \wVect^i_{\text{out}} (\tilde{\wVect}^i_{\text{in}})^T + \wVect^j_{\text{out}} (\tilde{\wVect}^j_{\text{in}})^T \right) \uVect} \,.
\end{equation}

\end{insight}

\subsection{From Hilbert Space to Physical Parameters}
This section bridges the \textbf{abstract function space} and the \textbf{physical parameter space} by mapping the mathematical operator derived in $\calH$ back into physical parameters. This parameter recovery is only necessary for merging. For pruning, the projection target is simply the null operator $\zeroVect$, which leads to $f(\xVect) = 0$; this is trivially realized by zeroing out the neuron's incoming weights, outgoing weights, and BN parameters. However, deploying the parent neuron $f_p^* \in \calH$ derived in (\ref{eq:fp_star}) requires determining the physical parameters (weights $\wVect^\text{raw}_p, b_p, \wVect_{p, \text{out}}$ and BN statistics $\beta_p, \gamma_p, \mu_p, \sigma_p$) that will configure the forward pass to reproduce its targeted non-zero activation profile.

\subsubsection{Input/Output Scaling}
To form a standard realizable neuron as described in (\ref{eq:realizable_neurons}), we equate $f_p^*(\tilde{\xVect}) = \wVect_{\text{out}}^* \Psi(\tilde{\wVect}_{\text{in}}^* \cdot \tilde{\xVect})$ and then specify the parameters $\tilde{\wVect}^*_{\text{in}}$ and $\wVect_{\text{out}}^*$. Because the PH-1 activation $\Psi$ exhibits scale symmetry, the amplitude $s^* / \sqrt{K(\uVect^*,\uVect^*)}$ can be factored into arbitrary input and output scales $\tilde{\wVect}_{\text{in}}^* = s_\text{in} \uVect^*$ and $\wVect_{\text{out}}^* = s_\text{out} \vVect^*$, for any $s_\text{in},s_\text{out} \geq 0$ satisfying $s_\text{in} s_\text{out} = s^* / \sqrt{K(\uVect^*,\uVect^*)}$. While any factorization yields the same mapping $\calX \rightarrow \R^c$, amplitude distribution impacts fine-tuning dynamics. To preserve the original layer's balance, we define the subspace Frobenius ratio $R_F \triangleq \normshort{\Wmat_{\text{in}}}_F / \normshort{\Wmat_{\text{out}}}_F$, where $\Wmat_{\text{in}} = [\tilde{\wVect}^i_{\text{in}} \mid \tilde{\wVect}^j_{\text{in}}]$ and $\Wmat_{\text{out}} = [\wVect^i_{\text{out}} \mid \wVect^j_{\text{out}}]$. Constraining the parent neuron to this ratio requires $\normshort{\tilde{\wVect}^*_{\text{in}}}_2 / \normshort{\wVect^*_{\text{out}}}_2 = R_F$. Since $\uVect^*$ and $\vVect^*$ are unit vectors, $s_\text{in} / s_\text{out} = R_F$. This uniquely determines the scale factors, yielding the final parameters:
\begin{equation}
\tilde{\wVect}_{\text{in}}^* = \sqrt{s^* R_F} \cdot K_\text{self}^{-1/4}\, \uVect^* \qquad,\qquad \wVect_{\text{out}}^* = \sqrt{\frac{s^*}{R_F}} \cdot K_\text{self}^{-1/4} \, \vVect^* \qquad,\qquad K_\text{self} \triangleq K(\uVect^*,\uVect^*) \,.
\end{equation}

\subsubsection{Raw Input and BN Parameters}
\label{sec:bn_invert}
While the Hilbert space formulation operates entirely on the \textbf{effective} input parameters $\tilde{\wVect}_{\text{in}} \triangleq (\wVect_{\text{in}}^{\text{eff}}, b)$, realizing the physical network requires recovering the underlying physical parameters: $\wVect^\text{raw}, \beta, \gamma, \mu,$ and $\sigma$. Since the parent direction $\widehat{\uVect}$ lies within the 2D subspace spanned by the augmented children, there exist projection coefficients $c_1$ and $c_2$ that produce the effective parameters:
\begin{equation}
\wVect_{p, \text{in}}^{\text{eff}} = c_1 \wVect_{\text{in},i}^{\text{eff}} + c_2 \wVect_{\text{in},j}^{\text{eff}} \quad \text{and} \quad b_p = c_1 b_i + c_2 b_j \,.
\end{equation}
By mapping these coefficients through the pre-activation distributions of the children, we can deduce the required BN statistics for the parent neuron. Because the physical BN equations form an under-constrained system, we resolve the ambiguity by anchoring the variance such that $\sigma_p^2 = \max(0, \gamma_p^2 - \epsilon)$. As rigorously derived in Appendix~\ref{app:bn_inversion_details}, this anchoring yields a closed-form recovery of all physical parameters. For the active regime $\gamma_p^2 \ge \epsilon$, these evaluate to:
\begin{equation}
\wVect_{\text{in},p}^{\text{raw}} =  \wVect_{\text{in},p}^{\text{eff}} \qquad , \qquad \mu_p = c_1 \beta_i + c_2 \beta_j - b_p
\end{equation}
\begin{equation}
\beta_p = c_1 \beta_i + c_2 \beta_j \qquad,\qquad \sigma_p = \gamma_p = \sqrt{c_1^2 \gamma_i^2 + c_2^2 \gamma_j^2 + 2 c_1 c_2 |\gamma_i| |\gamma_j| \hat{\rho}_{ij}} \,.
\end{equation}
where $\hat{\rho}_{ij}$ is from (\ref{eq:rho_hat_maintext}). Note that $\sigma_p \approx \gamma_p$ is an approximation that assumes the numerical stability constant $\epsilon$ is negligible. The exact boundary-safe formulation $\sigma_p^2 = \max(0, \gamma_p^2 - \epsilon)$ and the edge case for inactive features $\gamma_p^2 < \epsilon$ are deferred to Appendix~\ref{sec:numerical_stability_BN}. Furthermore, Appendix~\ref{app:bn_inversion_details} provides the full step-by-step derivation, along with a proof demonstrating that the physical forward pass acts as a self-correcting mechanism that ensures the network's mapping remains invariant to the sign of the recovered scale $\gamma_p$.

\section{Block Eviction}
\label{sec:block_eviction}
This section expands the granular compression cost $\calJ_{\text{bound}}$ established in Section~\ref{sec:layer_transition_costs} to a new macro-level operation: \textbf{block eviction}. Focusing on residual blocks in architectures like ResNet-50, we extend the previously developed continuous integral to evaluate block eviction alongside granular operations within a single, unified mathematical framework.

Consider the canonical residual block, which processes an input representation $X$ (capitalized to distinguish it from the flattened vector $\xVect$) through a three-stage mapping pathway $F(X)$. This pathway sequentially applies weight parameters $W_1$, $W_2$, and $W_3$, and the result is added to a skip connection to yield the final pre-activation $Y = X + F(X)$. We define \textbf{Block Eviction} as forcing $F(X) \to \zeroVect$, and thus collapsing the block into a pure identity mapping $Y = X$ (see Figure~\ref{fig:block_eviction}).

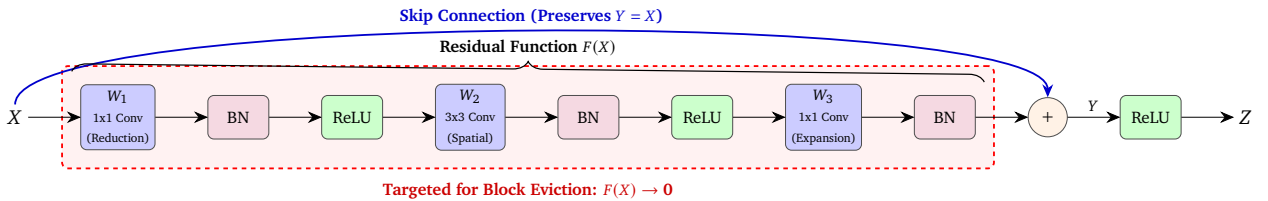
\begin{figure}[htbp]
    \centering
    \resizebox{\textwidth}{!}{%
    \begin{tikzpicture}[
        % --- Style Definitions ---
        node distance=1.2cm and 1.2cm,
        % Updated styles for multi-line text
        conv_style/.style={rectangle, rounded corners, draw=black!70, fill=blue!20, minimum height=1.2cm, minimum width=1.6cm, align=center},
        bn_style/.style={rectangle, rounded corners, draw=black!70, fill=purple!15, minimum height=1cm, minimum width=1.4cm, text centered},
        relu_style/.style={rectangle, rounded corners, draw=black!70, fill=green!20, minimum height=1cm, minimum width=1.4cm, text centered},
        op/.style={circle, draw=black!70, fill=orange!10, minimum size=0.9cm, text centered},
        data/.style={font=\Large\bfseries},
        arrow/.style={-{Stealth[length=3mm, width=2.5mm]}},
        brace/.style={decorate, decoration={brace, amplitude=10pt}},
        arrowlabel/.style={font=\small},
        % New style for indicating eviction
        evict_box/.style={rectangle, fill=red!5, draw=red, dashed, very thick, rounded corners, inner sep=12pt}
    ]

        %################################################################
        %### Base ResNet Block Diagram ##################################
        %################################################################
        
        % --- Input ---
        \node[data] (x_in) {$X$};

        % --- Residual Path F(x) ---
        % Mapped explicitly to W1, W2, W3 as described in Section 8
        \node[conv_style] (conv_a) [right=of x_in] {$W_1$ \\ \scriptsize 1x1 Conv \\ \scriptsize (Reduction)};
        \node[bn_style]   (bn_a)   [right=of conv_a] {BN};
        \node[relu_style] (relu_a) [right=of bn_a] {ReLU};
        
        \node[conv_style] (conv_b) [right=of relu_a] {$W_2$ \\ \scriptsize 3x3 Conv \\ \scriptsize (Spatial)};
        \node[bn_style]   (bn_b)   [right=of conv_b] {BN};
        \node[relu_style] (relu_b) [right=of bn_b] {ReLU};
        
        \node[conv_style] (conv_c) [right=of relu_b] {$W_3$ \\ \scriptsize 1x1 Conv \\ \scriptsize (Expansion)};
        \node[bn_style]   (bn_c)   [right=of conv_c] {BN};

        % --- Summation ---
        \node[op] (add) [right=of bn_c] {+};
        
        % --- Final Activation and Output ---
        \node[relu_style] (relu_out) [right=of add] {ReLU};
        \node[data] (y_out) [right=of relu_out] {$Z$};

        % --- Data Flow Arrows ---
        \draw[arrow] (x_in) -- (conv_a);
        \draw[arrow] (conv_a) -- (bn_a);
        \draw[arrow] (bn_a) -- (relu_a);
        \draw[arrow] (relu_a) -- (conv_b);
        \draw[arrow] (conv_b) -- (bn_b);
        \draw[arrow] (bn_b) -- (relu_b);
        \draw[arrow] (relu_b) -- (conv_c);
        \draw[arrow] (conv_c) -- (bn_c);
        \draw[arrow] (bn_c) -- (add);
        \draw[arrow] (add) -- node[above] {$Y$} (relu_out);
        \draw[arrow] (relu_out) -- (y_out);
        
        % --- Skip Connection Path (Emphasized) ---
        \draw[arrow, very thick, blue!80!black] (x_in) .. controls +(up:2.5cm) and +(up:2.5cm) .. 
            node[above, arrowlabel, font=\bfseries] {Skip Connection (Preserves $Y = X$)} (add);

        % --- Eviction Highlights (Background Layer) ---
        \pgfdeclarelayer{background}
        \pgfsetlayers{background,main}
        \begin{pgfonlayer}{background}
            % Draw the eviction boundary
            \node[evict_box, fit=(conv_a)(bn_c)(relu_b)] (eviction_boundary) {};
            % Add a label below the boundary
            \node[below=5pt of eviction_boundary, font=\bfseries\color{red!80!black}] {Targeted for Block Eviction: $F(X) \to \zeroVect$};
        \end{pgfonlayer}

        % --- Brace for F(x) ---
        \draw[brace, thick] ($(conv_a.north west)+(-0.2,0.3)$) -- ($(bn_c.north east)+(0.2,0.3)$) 
            node[midway, above=10pt, font=\bfseries] {Residual Function $F(X)$};

    \end{tikzpicture}%
    }
    \caption{Illustration of the canonical ResNet V1 residual block and its eviction process. Eviction forces the internal pathway $F(X) \to \zeroVect$, collapsing the block into a pure identity pre-activation $Y = X$.}
    \label{fig:block_eviction}
\end{figure}

\subsection{Motivation}
A dedicated macro-level operation is required because standard granular pruning cannot remove the block's final layer $W_3$. The output dimensionality of $F(X)$ must match the skip connection $X$ for element-wise addition. Consequently, granular compression can only deplete the internal layers $W_1,W_2$, which leaves the output channels of $W_3$ locked at their ambient size. Leaving a residual pathway active under these conditions creates two issues:
\begin{itemize}
    \item \textbf{Model Generalization:} When $W_1$ and $W_2$ are heavily depleted, the pathway functionally reduces to injecting an uncalibrated BN effective bias $B_{\text{eff}}$ into the skip connection $Y = X + B_{\text{eff}}$. This shifts downstream feature maps out of their calibrated domain, often causing catastrophic ReLU clipping and irreversible information loss when going from $Y$ to $Z$.
    \item \textbf{Execution Efficiency:} Retaining the massive $W_3$ parameter tensor simply to process a negligible subspace violates the core objective of compression.
\end{itemize}
Block eviction resolves both issues by projecting the pathway $F(X)$ to the null operator. By yielding a pure identity mapping $Y = X$, we avoid uncalibrated bias injection and leverage the fact that residual architectures are inherently designed to be robust to identity mappings (e.g., standard $\gamma=0$ initialization practices) \cite{goyal2017accurate, he2016identity}. Full mathematical details of this degradation are provided in Appendix~\ref{sec:app_block_eviction}.

\subsection{The Unified Macro Cost $\calJ_{\text{evict}}$}
To evaluate this macro-operation within our framework, we must expand our definition of layer state. To see why, observe that Axiom 2 imposes an infinite cost penalty on projecting an entire layer to zero to prevent disconnecting the network graph. However, this penalty creates an artificial barrier here, as the parallel identity mapping preserves overall connectivity of the block and keeps it alive. To account for this skip pathway, we formulate a macroscopic state $\Omega^{(l)} \triangleq (\Phi^{(l)}, \calI)$ that couples the targeted internal layer $\Phi^{(l)}$ with the ambient skip connection $\calI$.

The skip connection provides a parallel survival capacity $E_{\text{identity}}$ that keeps the mathematical projection stable. As rigorously derived in Appendix~\ref{sec:app_block_eviction}, integrating the continuous capacity cost over this macro-state and applying a linear upper bound to safely govern massive discrete architectural leaps yields a closed-form distortion criterion. For a standard residual bottleneck comprising two internal convolution layers $l \in \{1, 2\}$, the total macroscopic distortion is the linear sum of their independent projection bounds:
\begin{equation}
    \calJ_{\text{evict}} = \sum_{l=1}^{2} \calJ_{\text{layer}}(\Omega_a^{(l)}, \Omega_b^{(l)}) = \sum_{l=1}^{2} N_{\text{active}}^{(l)} \left( \frac{E_{\text{active}}^{(l)}}{E_{\text{identity}}} \right) \,.
\end{equation}
Here, $N_{\text{active}}^{(l)}$ and $E_{\text{active}}^{(l)}$ represent the active operator count and surviving capacity of internal layer $l$, respectively. The parallel survival capacity evaluates to the expected RMS energy of the identity operators conditioned by the preceding BN layer: $E_{\text{identity}} = \sum_{k=1}^{d_{\text{amb}}} \sqrt{\gamma_k^2 + \beta_k^2}$.

\begin{insight}{ResNet Block Eviction Cost}
\begin{equation}
    \calJ_{\text{evict}} = \frac{\sum_{l=1}^{2} N_{\text{active}}^{(l)} \, E_{\text{active}}^{(l)}}{\sum_{k=1}^{d_{\text{amb}}} \sqrt{\gamma_k^2 + \beta_k^2}} \,.
\end{equation}
\end{insight}

\section{Balancing Compression and Distortion}
\label{sec:rate_distortion}

All cost functionals discussed thus far ($\calJ_{\text{prune}}$ and $\calJ_{\text{merge}}$ for granular reductions, and $\calJ_{\text{evict}}$ for block evictions) measure the projection error incurred when transitioning from a given state to a reduced state. However, rate-distortion theory establishes that distortion alone cannot fully characterize a lossy compression scheme: lower signal distortion requires a higher bit count, while stronger compression inevitably increases distortion. To balance these competing objectives, we aim to minimize total distortion under a fixed bit count budget.

Progressive compression is therefore formulated as a trajectory planning problem within the action space. The goal is to craft a sequence of compression operations that yields a final model satisfying the allowable bit budget while minimizing the total accumulated distortion along the trajectory. Solving this represents a highly complex planning problem due to two primary challenges:
\begin{enumerate}
    \item \textbf{Dynamic State Dependency:} The cost of an action changes continuously as the model transitions between states. For example, pruning a single neuron shrinks the layer's residual capacity, which instantaneously alters the cost $\calJ$ of subsequent operations, such as pruning another neuron or evicting an entire block. Consequently, the mathematical cost landscape is constantly shifting.
	\item \textbf{Mutually Exclusive Actions:} The action space contains complex combinatorial dependencies. If the optimizer merges Neuron A with Neuron B, independent actions like ``Prune A'' or ``Merge A with C'' become permanently invalid.
\end{enumerate}
For computational tractability, we must relax these constraints. At each iteration, we temporarily assume all currently admissible operations will remain valid for future iterations, ignoring their mutually exclusive nature. While this generates a complete theoretical action sequence, executing the full trajectory would introduce compounding errors in both state transitions and capacity counts. Instead, we adopt a \textbf{receding-horizon strategy} \cite{camacho2013model, bertsekas2012dynamic}: we compute the optimal sequence, but execute only the immediate next action. Then we physically update the network and then re-evaluate all admissible functions from scratch. This single-step execution acts as an inherent auto-correction mechanism that ensures adherence to constraints over each short-term step.

Formally, let $\calA = \{1, 2, \dots, K\}$ denote the set of all admissible compression operations at the current encoding iteration, encompassing all feasible granular and macro operations. Each action $k$ incurs a distortion penalty $\calJ_k$ and releases $\Delta P_k$ parameters (see Appendix~\ref{sec:delta_P} for details on computing $\Delta P$). Assuming a standard fixed-precision representation (e.g., 32-bit floating-point), bit reduction is proportional to parameter reduction. This direct scaling allows us to express the allowable budget directly in terms of the parameter footprint. We frame this optimization as:
\begin{equation}
(a_1^*,\cdots,a_K^*) = \arg\min_{a_1,\cdots,a_K} \sum_{k=1}^K a_k \calJ_k \qquad\mbox{s.t.}\qquad \sum_{k=1}^K a_k \Delta P_k \geq P_0 - P_{\text{budget}} \quad,\quad \forall k \,;\, a_k \in \{0,1\} \,.
\end{equation}
where $P_0$ is the initial parameter count and $P_{\text{budget}}$ is the maximum allowable parameter footprint for the final model. This formulation is a \textbf{discrete knapsack problem}, which is well-known to be NP-Hard \cite{karp1972reducibility, garey1979computers}. We resolve this using a continuous relaxation heuristic, replacing the binary constraint $a_k \in \{0,1\}$ with a continuous bound $0 \leq a_k \leq 1$. This transforms the objective into a \textbf{continuous knapsack problem} (a specific class of linear programming) that admits a highly efficient analytical solution. As established by Dantzig \cite{dantzig1957discrete}, the exact optimal solution is found greedily: candidates are sorted by their \textbf{distortion rate} (DR), defined as the cost-to-capacity ratio $\calJ_k / \Delta P_k$, and assigned $a_k = 1$ in ascending order until the budget constraint is saturated. Because our receding-horizon framework executes only the single next action, the problem reduces to selecting the operation with the minimal DR: 
\begin{equation}
k^* = \arg\min_{k \in \calA} \frac{\calJ_k}{\Delta P_k} \,.
\end{equation}

While the \textit{receding-horizon strategy} mitigates the dynamic dependency of $\calJ$, the continuous knapsack solver still requires $\Delta P$ to satisfy Dantzig's \textit{Axiom of Item Independence}: the weight of one item cannot depend on the selection state of another. Particularly in our problem, evaluating operations using the dynamically shrinking live parameter footprint $\Delta P$ violates this axiom because adjacent layers share weight matrices; pruning a neuron physically shrinks the $\Delta P$ of its neighbors.

A naive optimization using this live $\Delta P$ triggers a failure mode: as a layer is compressed, the expected DR of neighboring structures artificially inflates. This repels the optimizer and may trap the architecture in a fragmented state that prevents the removal of contiguous blocks. Decoupling parameter yield from dynamic state via the static surrogate $\Delta P_k^{\text{init}}$ restores item independence and avoids this failure mode:
\begin{insight}{Action Selection Criterion}
\begin{equation}
k^* = \arg\min_{k \in \calA} \frac{\calJ_k}{\Delta P_k^{\text{init}}}\,.
\end{equation}
\end{insight}

\section{The Encoding Loop}
\label{sec:progressive_encoder}
With the optimal action selection now formally defined, we execute progressive encoding as a greedy dynamical system. At a high level, the algorithm continuously identifies the optimal action $k^*$ offering the lowest DR $\calJ_k/\Delta P_k^{\text{init}}$, performs a localized recalculation exclusively for the modified structures (e.g., a newly generated parent neuron $f_p^*$) and their immediate neighbors, decrements the relevant dimension count, and repeats. Specifically, the process operates in the following three phases and terminates once the target physical parameter budget is reached or no admissible compression operations remain:

\textbf{1. Initialization:} Before compression begins, the algorithm precomputes and caches the individual capacities of all neurons, the pairwise geometric cross-capacities of all valid merging pairs, and the total initial capacity of every layer (establishing the starting value for $E_{\text{rem}}$).

\textbf{2. The Greedy Scan:} At each iteration, the algorithm scans all $L$ layers (each containing roughly $N$ active neurons) to find the single optimal compression action $k^*$ yielding the lowest DR $\calJ_k/\Delta P_k^{\text{init}}$. For pruning, evaluating every individual candidate across the network requires $\calO(L \cdot N)$ operations. For merging, evaluating every valid pair requires checking $\frac{N(N-1)}{2}$ combinations per layer, leading to $\calO(L \cdot N^2)$ operations. Because querying the cached $\calJ$ for each candidate takes $\calO(1)$ time\footnote{As established in the practical notes of Section~\ref{sec:layer_transition_costs}, the distortion cost $\calJ$ relies on local variables: the capacity of the targeted neurons and the remaining capacity of their specific layer. Because evaluating $\calJ$ does not require querying the global network state, calculating the DR of any individual prune, merge, or block eviction operation is $\calO(1)$.}, the total computational complexity to find the optimal action at any step is bounded by the pairwise merge evaluations at $\calO(L \cdot N^2)$.

\textbf{3. Localized Update:} Once the globally optimal action is identified and executed, the network state must be \textbf{synchronized}. The algorithm decrements the layer's $E_{\text{rem}}$ by the capacity flux removed by the operation, and decrements the neuron count $N$. If the action was a merge, the algorithm also computes the capacity of the newly generated parent neuron $f_p^*$ and calculates the cross-capacities as well as optimal projection vectors) between this new parent and the $N-1$ surviving neighbors in its layer. These updated constants are injected into the cache, guaranteeing that the evaluation of $\calJ$ during subsequent greedy scans remains $\calO(1)$. This limits the network state recalculation to an $\calO(N)$ local update.

\section{Proof-of-Concept Applications}
\label{sec:experiments}

\subsection{Model Compression}

Because its encoding is progressive, any intermediate iteration serves as a valid compressed model, providing users with flexible trade-offs between compression rate and fidelity. Taxonomically, HOPE is a structured method: it eliminates entire neurons rather than zeroing out individual weights. This provides greater practical utility than unstructured pruning, which generates randomly sparse matrices requiring specialized hardware to realize actual computational speedups. We compare HOPE against three structured baselines that eliminate neurons below specific magnitude thresholds: \textit{$L_1$-Norm Input Pruning} \cite{liu2017learning} (scored by incoming weight $L_1$ norms); \textit{$L_1$-Norm Joint Pruning} (scored by concatenated incoming and outgoing $L_1$ norms); and \textit{BN Scale Pruning} \cite{liu2017learning} (using the BN scaling factor $\gamma$ as a proxy for importance).

\begin{minipage}{0.55\textwidth}
\centering
\includegraphics[width=3in]{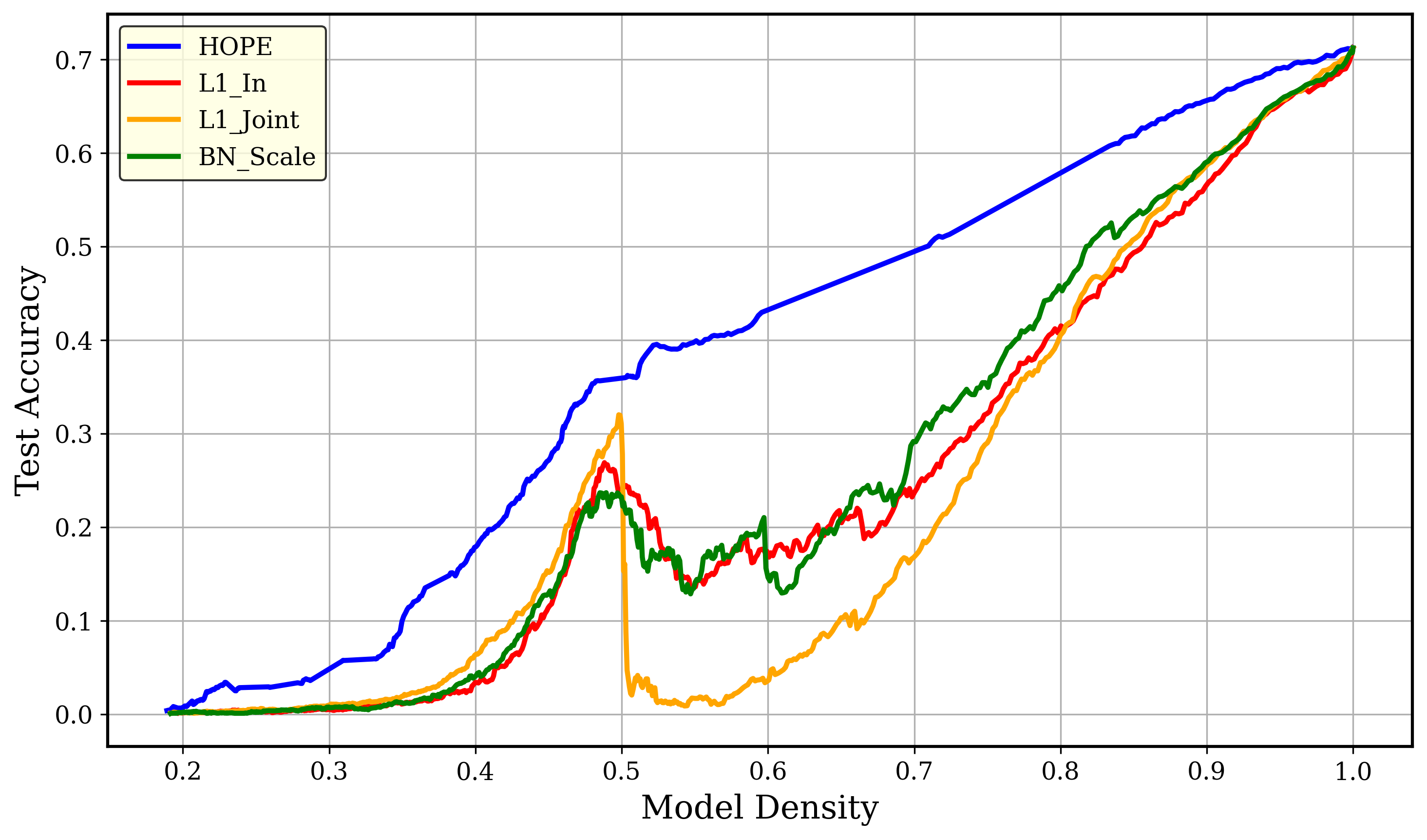}
\end{minipage}\hfill % \hfill pushes the two minipages apart
\begin{minipage}{0.40\textwidth}
Our experiments investigated the relationship between test set accuracy and model density (defined as the ratio of active to initial neurons across the entire network). For our compression assessment, we utilized Keras' publicly available ResNet-50 model checkpoint that is pre-trained on ImageNet. As demonstrated in the plot, HOPE yields models with superior accuracy compared to the baselines.
\end{minipage}

\subsection{Cross-Domain Transfer Learning}
\label{sec:deft_transfer}

HOPE’s capacity evaluation can be used for resolving the stability-plasticity dilemma in transfer learning. By merging redundancies, we can partition the network into a protected core and a plastic periphery and leverage it for parameter-efficient transfer.

\subsubsection{The Stability-Plasticity Dilemma and Current Bottlenecks}
Intelligent systems face a fundamental challenge: adapting to new domains without erasing foundational knowledge, a trade-off known as the Stability-Plasticity dilemma \cite{grossberg1987competitive}. Cognitive neuroscience models this via Complementary Learning Systems \cite{mcclelland1995why, kumaran2016what}, proposing that the brain insulates a stable, domain-specific core from a plastic periphery, as empirically supported by recent neuroimaging \cite{billot2024language, blank2016expanding, casto2026cerebellar}. Specifically, the brain extracts generalizable schemata from noisy state transitions by applying low-dimensional regularization to its representational geometry \cite{kimmel2026neural}.

During standard training, deep neural networks self-organize into a similar, albeit noisy and imperfect, dichotomy \cite{martin2017implicit}. They develop a sparse, load-bearing core surrounded by low-capacity representational slack \cite{frankle2019lottery, elcheairi2026theoretical}. Naively exploiting this emergent segregation by freezing the core fails; because layers remain entangled, peripheral updates shift activation flows, causing representational drift. Early solutions like PackNet \cite{mallya2018packnet} circumvented this using binary masks during inference, but these scale poorly and require \textit{a priori} task identities.

Robust continual learning requires explicit interventions, rather than relying on noisy emergent segregation. During source training, regularization \cite{wen2016learning, scardapane2017group} can amplify the segregation. During downstream adaptation, penalty-based methods like EWC \cite{kirkpatrick2017overcoming}, Synaptic Intelligence \cite{zenke2017continual}, and second-order pruning \cite{lecun1989optimal, hassibi1993optimal, singh2020woodfisher} prevent drift using \textit{locally convex} Fisher Information Matrices (FIMs) or Hessians. However, relying on local approximations makes these algorithms brittle to large domain shifts. Conversely, orthogonal projection methods stop drift by restricting updates to the null-space of previous tasks \cite{jaeger2014conceptors, saha2021gradient, zeng2019continuous, wang2021projecting, yang2025loranull, huggingface2026osf}. However, they remain computationally expensive due to the $\calO(N^3)$ operations and source forward passes needed to compute covariance matrices.

\subsubsection{DEFT}

To address the above challenges, we introduce \textbf{Dispersed Elastic Fine-Tuning (DEFT)}. Aligning with the Information Bottleneck principle \cite{tishby1999information}, DEFT treats learning as the compression of irrelevant slack space via global Hilbert-Schmidt operators, bypassing local loss curvature and empirical data passes. Leveraging the HOPE framework, DEFT analytically computes each neuron's capacity in $\calO(N)$ time to partition the network into a Universal Core and a Peripheral Slack. To prevent representational drift, DEFT severs weight projections from the slack to the core prior to transfer. This ensures the core remains frozen while the slack adapts, eliminating the need for inference-time masking or task identities.

DEFT governs parameter plasticity through a binary elasticity map $E \in \{0,1\}$. While prior methods also regulate plasticity \cite{zhou2024taco}, their reliance on weight sensitivity leaves them vulnerable to the scaling symmetries that HOPE mitigates. We formalize this by evaluating the pruning cost assigned to each neuron $i$ upon its removal during HOPE's progressive encoding process:
\begin{equation}
    \calJ_{\text{prune}}^{(i)} = \frac{N^{(i)} \cdot \|f_i\|_{\calH}}{E_b^{(i)}}
\end{equation}
where $N^{(i)}$ and $E_b^{(i)} \triangleq E_a^{(i)} - \normshort{f_i}_{\calH}$ denote the active neuron count and the remaining layer capacity, at the step neuron $i$ is pruned. To establish a global freezing threshold, we collect the set of all such costs across the entire encoding process, and filter out extinction artifacts resulting from near-zero capacities:
\begin{equation}
    \calC = \left\{ \calJ_{\text{prune}}^{(i)} \;\middle|\; E_b^{(i)} > \epsilon \right\}
\end{equation}
Given a target percentile hyperparameter $P \in [0, 100]$, we compute the threshold $J_P = \text{Percentile}(\calC, P)$ and the supremum $J_{\sup} = \max(\calC)$. For numerical stability against edge capacity regimes, the final locking threshold $J_{\text{lock}}$ is defined:
\begin{equation}
    J_{\text{lock}} = 
    \begin{cases} 
        J_P & \text{if } J_P \ge \epsilon \\
        J_{\sup} & \text{if } J_P < \epsilon \text{ and } J_{\sup} \ge \epsilon \\
        1 & \text{otherwise}
    \end{cases}
\end{equation}
The elasticity of neuron $i$ is formulated as:
\begin{equation}
E_i = \begin{cases} 
1 & \text{if } \calJ_{\text{prune}}^{(i)} < J_{\text{lock}} \\
0 & \text{if } \calJ_{\text{prune}}^{(i)} \ge J_{\text{lock}}
\end{cases}
\end{equation}
Under this formulation, high capacity neurons essential to the source architecture $\calJ_{\text{prune}}^{(i)} \ge J_{\text{lock}}$ are frozen by $E_i=0$, whereas low-capacity slack neurons are granted high plasticity via $E_i=1$.

\textbf{Dynamic Resolution of Redundancy:} Deep networks frequently fragment a single feature across multiple correlated neurons. If we freeze the network based on a static capacity threshold, we incorrectly lock up this redundant volume and deprive the target task of parameter space. \textbf{As illustrated in Figure~\ref{fig:deft_mechanisms}(a)}, DEFT resolves this by compressing these redundant features into a single rank-1 parent neuron. This consolidates the foundational source knowledge while releasing the freed child neurons into the plastic slack $E_i=1$. By transforming redundant copies into uncommitted parameter space, DEFT actively generates capacity for the target task.

\begin{figure}[htbp]
% This forces the text of the subcaptions to be centered
\captionsetup[subfigure]{justification=centering} 
\centering

% -----------------------------------------
% Left Subfigure: Redundancy Merge
% -----------------------------------------
\begin{subfigure}{0.48\textwidth}
    \centering
\begin{tikzpicture}[font=\sffamily, >=Stealth, scale=0.6, transform shape]
        % Core Box (y matched from -3 to 3)
        \path[drop shadow={opacity=0.2}, fill=coreColor!10, draw=black, thick, rounded corners] (-1, -3) rectangle (2, 3);
        \node[above, text=coreColor, font=\bfseries] at (0.5, 3.2) {Frozen Core};
        
        % Slack Box (y matched from -3 to 3)
        \path[drop shadow={opacity=0.2}, fill=slackColor!10, draw=black, thick, rounded corners] (6, -3) rectangle (9, 3);
        \node[above, text=slackColor, font=\bfseries] at (7.5, 3.2) {Plastic Slack};
        
        % Before (x restored to 1, y = 1.5, 0, -1.5)
        \node[circle, draw=black, top color=coreColor!20, bottom color=coreColor!40, drop shadow={opacity=0.2}, inner sep=5pt] (c1) at (1, 1.5) {$f_1$};
        \node[circle, draw=black, top color=coreColor!20, bottom color=coreColor!40, drop shadow={opacity=0.2}, inner sep=5pt] (c2) at (1, 0) {$f_2$};
        \node[circle, draw=black, top color=coreColor!20, bottom color=coreColor!40, drop shadow={opacity=0.2}, inner sep=5pt] (c3) at (1, -1.5) {$f_3$};
        \node[left=0.2cm of c2, align=right] {$M$ \\ Copies};

        % After (Aligned to y = 1.5, 0, -1.5)
        \node[circle, draw=black, top color=coreColor!60, bottom color=coreColor!90, drop shadow={opacity=0.4}, inner sep=8pt, font=\bfseries, text=white] (p) at (7.5, 1.5) {Parent};
        \node[circle, draw=black, fill=pink!40, inner sep=5pt, thick] (e1) at (7.5, 0) {Empty};
        \node[circle, draw=black, fill=pink!40, inner sep=5pt, thick] (e2) at (7.5, -1.5) {Empty};
        
        % Flow (Centered at y=0)
        \draw[->, ultra thick, black!60] (2.2, 0) -- (5.8, 0) node[midway, above, font=\bfseries, text=black] {MERGE} node[midway, below] {Error $\le \sum C_i$};
    \end{tikzpicture}
    \caption{Dynamic Resolution of Redundancy}
\end{subfigure}\hfill
% -----------------------------------------
% Right Subfigure: The Structural Mask
% -----------------------------------------
\begin{subfigure}{0.48\textwidth}
    \centering
    \begin{tikzpicture}[font=\sffamily, >=Stealth, scale=0.6, transform shape]
        % Layer boxes with gradients and shadows (y matched from -3 to 3)
        \path[drop shadow={opacity=0.3, shadow xshift=2pt, shadow yshift=-2pt}, fill=yellow!30, draw=black, thick, rounded corners] (-3, -3) rectangle (-1, 3);
        \node[font=\bfseries, text=black!70] at (-2, 3.2) {Layer $l$};
        
        \path[drop shadow={opacity=0.3, shadow xshift=2pt, shadow yshift=-2pt}, fill=yellow!30, draw=black, thick, rounded corners] (2, -3) rectangle (4, 3);
        \node[font=\bfseries, text=black!70] at (3, 3.2) {Layer $l+1$};

        % Nodes (Aligned to y = 1.5, -1.5)
        \node[circle, draw=black, thick, top color=coreColor!10, bottom color=coreColor!30, drop shadow={opacity=0.2}, minimum size=1cm, font=\bfseries] (L1C) at (-2, 1.5) {Core};
        \node[circle, draw=black, thick, top color=slackColor!10, bottom color=slackColor!30, drop shadow={opacity=0.2}, minimum size=1cm, font=\bfseries] (L1S) at (-2, -1.5) {Slack};
        \node[circle, draw=black, thick, top color=coreColor!10, bottom color=coreColor!30, drop shadow={opacity=0.2}, minimum size=1cm, font=\bfseries] (L2C) at (3, 1.5) {Core};
        \node[circle, draw=black, thick, top color=slackColor!10, bottom color=slackColor!30, drop shadow={opacity=0.2}, minimum size=1cm, font=\bfseries] (L2S) at (3, -1.5) {Slack};

        % Paths
        \draw[->, thick, coreColor] (L1C) -- (L2C) node[midway, above] {$w_{\text{core} \to \text{core}}$};
        \draw[->, thick, slackColor] (L1S) -- (L2S) node[midway, below] {$w_{\text{slack} \to \text{slack}}$};
        \draw[->, thick, gray!60] (L1C) -- (L2S);
        
        % The Cut Path
        \draw[->, thick, danger, dashed] (L1S) -- (L2C) node[midway, sloped, above, font=\bfseries] {Severed at $t=0$};
    \end{tikzpicture}
    \caption{The Structural Mask}
\end{subfigure}

\caption{The algorithmic mechanisms of DEFT. (a) Redundant features within the frozen core are compressed to generate new elastic target capacity. (b) A structural mask permanently severs cross-connections to protect the core from target-driven drift.}
\label{fig:deft_mechanisms}
\end{figure}
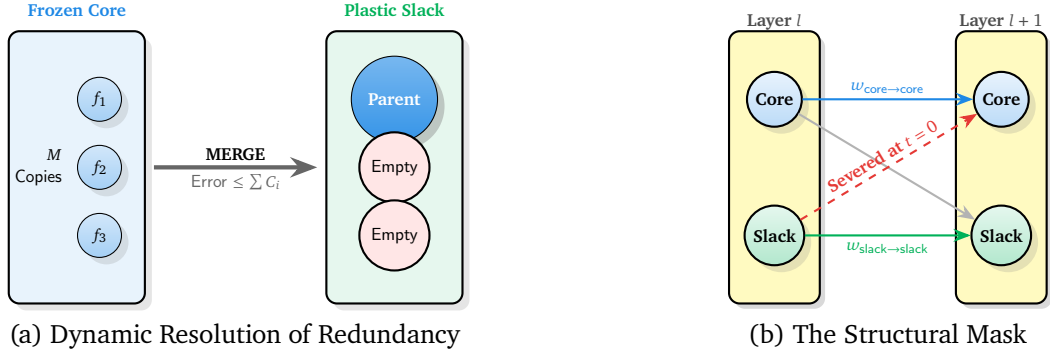

\textbf{Consistency at Initialization (The Structural Mask):} To prevent target-driven updates of the plastic slack from corrupting the frozen core, DEFT applies a structural mask at initialization \textbf{(Figure~\ref{fig:deft_mechanisms}b)}. It severs all connections pointing from upstream plastic neurons to downstream frozen core neurons. For a weight tensor connecting an upstream layer (with elasticities $\Emat_{\text{in}}$) to a downstream layer (with elasticities $\Emat_{\text{out}}$), the mask $\Mmat$ enforces:
\begin{equation}
    M_{j,k} = 
    \begin{cases} 
        0 & \text{if } E_{\text{in}, k} > 0 \text{ and } E_{\text{out}, j} = 0 \\
        1 & \text{otherwise}
    \end{cases}
\end{equation}
The initial weights for the target task are thus constrained to $\Wmat_0 = \Mmat \odot \Wmat_{\text{source}}$. 

\textbf{Theoretical Guarantees:} We prove in Appendix~\ref{app:theoretical_justification} that these mechanisms protect the source representation through a layer-to-layer bounding framework. First, the \textbf{Static Initialization Shock Bound} establishes that severing the slack-to-core connections introduces a static error strictly bounded by $\calO(\tau)$. Second, \textbf{Dynamic Decoupling} guarantees the core experiences zero dynamic interference during target fine-tuning. Because the cross-connections are severed at initialization and their weights are frozen from updating, any drifting signal from the learning slack subset is multiplied by zero, nullifying it before it can penetrate the core. Combined with the bounded projection errors of the merging operation, this framework ensures the cumulative degradation of the source task cannot compound exponentially, remaining anchored to an algorithmically verifiable constant.

\textbf{Gradient Scaling:} During optimization, the target gradients are element-wise scaled by the downstream neuron's elasticity $\Emat_{\text{out}}$ (uniformly broadcast across the input channels):
\begin{equation}
    \gVect_t = \Emat_{\text{out}} \odot \nabla_{\Wmat} \calL_{\text{target}}(\Wmat_t)
\end{equation}
where $\gVect_t$ represents the effective gradient passed to the optimizer state\footnote{ Scaling the gradient \textit{before} the optimizer step prevents velocity drift for frozen parameters.}, e.g., SGD with momentum.

\subsubsection{Experimental Setup}

To evaluate the stability-plasticity tradeoff, we rank methods using the \textbf{H-Score} commonly used in continual learning \cite{qiu2024continual, xie2025adadare, islam2025latest}. H-Score, defined as the harmonic mean of source retention and target accuracy, heavily penalizes poor performance in either domain. This ensures a high score is achieved only when a model excels on both tasks.

We evaluate models pre-trained on multi-class classification tasks derived from the CIFAR-100 dataset \cite{krizhevsky2009learning}. Each source task is constructed by randomly sampling 4 superclasses, which yields 20 fine-grained categories (5 per superclass). Building the source task around dense semantic clusters rather than sampling 20 arbitrary classes forces the network to learn hierarchical features to distinguish closely related concepts. We then transfer these specialized models to the full 10-class digit classification task in the SVHN dataset (street-level house numbers) \cite{netzer2011reading}.

As summarized in Table~\ref{tab:method_comparison}, we benchmark DEFT against the following baseline methods:
\begin{itemize}
    \item \textbf{Standard Full FT:} The entire pre-trained backbone is unfrozen, allowing the optimizer to alter representations across all layers. While maximizing target plasticity, it is highly susceptible to catastrophic forgetting.
    \item \textbf{Head-Only FT (Standard FT):} Representing the opposite extreme (linear probing), the pre-trained backbone is completely frozen, acting as a static feature extractor. Only the final linear classification head is optimized.
    \item \textbf{PEFT (BN-Tuning):} Operating on the premise that spatial feature extraction logic should remain invariant \cite{frankle2020training}, this method applies a binary gradient mask: core convolutional and dense kernels are frozen, while plasticity is isolated entirely to the affine BN parameters (scale $\gamma$ and shift $\beta$) and the newly initialized classification head.
    \item \textbf{EWC (Elastic Weight Consolidation):} Allows all parameters to update but applies a quadratic penalty constraining parameters deemed critical to the source task \cite{kirkpatrick2017overcoming}. To accurately lock foundational features, the empirical diagonal Fisher Information Matrix (FIM) is computed strictly over the source training dataset. See Appendix \ref{app:transfer_details} for the per-example FIM derivation and integration protocol.
\end{itemize}

Table~\ref{tab:method_comparison} provides a conceptual summary of these methodologies. Full details regarding network architecture, hyperparameter optimization, and reproducibility protocols are provided in Appendix~\ref{app:transfer_details}.

\begin{table}[ht]
\centering
\caption{Comparative Summary of Transfer Learning Methodologies}
\vspace{0.2cm}
\begin{tabular}{@{}lccccc@{}}
\toprule
\textbf{Property} & \textbf{Full FT} & \textbf{Head-Only} & \textbf{PEFT} & \textbf{EWC} & \textbf{DEFT (Ours)} \\
\midrule
Updates Backbone Features        & \checkmark & - & - & \checkmark & \checkmark \\
Mitigates Catastrophic Forgetting& - & \checkmark & \checkmark & \checkmark & \checkmark \\
Parameter-Specific Modulation    & - & - & - & \checkmark & \checkmark \\
Source Data Independence         & \checkmark & \checkmark & \checkmark & - & \checkmark \\
Structure \& Redundancy Aware    & - & - & - & - & \checkmark \\
\bottomrule
\end{tabular}
\label{tab:method_comparison}
\end{table}

\subsubsection{Results and Discussion} 
The final test set evaluations across $4$ independent cross-domain trials ($20$ discrete CIFAR-100 $\to$ SVHN scenarios) are presented in Table \ref{tab:results}. While Standard Full FT achieves the highest target performance $94.09\%$ by freely overwriting network weights, it completely destroys the source representation, crashing source retention to a baseline low of $7.52\%$. Conversely, Head-Only FT best preserves source knowledge, but its $36.11\%$ target accuracy highlights the severe domain gap; pre-trained features are insufficient to separate SVHN digits linearly. EWC behaves remarkably similarly to Full FT: it learns the target domain ($93.94\%$) but fails to significantly arrest catastrophic forgetting ($6.74\%$). DEFT successfully bridges this stability-plasticity gap. By routing target gradients into elastic neurons characterized by their low capacity, DEFT captures nearly all the plasticity of Standard Full FT ($94.09\%$ vs. $89.79\%$). Simultaneously, by masking the core foundational features, it halts catastrophic forgetting, retaining $52.14\%$ of the source accuracy. Computing the harmonic mean of the two accuracy values leads to the H-Score of method, in which DEFT significantly outperform all the baselines.

\begin{table}[ht]
\centering
\caption{Cross-Domain Transfer Learning Results (4 Trials, 5 tasks each). Metrics represent Test Set accuracy evaluated at the optimal target validation epoch. All metrics are averaged across all tasks and trials.}
\vspace{0.2cm}
\begin{tabular}{@{}lccc@{}}
\toprule
\textbf{Method} & \textbf{Target Acc (SVHN)} & \textbf{Source Retention (CIFAR)} & \textbf{H-Score} \\
\midrule
\textbf{DEFT} (Ours) & $89.79 \pm 0.84$ & $52.14 \pm 5.29$ & $\mathbf{65.82 \pm 3.96}$ \\
\textbf{Head-Only} & $36.11 \pm 2.79$ & $63.13 \pm 4.62$ & $45.79 \pm 2.05$ \\
\textbf{Full FT} & $94.09 \pm 0.28$ & $7.52 \pm 1.63$ & $13.88 \pm 2.84$ \\
\textbf{EWC} & $93.94 \pm 0.22$ & $6.74 \pm 1.74$ & $12.54 \pm 2.99$ \\
\textbf{PEFT} & $81.91 \pm 0.49$ & $5.44 \pm 0.98$ & $10.18 \pm 1.63$ \\
\bottomrule
\end{tabular}
\label{tab:results}
\end{table}

\section{Acknowledgment}
We thank Juno Kim, Vaishnavh Nagarajan, Atish Agarwala, Spencer Frei, Lisa Schut, Alan Malek, Gil shamir, and Bruno Mlodozeniec of Google DeepMind for their helpful comments and discussions.

\newpage

\bibliography{curr_learn}
\bibliographystyle{apalike}

\clearpage
\appendix

\section*{Appendix Table of Contents}
\startcontents[appendices]
\printcontents[appendices]{}{1}{\setcounter{tocdepth}{2}}
\newpage

\section{Hilbert Spaces}
\label{sec:hilbert}
This section provides a quick introduction to Hilbert spaces, focusing on the concepts of inner products, completeness, and total sets. We bridge abstract functional analysis with the specific architectural choices of the HOPE framework, demonstrating how the $L_2$ embedding of neural functions creates a unique and complete geometric environment for optimization over network structures.

\subsection{Introduction: Why a Hilbert Space?}

We treat compression operations like pruning and merging as \textit{projections}. When we replace two neurons with one, we are attempting to find a single element that "represents" a multi-dimensional subspace. To perform this operation rigorously, we need three things:
\begin{enumerate}
    \item A \textbf{Space} that contains all possible neural identities.
    \item A \textbf{Metric} (Inner Product) to measure "closeness" and "alignment."
    \item \textbf{Completeness} to ensure that our optimizations actually have solutions.
\end{enumerate}
A Hilbert space $\calH$ provides these three pillars.

\subsection{The Inner Product: The Ruler of Geometry}

The defining feature of a Hilbert space is the \textbf{inner product}. While a vector space only lets us add and scale elements, an inner product space lets us talk about \textit{angles} and \textit{lengths}.

\begin{definition}[Inner Product Axioms]
An inner product on a vector space $V$ over $\R$ is a mapping $\inner{\cdot}{\cdot} : V \times V \to \R$ satisfying for all $f, g, h \in V$ and $a \in \R$:
\begin{itemize}
    \item \textbf{Symmetry:} $\inner{f}{g} = \inner{g}{f}$.
    \item \textbf{Linearity:} $\inner{af + g}{h} = a\inner{f}{h} + \inner{g}{h}$.
    \item \textbf{Positive Definiteness:} $\inner{f}{f} \geq 0$, and $\inner{f}{f} = 0 \iff f = 0$.
\end{itemize}
\end{definition}

\subsubsection{Connection to HOPE: The Expectation Metric}

In the HOPE framework, we operate on functions $f: \calX \to \R^c$. We define our inner product relative to a surrogate distribution $P_{\calX}$:
\begin{equation}
    \label{eq:hope_inner}
    \inner{f_i}{f_j}_{\calH} \triangleq \E_{\xVect \sim P_{\calX}} \left[ f_i(\xVect)^T f_j(\xVect) \right]
\end{equation}

\begin{proposition}[Validity of the HOPE Metric]
The functional defined in (\ref{eq:hope_inner}) satisfies the inner product axioms.
\end{proposition}

\begin{proof}
Linearity and symmetry follow directly from the linearity of the expectation operator $\E$ and the symmetry of the Euclidean dot product. Positive definiteness is guaranteed because $\inner{f}{f} = \E[\norm{f(\xVect)}^2] \geq 0$. The definiteness $\inner{f}{f}=0 \implies f=0$ is satisfied in the $L_2$ sense (i.e., $f$ is zero "almost everywhere").
\end{proof}

\subsection{The Ambient Space: $L_2(\calX, P_{\calX})$}

A Hilbert space is more than just an inner product space; it must be \textbf{complete}. In finite dimensions (like $\R^n$), every inner product space is complete. In function spaces, this is not true.

\begin{definition}[Completeness]
A space is complete if every Cauchy sequence $\{f_n\}$ (a sequence where elements get arbitrarily close to each other) converges to an element $f$ that is also \textbf{inside} the space.
\end{definition}

If we worked only with continuous functions, the space would not be complete. For example, a sequence of continuous functions can converge to a step function (which is discontinuous). This would be a disaster for compression, as our "best parent" might not even exist in our space.

\subsubsection{The $L_2$ Embedding}
To avoid this, HOPE embeds neurons into $L_2(\calX, P_{\calX})$, the space of square-integrable functions. 
\begin{itemize}
    \item \textbf{Energy Bound:} Every function in $\calH$ has finite energy: $\E[\norm{f}^2] < \infty$.
    \item \textbf{Closure:} By definition, $L_2$ is complete. Every optimization we perform (minimizing the cost functional of compression operations) is guaranteed to have a valid result within the ambient space.
\end{itemize}

\subsection{The Tensor Product: Splitting Continuous and Discrete Spaces}

While we defined our ambient space as vector-valued functions $\calH = L_2(\calX, P_{\calX}; \R^c)$, computing the inner product directly in this monolithic space obscures the internal structure of a neural network. HOPE simplifies this by utilizing a \textbf{tensor product space}.

A neuron's operation naturally splits into two phases:
\begin{enumerate}
    \item \textbf{The Continuous Input Landscape:} The effective input weights (which absorb BN statistics) and the activation function create a continuous scalar landscape $g_i(\xVect) = \Psi((\wVect_{\text{in},i}^{\text{eff}})^T \xVect + b_i)$. We embed this function into a scalar Hilbert space $\calH_{\text{in}} \triangleq L_2(\calX, P_{\calX}; \R)$.
    \item \textbf{The Discrete Output:} This scalar activation is broadcast to the next layer along a finite-dimensional output weight vector $\wVect_{\text{out}, i}$. We define this output space as $\calH_{\text{out}} \triangleq \R^c$.
\end{enumerate}

By taking the tensor product of these two spaces, we construct the full ambient space mapping: $\calH \cong \calH_{\text{in}} \otimes \calH_{\text{out}}$. Under this formulation, each individual neuron is modeled as a \textbf{rank-1 Hilbert-Schmidt operator}, represented by the outer product of its input function and output vector:
\begin{equation}
    f_i \triangleq g_i \otimes \wVect_{\text{out}, i}
\end{equation}

\subsubsection{Factoring the Metric}
This tensor structure is what makes HOPE computationally tractable. The inner product of two rank-1 tensors elegantly factors into the product of their individual space inner products:
\begin{equation}
    \inner{f_i}{f_j}_{\calH} = \inner{g_i \otimes \wVect_{\text{out},i}}{g_j \otimes \wVect_{\text{out},j}}_{\calH} = \inner{g_i}{g_j}_{\calH_{\text{in}}} \cdot \inner{\wVect_{\text{out},i}}{\wVect_{\text{out},j}}_{\calH_{\text{out}}}
\end{equation}

Because $\inner{g_i}{g_j}_{\calH_{\text{in}}}$ is the expected alignment of their non-linear activations over the distribution $P_{\calX}$, we define this as the kernel $K(i,j)$. This allows us to separate the continuous-functional evaluation from the discrete parameters, reducing the full Hilbert space inner product to:
\begin{equation}
    \inner{f_i}{f_j}_{\calH} = K(i,j) \cdot \inner{\wVect_{\text{out},i}}{\wVect_{\text{out},j}}_{\R^c}
\end{equation}

\subsection{Uniqueness and the Total Set Property}

One might ask: "We defined the inner product for any functions $f, g$. But in HOPE, we only ever calculate it for single ReLU neurons. Is that enough to define the whole space?"

This is the most critical part of the theory.

\begin{definition}[Total Set]
A set $S \subset \calH$ is a \textbf{total set} if the set of all finite linear combinations of elements in $S$ is \textbf{dense} in $\calH$.
\end{definition}

If $S$ is total, then knowing the inner product for every pair in $S$ uniquely determines the inner product for the \textbf{entire} Hilbert space. 

\subsubsection{The Universal Approximation Theorem}

In the context of HOPE, our "dictionary" of functions is the set of single neurons $\calN$:
\begin{equation}
    \calN = \{ f(\xVect) = \wVect_{out} \Psi(\tilde{\wVect}_{\text{in}} \cdot \tilde{\xVect}) \}
\end{equation}

\begin{theorem}[HOPE Uniqueness]
The set $\calN$ is a total set in $L_2(\calX, P_{\calX})$.
\end{theorem}

\begin{proof}[Discussion]
By the Universal Approximation Theorem, linear combinations of ReLU neurons can approximate any square-integrable function to arbitrary precision. In the language of Hilbert spaces, this means $span(\calN)$ is dense in $\calH$. 

Because $\calN$ is total, the definition of the \textbf{function correlation kernel} $K(i, j) = \E[\Psi(y_i)\Psi(y_j)]$ is sufficient to uniquely characterize the metric of the entire ambient space. We do not need a "separate" definition for the inner product of sums of neurons; it is uniquely forced upon the space by the behavior of the single neurons.
\end{proof}

\subsection{Neuron Synthesis by Projection}

Finally, we see why this matters. The generation process relies on finding an \textbf{Optimal Subspace Projection}.
When we want to merge neurons $i$ and $j$, their joint contribution $[f_i, f_j]$ spans a 2-dimensional tensor subspace. Because a physical parent neuron must produce a single unified output, it is modeled as a constrained rank-1 approximation $[f_p, f_p]$. 

Rather than a simple orthogonal projection of a sum, HOPE finds the optimal parent $f_p^*$ by minimizing the expected Frobenius projection error under the Hilbert-Schmidt norm, scaled by the remaining capacity of the layer:
\begin{equation}
    f_p^* = \argmin_{f_p \in \calN} \frac{\sqrt{\norm{f_i - f_p}_{\calH}^2 + \norm{f_j - f_p}_{\calH}^2}}{E_a - \norm{f_i}_{\calH} - \norm{f_j}_{\calH} + \norm{f_p}_{\calH}}
\end{equation}
In a Hilbert space, the numerator translates to a rigorous geometric projection distance. Without the Hilbert space structure (inner products and completeness), the concepts of "closest operator" and "functional alignment" would have no meaning.

\subsection{Summary}
A summary of the key concepts and their utility in HOPE are provided in Table~\ref{tab:hilbert_hope}.

\begin{table}[h]
\centering
\begin{tabular}{ll}
\toprule
\textbf{Hilbert Concept} & \textbf{HOPE Implementation} \\
\midrule
Inner Product $\inner{f}{g}$ & Expected dot product over $P_{\calX}$ \\
Norm $\norm{f}$ & Functional Capacity (Square root of Signal Energy) \\
Ambient Space $\calH$ & Tensor product space $\calH_{\text{in}} \otimes \calH_{\text{out}}$ \\
Operators & Neurons modeled as rank-1 Hilbert-Schmidt operators \\
Basis / Total Set $\calN$ & The manifold of realizable single neurons \\
Completeness & Guarantees that the "best parent" is well-defined \\
\bottomrule
\end{tabular}
\caption{Mapping of Abstract Hilbert Concepts to the HOPE Framework.}
\label{tab:hilbert_hope}
\end{table}

\section{Implementation Notes}

\subsection{Adaptation for Convolutional Layers}
\label{sec:conv_nets}
Our formalism extends to convolutional networks by defining a ``neuron'' as a filter in layer $A$ producing feature map $i$. The joint parameter vector $\wVect_i = [\wVect_n^\top, \wVect_c^\top]^\top$ is constructed by vectorizing the filter's respective input and output kernels. The input space vector $\wVect_n$ encapsulates the filter's local receptive field. Assuming layer $A$ possesses a weight tensor $\Kmat_A \in \R^{h_A \times w_A \times C_{\text{in}} \times C_{\text{out}}}$, the vector $\wVect_n$ corresponding to filter $i$ is the flattened spatial slice $\Kmat_A[:, :, :, i] \in \R^n$, where $n = h_A \times w_A \times C_{\text{in}}$. Conversely, the output space vector $\wVect_c$ captures the filter's downstream influence on the subsequent layer $B$. Because the activation map of filter $i$ acts as the $i$-th input channel to layer $B$, $\wVect_c$ is formed by extracting and flattening the corresponding slice of the downstream tensor $\Kmat_B[:, :, i, :] \in \R^c$, where $c = h_B \times w_B \times C_{\text{out}, B}$.

To allow the surrogate distribution to serve as a location-invariant prior without requiring intractable coordinate-specific covariance modeling, we assume spatial stationarity (ergodicity) across the feature map and construct it using the globally averaged BN variance $\sigma_i^2$. While boundary zero-padding breaks local stationarity, the high-dimensional spatial aggregation of modern BN buffers absorbs these edge-effects into a into a single global average.

\subsection{Deriving the Parameter Footprint $\Delta P$}
\label{sec:delta_P}

As established in Section~\ref{sec:rate_distortion}, the parameter footprint $\Delta P$ quantifies the number of physical parameters removed by a compression action. To preserve Dantzig's Axiom of Item Independence, this criterion uses a static surrogate, $\Delta P^{\text{init}}$, evaluated on the initial network state and decoupled from the dynamically changing network. For any operation, the footprint tracks the total removed parameters:
\begin{equation}
    \Delta P^{\text{init}} = \|W_{\text{in}}\|_0 + \|W_{\text{out}}\|_0 + \|\theta_{\text{aux}}\|_0
\end{equation}
where $W_{\text{in}}$, $W_{\text{out}}$, and $\thetaVect_{\text{aux}}$ denote the input weights, output weights, and auxiliary parameters (e.g., BN parameters)  respectively, and $\|\cdot\|_0$ counts the number of non-zero elements.

\subsubsection{Granular Operations}
For pruning or merging, $\Delta P^{\text{init}}$ comprises the incoming weights, outgoing weights, and BN parameters of a single target neuron or filter. For example:
\begin{itemize}
    \item \textbf{Transformer MLP:} Removing a neuron in an MLP with a $d_{\text{model}} \rightarrow 4d_{\text{model}}$ expansion yields $\Delta P^{\text{init}} = 2d_{\text{model}} + 1$.
    \item \textbf{Convolutional Networks:} For a filter with spatial dimensions $H \times W$, the footprint scales with the receptive field and projective kernel: $\Delta P^{\text{init}} = H \cdot W \cdot C_{\text{in}} + C_{\text{out}} + 4$, where $4$ accounts for the BN parameters $\gamma, \beta, \mu, \sigma^2$.
\end{itemize}

\begin{insight}{Example: Architectural Symmetry in ResNet-50}
Evaluating $\Delta P^{\text{init}}$ locally reveals symmetries across different layer types. Consider a ResNet-50 bottleneck block with a base channel width $N$ and a $4\times$ expansion ratio:
\begin{itemize}
    \item \textbf{1$\times$1 Squeeze Layer:} Removing one filter deletes $4N$ input connections and $9N$ output connections to the subsequent $3\times3$ layer. Yield: $\Delta P^{\text{init}} = 13N$.
    \item \textbf{3$\times$3 Spatial Layer:} Removing one filter deletes $9N$ input connections and $4N$ output connections to the subsequent $1\times1$ expansion layer. Yield: $\Delta P^{\text{init}} = 13N$.
\end{itemize}
Despite differing spatial tensor shapes, the parameter yield per filter in both positions evaluates to $13N$. This symmetry allows the cost $\calJ$ to arbitrate compression across heterogeneous layers without being biased by raw parameter array shapes.
\end{insight}

\subsubsection{Macro Operations: Block Eviction}
Unlike granular operations that yield incremental savings, block eviction removes entire layers simultaneously. For a residual block with internal convolutional layers $W_1$ and $W_2$, and a terminal expansion layer $W_3$, the static parameter yield expands to:
\begin{equation}
    \Delta P_{\text{evict}}^{\text{init}} = \|W_1\|_0 + \|W_2\|_0 + \|\theta_{\text{aux}}\|_0
\end{equation}
where $\|W_1\|_0$ and $\|W_2\|_0$ denote the parameters of the internal layers, and $\|\theta_{\text{aux}}\|_0$ counts the BN parameters $\mu, \sigma^2, \gamma, \beta$ across the entire block, including those of $W_3$. Note that the weight matrix $W_3 \in \R^{d_{\text{amb}} \times d_{\text{bottleneck}}}$ is \textit{not} explicitly included in this sum. Because removing a filter in $W_2$ removes its corresponding outgoing connections in $W_3$, the memory footprint of $W_3$ is naturally accounted for when evaluating $W_2$. Explicitly adding the ambient dimensions $\|W_3\|_0 = d_{\text{amb}} \times d_{\text{bottleneck}}$ would count the same parameters twice \textit{within a single macro action}. This \textit{intra-action double-counting} would inflate the parameter footprint of block eviction and artificially lower its Distortion Rate (DR), giving it an unfair advantage over granular operations.

\subsection{Cross-Action Overlap and Uniform Scaling}
\label{sec:appendix_double_counting}

While we avoid counting parameters twice within a \textit{single} action (as seen with $W_3$), evaluating the entire decision space using a static footprint $\Delta P^{\text{init}}$ introduces an overlap \textit{between different competing actions}. Let $i$ and $j$ be targeted neurons in adjacent layers $l$ and $l+1$. To preserve Dantzig's Axiom of Item Independence, an action targeting layer $l$ must not alter the state variables used to evaluate layer $l+1$. Consequently, their shared weight $W^{(l)}_{j,i}$ is counted independently in both evaluations: $W^{(l)}_{j,i} \in \Delta P_i^{\text{init}}$ and $W^{(l)}_{j,i} \in \Delta P_j^{\text{init}}$.

For any sequence of actions $\calS$, this \textit{cross-action overlap} overestimates the true number of parameters recovered $\sum_{k \in \calS} \Delta P_k^{\text{init}} > \Delta P_{\calS}^{\text{live}}$. This overestimation artificially lowers the computed DR compared to the live network state:
\begin{equation}
\text{DR}_k = \frac{\calJ_k}{\Delta P_k^{\text{init}}} < \frac{\calJ_k}{\Delta P_k^{\text{live}}}
\end{equation}

To correct this approximation without violating item independence, HOPE relies on uniform scaling. Because this cross-action overlap applies systematically across the entire action space $\calA$ (all granular and macro candidates interact with their neighbors), it acts as a uniform scaling factor $\alpha \ge 1$ such that $\Delta P_k^{\text{init}} \approx \alpha \Delta P_k^{\text{live}}$.

Because the distortion cost $\calJ$ is evaluated independently of the parameter counts, and $\alpha$ applies uniformly, the relative ordering of the distortion rates is preserved:
\begin{equation}
\frac{\calJ_a}{\Delta P_a^{\text{init}}} < \frac{\calJ_b}{\Delta P_b^{\text{init}}} \iff \frac{\calJ_a}{\Delta P_a^{\text{live}}} < \frac{\calJ_b}{\Delta P_b^{\text{live}}}
\end{equation}
Since the greedy continuous knapsack solver selects $\arg\min_{k \in \calA} \text{DR}_k$, this uniform scaling ensures that the theoretical fairness of the optimal action selection remains intact.

\subsection{Computational Complexity and the Decoupled Cache}
\label{sec:decoupled_cache}

Evaluating the transition cost $\calJ$ efficiently poses a computational challenge because it is inversely coupled to the monotonically decreasing layer capacity $E_{\text{rem}}$. In a layer with $N$ neurons, recomputing the non-linear weight-space geometry (e.g., Rank-2 Singular Value Decompositions) for all $\calO(N^2)$ candidate pairs every time $E_{\text{rem}}$ decreases requires $\calO(N^3)$ execution time. 

Conversely, caching $\calJ$ values in a standard priority queue via submodular approximations (e.g., Minoux's Lazy Update) leads to stale estimations. Because the transition cost goes to $\infty$ as capacity approaches $0$, small capacity reductions cause the true costs to spike. A delayed queue would underestimate these costs and cause the optimizer to select sub-optimal actions and potentially collapse the layer. We resolve this bottleneck using an $\calO(1)$ \textit{Decoupled Cache}.

\paragraph{Decoupling.}
The computationally expensive optimal projections $\uVect^*, \vVect^*$ and the cost components ($a = \|f_i\|_{\calH}^2 + \|f_j\|_{\calH}^2$ and $b = \inner{\psi^*}{f_i+f_j}_{\calH}$, derived in Section~\ref{sec:parent_params}) depend only on weights and static BN parameters associated with that layer. Since our method does not rely on any cross-layer criterion (such as Fisher Information), we guarantee that $\uVect^*, \vVect^*, a, b$ remain independent of both the downstream architecture and the dynamic capacity $E_{\text{rem}}$. 

The only variable dependent on $E_{\text{rem}}$ is the optimal scalar magnitude $s^*$. At initialization, the framework computes $\uVect^*$ and $\vVect^*$ to evaluate and cache only the scalar constants $a$ and $b$. To prevent $\calO(N^2)$ memory exhaustion, the high-dimensional vectors $\uVect^*,\vVect^*$ are then discarded. During the greedy search, evaluating the cost of any action requires querying the cached constants and the live remaining capacity $E_{\text{rem}} = \max(E_a - \|f_i\|_{\calH} - \|f_j\|_{\calH}, \epsilon)$ to compute $s^*$ and $\calJ$ analytically in $\calO(1)$ time:
\begin{equation}
    s^* = \frac{a + b E_{\text{rem}}}{2 E_{\text{rem}} + b}
\end{equation}

\paragraph{Index Determinism and JIT Generation.}
We sort the active neurons to evaluate undirected pairs only where $i < j$. By enforcing this index determinism, the algorithm guarantees a 100\% cache hit rate and makes it viable to abandon priority queues entirely. At every step, the algorithm executes an $\calO(1)$ global scan over all remaining pairs using the live layer capacity, maintaining mathematically perfect freshness. Once the global minimum is selected, the framework executes a Just-In-Time (JIT) generation, re-evaluating the fast Rank-2 SVD only for the single winning pair ($\sim 1$ ms) to retrieve its $\uVect^*,\vVect^*$ for physical deployment.

\paragraph{Computing the BN Variance.}
Deploying the JIT-generated parent requires setting its BN variance $\gamma_p^2$. Under the surrogate distribution $P_{\calX}$, $\gamma_p^2$ depends only on the warped correlation $\hat{\rho}_{ij}$. A naive approach might substitute the simple correlation of the raw weights $\rho_{\text{raw}}$ into the variance equation. However, this incorrectly mixes parameter-space metrics with signal-space statistics. Doing so produces incorrect BN moving averages, which miscalibrates the network during the forward pass. The framework avoids this by computing the physical variance using $\hat{\rho}_{ij}$ directly retrieved from the cache. This ensures the deployed parameters maintain the correct statistical behavior without requiring empirical forward passes.

\subsection{Numerical Stability and BN Parameters}
\label{sec:numerical_stability_BN}

Section~\ref{sec:bn_invert} derives the mapping from the parent neuron's effective parameters ($\wVect_{\text{in},p}^{\text{eff}}$ and $b_p$) to its physical network variables: the raw input weights $\wVect_{\text{in},p}^{\text{raw}}$ and the BN parameters $\gamma_p, \beta_p, \mu_p, \sigma_p^2$. A naive assignment of these physical variables alters the pre-activation signal and breaks the critical equivalence $y_p = (\wVect_{\text{in},p}^{\text{eff}})^T \xVect + b_p$. This mismatch miscalibrates the network and immediately degrades accuracy on the first forward pass.

To prevent this, HOPE enforces the exact analytical mapping. However, deploying these formulas in practice requires safely handling numerical bounds, such as preventing negative variance $\sigma_p^2 < 0$ and avoiding \texttt{NaN} in the BN denominator $\sqrt{\sigma_p^2 + \epsilon}$. This section details how to resolve these numerical edge cases and ensure the deployed model remains mathematically well-defined and numerically stable.

\subsubsection{The $\epsilon$ Boundary Regime and Variance Clamping}
To resolve the under-constrained BN system, the framework fixes the physical variance as $\sigma_p^2 = \max(0, \gamma_p^2 - \epsilon)$, where $\epsilon$ is a small stability constant (e.g., $10^{-5}$). For active features $\gamma_p^2 \ge \epsilon$, this evaluates smoothly. This allows the denominator of the BN transformation to simplify to $|\gamma_p|$.

However, as the progressive encoder shrinks the network, the derived parent variance $\gamma_p^2$ may occasionally fall below the numerical floor $\epsilon$. Without the $\max(0, \cdot)$ operator, enforcing $\sigma_p^2 = \gamma_p^2 - \epsilon$ would result in a negative variance, which causes \texttt{NaN} during inference. The clamping operator safely bounds the physical variance at $\sigma_p^2 = 0$. 

When $\sigma_p^2$ is clamped to $0$, the BN denominator $\sqrt{\sigma_p^2 + \epsilon}$ evaluates to $\sqrt{\epsilon}$. Consequently, the scale factor becomes $\gamma_p / \sqrt{\epsilon}$. Substituting this into the effective bias equation yields:
\begin{equation}
    b_p = \beta_p - \frac{\gamma_p}{\sqrt{\epsilon}} \mu_p \quad \implies \quad \mu_p = \frac{\sqrt{\epsilon}}{\gamma_p} (\beta_p - b_p)
\end{equation}
As the parent neuron's scale $\gamma_p$ approaches $0$, dividing by $\gamma_p$ causes $\mu_p \rightarrow \infty$. To prevent this, the framework bypasses this calculation for inactive features and sets $\wVect_{\text{in},p}^{\text{raw}} = \zeroVect$.

\subsubsection{Running Variance Offset During Fine-Tuning}
Setting the initial BN running variance to $\sigma_p^2 = \gamma_p^2 - \epsilon$ ensures the network's output is perfectly preserved during inference. However, when the model resumes training for fine-tuning, the actual batch variance computed during the forward pass evaluates to $\sigma_{\text{batch}}^2 = \gamma_p^2$. This introduces a minor mismatch between the empirical batch variance and the stored running variance:
\begin{equation}
    \sigma_{\text{batch}}^2 - \sigma_p^2 = \gamma_p^2 - (\gamma_p^2 - \epsilon) = \epsilon
\end{equation}
Because this discrepancy evaluates to the numerical constant $\epsilon$ (typically $10^{-5}$), its impact is negligible. Standard optimizers (e.g., Adam, AdamW) seamlessly absorb this $\calO(\epsilon)$ offset during the initial training steps without destabilizing the network or degrading performance.

\section{Main Paper Proofs}
\subsection{Layer Transition Costs}

For convenience, we first recall the axioms of the cost $\calJ$ from the main paper.

\begin{insight}{Axioms of the Cost Functional $\calJ$}
To ensure a well-posed definition of the cost $\calJ$, the framework introduces the following three axioms:
\begin{itemize}
    \item \textbf{Magnitude Neutrality:} $\calJ$ must be scale invariant: $\forall k > 0 \,;\, \calJ(k\, \Phi_a , k\, \Phi_b) = \calJ(\Phi_a , \Phi_b)$.
    \item \textbf{Connectivity Preservation:} $\calJ$ must establish an asymptotic barrier preventing layer extinction: $\lim_{E(\Phi_b) \rightarrow 0^+} \calJ = \infty$.
    \item \textbf{Infinitesimal Capacity Dependence:} $\calJ$ must be additive along continuous paths and driven by the reduction in layer capacity: $\calJ(\Phi_a, \Phi_b) = \int_{0}^{1} -\xi(\Phi(t)) \dot{E}(t) dt$, where $\dot{E}(t) \triangleq dE(\Phi(t))/dt$ and $\xi(\Phi(t)) > 0$ is a state-dependent density function.
\end{itemize}
\end{insight}

\begin{lemma}[Uniqueness of the $L_1$ Capacity]
\label{prop:unique_l1}
Let $\Phi = (f_1, \dots, f_N) \in \calH^N$ denote a layer state. Assume the capacity functional $E: \bigcup_{N=1}^\infty \calH^N \to \R_{\ge 0}$ satisfies:
\begin{enumerate}
    \item \textbf{Identity:} $\forall f \in \calH, \; E((f)) = \norm{f}_{\calH}$.
    \item \textbf{Symmetry \& Separability:} $\exists \, g: \R_{\ge 0} \to \R$ (continuous and strictly monotonic) and a function $h$ such that $\forall \Phi \in \calH^N, \; E(\Phi) = h\left( \sum_{k=1}^N g(\norm{f_k}_{\calH}) \right)$.
    \item \textbf{Partition Invariance:} $\forall f \in \calH, \, \forall N \in \mathbb{Z}_{\ge 1}, \; E((f)) = E\big((\underbrace{f/N, \dots, f/N}_{N \text{ times}})\big)$.
\end{enumerate}
Then $E(\Phi) = \sum_{k=1}^N \norm{f_k}_{\calH}$.
\end{lemma}

\begin{proof}
For a single-neuron state $\Phi = (f)$, Conditions 1 and 2 imply:
\begin{equation*}
    E((f)) = h\big(g(\norm{f}_{\calH})\big) = \norm{f}_{\calH} \implies h \equiv g^{-1} \text{ on } \text{Im}(g)
\end{equation*}
Thus, the functional simplifies to $E(\Phi) = g^{-1}\left( \sum_{k=1}^N g(\norm{f_k}_{\calH}) \right)$.

By Condition 3 and the positive homogeneity of the norm ($\norm{f/N}_{\calH} = \norm{f}_{\calH}/N$ for $N \ge 1$):
\begin{equation*}
    \norm{f}_{\calH} = E\big((\underbrace{f/N, \dots, f/N}_{N \text{ times}})\big) = g^{-1}\left( \sum_{k=1}^N g\left(\frac{\norm{f}_{\calH}}{N}\right) \right) = g^{-1}\left( N g\left(\frac{\norm{f}_{\calH}}{N}\right) \right)
\end{equation*}

Applying $g$ to both sides and substituting $x \triangleq \norm{f}_{\calH} \ge 0$ yields:
\begin{equation*}
    g(x) = N g\left(\frac{x}{N}\right) \qquad \forall x \ge 0, \; \forall N \in \mathbb{Z}_{\ge 1}
\end{equation*}

For $x=0$, $g(0) = N g(0) \implies g(0) = 0$. For any rational $q = M/N > 0$, substituting $y = x/N$ yields $g(My) = M g(y) = q N g(y) = q g(Ny)$, meaning $g(qx) = qg(x)$. Since $g$ is continuous, this linearity extends to all $x \in \R_{\ge 0}$, yielding $g(x) = cx$ for some constant $c$. Since $g$ is strictly monotonic, $c \neq 0$.

Substituting $g(x) = cx$ and $g^{-1}(y) = y/c$ into the expression for $E(\Phi)$ gives:
\begin{equation*}
    E(\Phi) = \frac{1}{c} \sum_{k=1}^N c \norm{f_k}_{\calH} = \sum_{k=1}^N \norm{f_k}_{\calH}
\end{equation*}
\end{proof}

\begin{theorem}[Integral Formulation of Scale-Invariant Cost]
\label{theorem:path_log_form}
Under the Axioms of Magnitude Neutrality, Connectivity Preservation, and Infinitesimal Capacity Dependence, the transition cost along a continuous deformation path $\Phi: [0,1] \rightarrow \calH^N $ with boundary conditions $\Phi(0) = \Phi_a$ and $\Phi(1) = \Phi_b$, is determined as the integral:
\begin{equation}
    \calJ(\Phi_a, \Phi_b) = \int_{0}^{1} -c(\Phi(t)) \frac{\dot{E}(t)}{E(\Phi(t))} dt\footnote{Since the Hilbert norm $\norm{f}_{\calH}$ is non-differentiable at $f = \zeroVect$, this integral is defined over paths where the active capacity remains positive $E(\Phi(t)) > 0$.}
\end{equation}
where $E(\Phi(t))$ is the instantaneous capacity, $\dot{E}(t) < 0$ is the rate of capacity reduction, and $c(\Phi(t)) > 0$ is a scale-invariant factor.
\end{theorem}

\begin{proof}
\noindent\textbf{1. Differential Form:}
By the Infinitesimal Capacity Dependence axiom, the differential cost along a path is driven by capacity reduction:
\begin{equation*}
    \dot{\calJ}(t) = -\xi(\Phi(t)) \dot{E}(t)
\end{equation*}
Since active capacity decreases during compression $\dot{E}(t) < 0$ and the cost rate must be positive $\dot{\calJ}(t) > 0$, we require $\xi(\Phi(t)) > 0$.

\vspace{0.5em}
\noindent\textbf{2. Magnitude Neutrality:}
For any scalar $k > 0$, Magnitude Neutrality requires $\calJ(k\Phi_a, k\Phi_b) = \calJ(\Phi_a, \Phi_b)$. By Lemma~\ref{prop:unique_l1}, capacity scales linearly $E(k\Phi) = kE(\Phi)$, so $\dot{E}(k\Phi(t)) = k\dot{E}(t)$. Integrating over the scaled path yields:
\begin{equation*}
    \int_{0}^{1} -\xi(k\Phi(t)) \, k \, \dot{E}(t) dt = \int_{0}^{1} -\xi(\Phi(t)) \dot{E}(t) dt
\end{equation*}
Assuming continuous integrands, since this equality holds for any valid continuous path, the integrands must be identical pointwise:
\begin{equation*}
    k \, \xi(k\Phi) \dot{E}(t) = \xi(\Phi) \dot{E}(t) \implies \xi(k\Phi) = k^{-1} \xi(\Phi)
\end{equation*}
Thus, $\xi$ is a homogeneous function of degree $-1$.

\vspace{0.5em}
\noindent\textbf{3. Scale-Invariant Factor:}
Define $c(\Phi) \triangleq \xi(\Phi) E(\Phi)$. Scaling the state by $k$ yields:
\begin{equation*}
    c(k\Phi) = \xi(k\Phi) E(k\Phi) = (k^{-1} \xi(\Phi)) (k E(\Phi)) = c(\Phi)
\end{equation*}
This shows $c(\Phi)$ is scale-invariant. Substituting $\xi(\Phi) = c(\Phi) / E(\Phi)$ back into the differential form gives the integral:
\begin{equation*}
    \calJ(\Phi_a, \Phi_b) = \int_{0}^{1} -c(\Phi(t)) \frac{\dot{E}(t)}{E(\Phi(t))} dt
\end{equation*}

\vspace{0.5em}
\noindent\textbf{4. Connectivity Preservation:}
This axiom mandates an infinite cost barrier against layer extinction: $\lim_{E_b \to 0^+} \calJ(\Phi_a, \Phi_b) = \infty$, where $E_b = E(\Phi_b)$ and $E_a = E(\Phi_a)$. Applying the change of variables $dE = \dot{E}(t) dt$ and reversing the limits (which absorbs the negative sign since $E_b < E_a$ due to $\dot{E}(t)<0$) gives:
\begin{equation*}
    \calJ(\Phi_a, \Phi_b) = \int_{E_a}^{E_b} -c(\Phi) \frac{dE}{E} = \int_{E_b}^{E_a} c(\Phi) \frac{dE}{E}
\end{equation*}
If $c(\Phi)$ is bounded below by a constant $c_{\min} > 0$ along the path to extinction:
\begin{equation*}
    \calJ(\Phi_a, \Phi_b) \ge \int_{E_b}^{E_a} c_{\min} \frac{dE}{E} = c_{\min} \big( \ln(E_a) - \ln(E_b) \big)
\end{equation*}
Taking the limit as $E_b \to 0^+$ yields $\infty$, satisfying the axiom.
\end{proof}

\begin{insight}{Relative Differential Cost}
For any differentiable deformation path $\Phi: [0,1] \rightarrow \calH^N$, the rate of geometric projection cost accumulation $\dot{\calJ}_{\text{proj}}(t)$ is defined as:
\begin{equation}
\label{def:inf_weber}
    \dot{\calJ}_{\text{proj}}(t) \triangleq c(\Phi(t)) \frac{\dot{s}(t)}{E(\Phi(t))}
\end{equation}
where $\dot{s}(t) = \norm{\dot{\Phi}(t)}_{\calH^N} \ge 0$ is the instantaneous geometric speed of the state vector.
\end{insight}

\begin{proof}[\textbf{Derivation}]
The scalar capacity rate $\dot{E}(t)$ from Theorem~\ref{theorem:path_log_form} assigns zero penalty to geometric deformations (e.g., orthogonal rotations) where scalar capacity is conserved $\dot{E}(t) = 0$. To capture structural distortion, we calibrate the criterion along a localized orthogonal path.

Consider a compressive path where a single active neuron $f_k$ is scaled toward $\zeroVect$ by a decreasing scalar $\alpha(t) \in [0,1]$ $\dot{\alpha}(t) < 0$, while all other neurons remain static. The capacity is $E(\Phi(t)) = \alpha(t) \norm{f_k}_{\calH} + \sum_{i \neq k} \norm{f_i}_{\calH}$, giving $\dot{E}(t) = \dot{\alpha}(t) \norm{f_k}_{\calH} < 0$.

The geometric speed $\dot{s}(t)$ is the norm of the state derivative vector. Since only the $k$-th coordinate changes, this vector is 1-sparse, ensuring the norms coincide:
\begin{equation*}
    \dot{s}(t) = \norm{\dot{\Phi}(t)}_{\calH^N} = \abs{\dot{\alpha}(t)} \norm{f_k}_{\calH} = -\dot{\alpha}(t) \norm{f_k}_{\calH} = -\dot{E}(t)
\end{equation*}

Substituting $\dot{s}(t) = -\dot{E}(t)$ into the cost baseline $\dot{\calJ}_{\text{capacity}}(t) = -c(\Phi) \frac{\dot{E}}{E}$ yields $\dot{\calJ}(t) = c(\Phi) \frac{\dot{s}}{E}$ for this orthogonal axis. Because the ambient space $\calH^N$ is isotropic, we define the geometric projection cost $\dot{\calJ}_{\text{proj}}(t)$ as this ratio to generalize to arbitrary trajectories.
\end{proof}

\begin{lemma}[Capacity Bound for Correlated Projections]
\label{lemma:discontinuity_capacity_bound}
Let $f_i, f_j \in \calH$ be active candidate neurons $\norm{f_i}_\calH > 0, \norm{f_j}_\calH > 0$, and let $f_p \in \calH$ be their optimal parent neuron. Let $E_{\text{rem}} \triangleq E_a - \|f_i\|_\calH - \|f_j\|_\calH \ge 0$. Define the functional correlation as $\rho_{ij} \triangleq \frac{\inner{f_i}{f_j}_\calH}{\norm{f_i}_\calH \norm{f_j}_\calH}$. 

For the continuous straight-line deformation path $\Phi(t)$ in $\calH^N$ connecting the initial state $\Phi_a$ to the pre-deletion target state $\tilde{\Phi}_b$, there exists a correlation threshold $\rho^*(f_i, f_j) \in (0, 1)$ such that if $\rho_{ij} \ge \rho^*$, then:
\begin{equation}
    E(\Phi(t)) \ge E(\Phi_b) \quad \forall \, t \in [0, 1]
\end{equation}
where $E(\Phi_b) = E_{\text{rem}} + \norm{f_p}_\calH$ is the capacity of the post-deletion terminal state $\Phi_b \in \calH^{N-1}$.
\end{lemma}

\begin{proof}
Parameterizing the path as $\Phi(t) = (1-t)\Phi_a + t\tilde{\Phi}_b$ for $t \in [0, 1]$, the targeted neurons transition via $f_i(t) = (1-t)f_i + t f_p$ and $f_j(t) = (1-t)f_j + t f_p$. By Lemma~\ref{prop:unique_l1} and the triangle inequality $\norm{u}_\calH + \norm{v}_\calH \ge \norm{u+v}_\calH$:
\begin{equation*}
    E(\Phi(t)) \ge E_{\text{rem}} + \norm{(1-t)(f_i + f_j) + 2tf_p}_\calH
\end{equation*}
To prove $E(\Phi(t)) \ge E_{\text{rem}} + \norm{f_p}_\calH$, we require:
\begin{equation} \label{eq:buffer_condition}
    \norm{(1-t)(f_i + f_j) + 2tf_p}_\calH > \norm{f_p}_\calH \quad \forall \, t \in [0, 1]
\end{equation}

We evaluate this in the collinear limit $\rho_{ij} \to 1$. Let $f_i = x \hat{u}$, $f_j = y \hat{u}$, and $f_p = z \hat{u}$ for a shared unit vector $\hat{u}$ and scalars $x, y, z > 0$. The condition simplifies to:
\begin{equation*}
    (1-t)(x+y) + 2tz > z
\end{equation*}
Since this expression is linear in $t$, its minimum occurs at the boundaries:
\begin{itemize}
    \item \textbf{At $t=1$:} $2z > z$, which inherently holds since $z > 0$.
    \item \textbf{At $t=0$:} $x+y > z$. From the optimal scale derivation (Section~\ref{sec:parent_params}), $z = \frac{a + b E_{\text{rem}}}{2E_{\text{rem}} + b}$. In the collinear limit, $a = x^2 + y^2$ and $b = x + y$. Thus:
    \begin{align*}
        x + y &> \frac{x^2 + y^2 + (x+y)E_{\text{rem}}}{2E_{\text{rem}} + x + y} \\
        (x+y)(2E_{\text{rem}} + x + y) &> x^2 + y^2 + (x+y)E_{\text{rem}} \\
        2E_{\text{rem}}(x+y) + x^2 + 2xy + y^2 &> x^2 + y^2 + E_{\text{rem}}(x+y) \\
        E_{\text{rem}}(x+y) + 2xy &> 0
    \end{align*}
    Since $x > 0$, $y > 0$, and $E_{\text{rem}} \ge 0$, this strict inequality unconditionally holds.
\end{itemize}

Since both endpoints satisfy the strict inequality, it holds for all $t \in [0, 1]$ in the collinear limit. Because the Hilbert norm and $f_p$ are continuous with respect to $\rho_{ij}$, this inequality is preserved in a neighborhood around $\rho_{ij} = 1$. Thus, there exists a threshold $\rho^*(f_i, f_j) \in (0, 1)$ ensuring the condition for $\rho_{ij} \ge \rho^*$.
\end{proof}

\begin{theorem}[Discrete Transition Cost Bound]
\label{thm:discrete_subspace}
For any structural reduction from an initial state $\Phi_a \in \calH^N$ to a terminal state $\Phi_b \in \calH^{N-1}$ (specifically, pruning a neuron, or merging a correlated pair with $\rho_{ij} \ge \rho^*$), the continuous projection cost $\calJ_{\text{proj}}$ evaluated along the straight-line path to the pre-deletion target $\tilde{\Phi}_b \in \calH^N$ is upper-bounded by the discrete proxy $\calJ_{\text{bound}}$:
\begin{equation}
    \calJ_{\text{proj}}(\Phi_a , \Phi_b) \le c(\Phi_a) \frac{D(\Phi_a, \tilde{\Phi}_b)}{E(\Phi_b)} \equiv \calJ_{\text{bound}}(\Phi_a , \Phi_b)
\end{equation}
where $E(\Phi_b)$ is the post-deletion capacity, and $D(\Phi_a, \tilde{\Phi}_b) = \norm{\Phi_a - \tilde{\Phi}_b}_{\calH^N}$ is the Euclidean distance in the configuration space.
\end{theorem}

\begin{proof}
By Definition~\ref{def:inf_weber}, the cost along $\Phi(t) = (1-t)\Phi_a + t \tilde{\Phi}_b$ is:
\begin{equation*}
    \calJ_{\text{proj}}(\Phi_a, \Phi_b) = \int_0^1 c(\Phi(t)) \frac{\dot{s}(t)}{E(\Phi(t))} dt
\end{equation*}

Since the active neuron count remains invariant prior to $t=1$, $c(\Phi(t)) = c(\Phi_a)$ almost everywhere on $[0,1]$ and can be factored out. 

Next, we bound the dynamic capacity $E(\Phi(t))$:
\begin{itemize}
    \item \textbf{Pruning:} $E(\Phi(t)) = E_a - t\norm{f_i}_\calH$. Its minimum is $E(1) = E_a - \norm{f_i}_\calH = E(\Phi_b)$.
    \item \textbf{Merging:} Lemma~\ref{lemma:discontinuity_capacity_bound} guarantees $E(\Phi(t)) \ge E(\Phi_b)$ for all $t \in [0,1]$.
\end{itemize}
In both cases, $E(\Phi(t)) \ge E(\Phi_b) > 0$. Substituting $1/E(\Phi(t)) \le 1/E(\Phi_b)$ into the integral establishes an upper bound:
\begin{equation} \label{eq:cost_inequality}
    \calJ_{\text{proj}}(\Phi_a, \Phi_b) \le \frac{c(\Phi_a)}{E(\Phi_b)} \int_0^1 \dot{s}(t) dt
\end{equation}

The geometric speed $\dot{s}(t) = \norm{\dot{\Phi}(t)}_{\calH^N} = \norm{\tilde{\Phi}_b - \Phi_a}_{\calH^N} \equiv D(\Phi_a, \tilde{\Phi}_b)$ is constant. Evaluating the integral yields $D(\Phi_a, \tilde{\Phi}_b)$, leading to:
\begin{equation*}
    \calJ_{\text{proj}}(\Phi_a, \Phi_b) \le c(\Phi_a) \frac{D(\Phi_a, \tilde{\Phi}_b)}{E(\Phi_b)} \equiv \calJ_{\text{bound}}(\Phi_a , \Phi_b)
\end{equation*}
\end{proof}

\begin{proposition}[Axiomatic Consistency of the Bounded Proxy]
\label{prop:bound_axioms}
The bounded projection cost $\calJ_{\text{bound}}(\Phi_a, \Phi_b) = c(\Phi_a) \frac{D(\Phi_a, \tilde{\Phi}_b)}{E(\Phi_b)}$ satisfies Axiom 1 (Magnitude Neutrality) and Axiom 2 (Connectivity Preservation).
\end{proposition}

\begin{proof}
\noindent\textbf{Axiom 1 (Magnitude Neutrality):} For any $k > 0$, we have $c(k\Phi_a) = c(\Phi_a)$ (scale-invariant by Theorem~\ref{theorem:path_log_form}), $D(k\Phi_a, k\tilde{\Phi}_b) = k D(\Phi_a, \tilde{\Phi}_b)$, and $E(k\Phi_b) = k E(\Phi_b)$ (linear scaling by Lemma~\ref{prop:unique_l1}). Thus, $k$ cancels out:
\begin{equation*}
    \calJ_{\text{bound}}(k\Phi_a, k\Phi_b) = c(\Phi_a) \frac{k D(\Phi_a, \tilde{\Phi}_b)}{k E(\Phi_b)} = \calJ_{\text{bound}}(\Phi_a, \Phi_b)
\end{equation*}

\noindent\textbf{Axiom 2 (Connectivity Preservation):} For an initial state with $E(\Phi_a) > 0$, we evaluate $\lim_{E(\Phi_b) \rightarrow 0^+} \calJ_{\text{bound}}(\Phi_a, \Phi_b)$. 

By the definitions of pruning and merging, $E(\tilde{\Phi}_b) \le 2E(\Phi_b)$.\footnote{For pruning, $E(\tilde{\Phi}_b) = E(\Phi_b)$. For merging, $E(\tilde{\Phi}_b) = E_{\text{rem}} + 2\norm{f_p}_\calH \le 2(E_{\text{rem}} + \norm{f_p}_\calH) = 2E(\Phi_b)$.} Thus, $E(\Phi_b) \rightarrow 0^+ \implies E(\tilde{\Phi}_b) \rightarrow 0$. Because norms are equivalent on a finite-dimensional space, $\norm{\Phi}_{\calH^N} \le E(\Phi) \le \sqrt{N} \norm{\Phi}_{\calH^N}$, which implies $\norm{\tilde{\Phi}_b}_{\calH^N} \rightarrow 0$.

Applying the reverse triangle inequality gives:
\begin{equation*}
    D(\Phi_a, \tilde{\Phi}_b) = \norm{\Phi_a - \tilde{\Phi}_b}_{\calH^N} \ge \abs{\norm{\Phi_a}_{\calH^N} - \norm{\tilde{\Phi}_b}_{\calH^N}}
\end{equation*}
Taking the limit as $\norm{\tilde{\Phi}_b}_{\calH^N} \rightarrow 0$ yields a positive lower bound:
\begin{equation*}
    \lim_{E(\Phi_b) \rightarrow 0^+} D(\Phi_a, \tilde{\Phi}_b) = \norm{\Phi_a}_{\calH^N} \ge \frac{1}{\sqrt{N}} E(\Phi_a) > 0
\end{equation*}
Since the numerator is bounded below by a positive constant, dividing by $E(\Phi_b) \to 0^+$ strictly diverges to $\infty$.
\end{proof}

\begin{corollary}[Locality of the Projection Error]
\label{cor:locality}
For a structural reduction modifying a localized subset of neurons $\calS \subset \{1, \dots, N\}$ (e.g., $|\calS|=1$ for pruning, $|\calS|=2$ for merging), the discrete projection bound evaluates over the perturbed subspace:
\begin{equation}
    \calJ_{\text{bound}}(\Phi_a, \Phi_b) = c(\Phi_a) \frac{\sqrt{\sum_{k \in \calS} \norm{f_k^{(a)} - \tilde{f}_k^{(b)}}_{\calH}^2}}{E(\Phi_b)}
\end{equation}
where $\tilde{f}_k^{(b)}$ are the components of the pre-deletion target state $\tilde{\Phi}_b \in \calH^N$.
\end{corollary}

\begin{proof}
By Theorem~\ref{thm:discrete_subspace}, the Euclidean distance expands as:
\begin{equation*}
    D(\Phi_a, \tilde{\Phi}_b) = \sqrt{\sum_{k \in \calS} \norm{f_k^{(a)} - \tilde{f}_k^{(b)}}_{\calH}^2 + \sum_{j \notin \calS} \norm{f_j^{(a)} - \tilde{f}_j^{(b)}}_{\calH}^2}
\end{equation*}
For any unperturbed coordinate $j \notin \calS$, $f_j^{(a)} = \tilde{f}_j^{(b)}$, meaning the second summation vanishes. 

The terminal capacity $E(\Phi_b)$ is also computed locally via $E(\Phi_b) = E(\Phi_a) - \sum_{k \in \calS} \norm{f_k^{(a)}}_{\calH} + \sum_{m \in \calS_{\text{new}}} \norm{f_m^{(b)}}_{\calH}$, where $\calS_{\text{new}}$ represents newly generated components (e.g., $\{f_p\}$). Because $|\calS|$ and $|\calS_{\text{new}}|$ depend solely on the localized operation, evaluating $\calJ_{\text{bound}}$ requires $\calO(1)$ operations, given $E(\Phi_a)$ is cached.
\end{proof}

\subsection{Generating Parent Neuron}

\begin{theorem}[Exact Optimal Parent Direction]
\label{thm:optimal_parent}
For a fixed magnitude $s > 0$, the inner optimization for merging neurons $f_i, f_j \in \calN$:
\begin{equation}
	\label{eq:inner_opt_exact}
    \psi^* = \argmin_{\substack{\psi \in \calN \\ \norm{\psi}_{\calH} = 1}} \frac{\sqrt{\norm{s \psi - f_i}_{\calH}^2 + \norm{s \psi - f_j}_{\calH}^2}}{E_a - \norm{f_i}_{\calH} - \norm{f_j}_{\calH} + s} 
\end{equation}
is minimized by the parameterized function:
\begin{equation}
\psi^* = \frac{\Psi(\uVect^* \cdot \tilde{\xVect})}{\sqrt{K(\uVect^*,\uVect^*)}} \vVect^*
\end{equation}
where the unit vectors $\uVect^* \in \R^{n+1}$ and $\vVect^* \in \R^c$ are given by:
\begin{equation}
\label{eq:opt_u_exact}
\uVect^* = \argmax_{\substack{\norm{\uVect}_2=1 \\ K(\uVect, \uVect) > 0}} \frac{\norm{\sum_{k \in \{i,j\}} K(\uVect, \tilde{\wVect}^k_{\text{in}}) \wVect^k_{\text{out}}}_2}{\sqrt{K(\uVect,\uVect)}}
\end{equation}
\begin{equation}
\vVect^* = \frac{\sum_{k \in \{i,j\}} K(\uVect^*, \tilde{\wVect}^k_{\text{in}}) \wVect^k_{\text{out}}}{\norm{\sum_{k \in \{i,j\}} K(\uVect^*, \tilde{\wVect}^k_{\text{in}}) \wVect^k_{\text{out}}}_2}
\end{equation}
\end{theorem}

\begin{proof}
Since $s>0$ and the denominator of (\ref{eq:inner_opt_exact}) is positive and independent of $\psi$ (due to $\norm{\psi}_{\calH} = 1$), minimizing the objective is equivalent to minimizing the squared numerator:
\begin{align}
    \norm{s \psi - f_i}_{\calH}^2 + \norm{s \psi - f_j}_{\calH}^2 &= 2 \norm{s \psi}_{\calH}^2 + \norm{f_i}_{\calH}^2 + \norm{f_j}_{\calH}^2 - 2 \inner{s \psi}{f_i + f_j}_{\calH} \nonumber \\
    \label{eq:expanded_numerator}
    &= 2 s^2 + \norm{f_i}_{\calH}^2 + \norm{f_j}_{\calH}^2 - 2 s \inner{\psi}{f_i + f_j}_{\calH}
\end{align}
Because $s > 0$ and $\norm{f_i}_{\calH}, \norm{f_j}_{\calH}$ are constant with respect to $\psi$, this reduces to maximizing the inner product:
\begin{equation}
\label{eq:alignment_reduced}
\psi^* = \argmax_{\substack{\psi \in \calN \\ \norm{\psi}_{\calH} = 1}} \inner{\psi}{f_i + f_j}_{\calH}
\end{equation}

To enforce $\psi \in \calN$, we decompose it into unit directions $\uVect \in \R^{n+1}, \vVect \in \R^c$ and magnitudes $\alpha, \beta > 0$. By the PH-1 property of $\Psi$, $\psi = \beta \vVect \Psi(\alpha \uVect \cdot \tilde{\xVect}) = \alpha \beta \vVect \Psi(\uVect \cdot \tilde{\xVect})$. 

The unit-norm constraint $\norm{\psi}_{\calH} = 1$ requires:
\begin{equation*}
    1 = (\alpha \beta)^2 \norm{\vVect}_2^2 \E_{\xVect \sim P_\calX} \left[ \Psi^2( \uVect \cdot \tilde{\xVect}) \right] = (\alpha \beta)^2 K(\uVect,\uVect) \implies \alpha \beta = \frac{1}{\sqrt{K(\uVect,\uVect)}}
\end{equation*}
Assuming $K(\uVect, \uVect) > 0$ (otherwise $\psi = \zeroVect$), substituting $\alpha \beta$ gives:
\begin{equation*}
    \psi = \frac{\Psi(\uVect \cdot \tilde{\xVect})}{\sqrt{K(\uVect,\uVect)}} \vVect
\end{equation*}

Substituting this into (\ref{eq:alignment_reduced}) and expanding the tensor inner product $\inner{g \otimes \vVect}{h \otimes \wVect}_\calH = \inner{g}{h}_{\calH_{\text{in}}} \inner{\vVect}{\wVect}_2$ yields:
\begin{align}
    \inner{\frac{\Psi(\uVect \cdot \tilde{\xVect})}{\sqrt{K(\uVect,\uVect)}} \vVect}{f_i + f_j}_\calH &= \sum_{k \in \{i,j\}} \inner{\frac{\Psi(\uVect \cdot \tilde{\xVect})}{\sqrt{K(\uVect,\uVect)}} \vVect}{\wVect^k_{\text{out}} \Psi(\tilde{\wVect}^k_{\text{in}} \cdot \tilde{\xVect}) }_\calH \nonumber \\
    &= \frac{1}{\sqrt{K(\uVect,\uVect)}} \sum_{k \in \{i,j\}} \E_{\xVect \sim P_\calX} \left[ \Psi(\uVect \cdot \tilde{\xVect}) \Psi(\tilde{\wVect}^k_{\text{in}} \cdot \tilde{\xVect}) \right] \inner{\vVect}{\wVect^k_{\text{out}}}_2 \nonumber \\
    \label{eq:alignment_kernelized}
    &= \inner{\vVect}{\frac{1}{\sqrt{K(\uVect,\uVect)}} \sum_{k \in \{i,j\}} K(\uVect, \tilde{\wVect}^k_{\text{in}}) \wVect^k_{\text{out}}}_2
\end{align}

Let $\zVect(\uVect) = \sum_{k \in \{i,j\}} K(\uVect, \tilde{\wVect}^k_{\text{in}}) \wVect^k_{\text{out}}$. For a fixed $\uVect$, maximizing $\inner{\vVect}{\zVect(\uVect)}_2$ subject to $\norm{\vVect}_2 = 1$ requires $\vVect$ to align with $\zVect(\uVect)$ via the Cauchy-Schwarz inequality:
\begin{equation*}
    \vVect^*(\uVect) = \frac{\zVect(\uVect)}{\norm{\zVect(\uVect)}_2} 
\end{equation*}
Substituting $\vVect^*(\uVect)$ back into (\ref{eq:alignment_kernelized}) simplifies the inner product to $\norm{\zVect(\uVect)}_2 / \sqrt{K(\uVect,\uVect)}$, leaving $\uVect^*$ as the sole maximizer in Equation~(\ref{eq:opt_u_exact}).
\end{proof}

\begin{proposition}[Separability and Homogeneity of PH-1 Kernels]
\label{prop:separability_kernel}
Let $\Psi: \R \to \R$ be a PH-1 activation function $\forall c>0, \Psi(cz) = c\Psi(z)$. Let the functional kernel over an isotropic probability distribution $P_{\calV}$ be $K(\xVect, \yVect) = \E_{\vVect \sim P_{\calV}} [\Psi(\xVect^T \vVect) \Psi(\yVect^T \vVect)]$. For any unit vector $\xVect$ $\norm{\xVect}_2=1$ and any $\yVect \neq \zeroVect$, the kernel factorizes as:
\begin{equation}
    K(\xVect, \yVect) = \norm{\yVect}_2\, k(\rho) \quad \text{where} \quad \rho = \inner{\xVect}{\frac{\yVect}{\norm{\yVect}_2}}_2
\end{equation}
\end{proposition}

\begin{proof}
Let $\hat{\yVect} = \yVect / \norm{\yVect}_2$. By the PH-1 property and linearity of expectation, the positive magnitude $\norm{\yVect}_2$ factors out:
\begin{equation*}
    K(\xVect, \yVect) = \E_{\vVect \sim P_{\calV}} [\Psi(\xVect^T \vVect) \Psi(\norm{\yVect}_2 \hat{\yVect}^T \vVect)] = \norm{\yVect}_2 \E_{\vVect \sim P_{\calV}} [\Psi(\xVect^T \vVect) \Psi(\hat{\yVect}^T \vVect)]
\end{equation*}

Since $P_{\calV}$ is isotropic, it is invariant under orthogonal transformations. Let $\Rmat$ be an orthogonal matrix such that $\Rmat \xVect = \eVect_1$ and $\Rmat \hat{\yVect} = \rho \eVect_1 + \sqrt{1-\rho^2} \eVect_2$, where $\rho = \inner{\xVect}{\hat{\yVect}}_2$. 

Applying the change of variables $\uVect = \Rmat\vVect$, we have $\uVect \sim P_{\calV}$. Substituting $\xVect^T \vVect = \xVect^T \Rmat^T \uVect = (\Rmat\xVect)^T \uVect = \eVect_1^T \uVect = u_1$ and similarly $\hat{\yVect}^T \vVect = \hat{\yVect}^T \Rmat^T \uVect = (\Rmat\hat{\yVect})^T \uVect = \rho u_1 + \sqrt{1-\rho^2} u_2$ yields:
\begin{equation*}
    \E_{\vVect \sim P_{\calV}} [\Psi(\xVect^T \vVect) \Psi(\hat{\yVect}^T \vVect)] = \E_{\uVect \sim P_{\calV}} [\Psi(u_1) \Psi(\rho u_1 + \sqrt{1-\rho^2} u_2)] 
\end{equation*}
Because this expectation depends only on $\rho$, we can define it as $k(\rho)$, establishing the separable form $K(\xVect, \yVect) = \norm{\yVect}_2 k(\rho)$.
\end{proof}

\begin{proposition}[Boundary Identity]
\label{prop:boundary_identity}
Let $\Psi$ be a PH-1 function. Under a standard bivariate Gaussian reference distribution with correlation $\rho$, the induced angular kernel $k(\rho) = \E[\Psi(X)\Psi(Y)]$ satisfies $k(1) = k'(1)$.
\end{proposition}

\begin{proof}
By the PH-1 property, $\Psi(x) = C_+ x \mathbb{I}_{x>0} + C_- x \mathbb{I}_{x<0}$ for constants $C_+ = \Psi(1)$ and $C_- = -\Psi(-1)$. Its derivative (defined almost everywhere) is $\Psi'(x) = C_+ \mathbb{I}_{x>0} + C_- \mathbb{I}_{x<0}$.

Let $(X, Y)$ be standard bivariate Gaussian with correlation $\rho$. At $\rho=1$, $X=Y$ almost surely, giving $k(1) = \E[\Psi(X)^2]$. By Price's Theorem $\frac{\partial}{\partial \rho}\E[f(X)g(Y)] = \E[f'(X)g'(Y)]$, we have $k'(\rho) = \E[\Psi'(X)\Psi'(Y)]$, which at $\rho=1$ gives $k'(1) = \E[\Psi'(X)^2]$.

Evaluating these expectations:
\begin{align*}
    k(1) &= \E\big[ (C_+ X \mathbb{I}_{X>0} + C_- X \mathbb{I}_{X<0})^2 \big] = C_+^2 \E[X^2 \mathbb{I}_{X>0}] + C_-^2 \E[X^2 \mathbb{I}_{X<0}] \\
    k'(1) &= \E\big[ (C_+ \mathbb{I}_{X>0} + C_- \mathbb{I}_{X<0})^2 \big] = C_+^2 \E[\mathbb{I}_{X>0}] + C_-^2 \E[\mathbb{I}_{X<0}]
\end{align*}
By the symmetry of the standard normal distribution, $\E[X^2 \mathbb{I}_{X>0}] = \E[X^2 \mathbb{I}_{X<0}] = \frac{1}{2} \E[X^2] = \frac{1}{2}$ and $\E[\mathbb{I}_{X>0}] = \E[\mathbb{I}_{X<0}] = \frac{1}{2}$. Substituting these yields $k(1) = \frac{1}{2}(C_+^2 + C_-^2)$ and $k'(1) = \frac{1}{2}(C_+^2 + C_-^2)$. Thus, $k(1) = k'(1)$.
\end{proof}

\begin{proposition}[Kernel Positivity at Boundary]
\label{prop:kernel_positivity}
For any non-trivial PH-1 function $\Psi(x) \not\equiv 0$, the induced angular kernel satisfies $k(1) > 0$.
\end{proposition}

\begin{proof}
As derived in Proposition~\ref{prop:boundary_identity}, $k(1) = \frac{1}{2}(C_+^2 + C_-^2)$. Since $\Psi$ is non-trivial, at least one of $C_+$ or $C_-$ is non-zero. Consequently, $C_+^2 + C_-^2 > 0$, implying $k(1) > 0$.
\end{proof}

\subsection{Block Eviction}

\begin{proposition}
\label{prop:macro_capacity}
Let $\Omega = (\Phi_1, \dots, \Phi_n)$ be a macroscopic state, where each $\Phi_i \in \calH_i^{N_i}$ is a layer state consisting of rank-1 operators $f_{i,j} \in \calH_i$. If the macroscopic capacity $E(\Omega)$ is a symmetric, separable, and homogeneous function of the capacities of its constituent operators, and is invariant under arbitrary operator partitioning, it is uniquely defined as the sum of the constituent Hilbert-Schmidt norms:
\begin{equation}
E(\Omega) = \sum_{i=1}^n E(\Phi_i) = \sum_{i=1}^n \sum_{j=1}^{N_i} \norm{f_{i,j}}_{\calH_i} \,.
\end{equation}
\end{proposition}

\begin{proof}
Because $E(\Omega)$ is a symmetric, separable, and homogeneous functional of the individual capacities, it must take the form of an $L_p$ norm:
\begin{equation}
E(\Omega) = \left( \sum_{i=1}^n \sum_{j=1}^{N_i} \norm{f_{i,j}}_{\calH_i}^p \right)^{\frac{1}{p}} \quad \text{for some } p > 0 \,.
\end{equation}

The capacity invariance axiom requires that $E(\Omega)$ remains invariant if any operator is partitioned. Consider an arbitrary operator $f_0 \in \Omega$. Partitioning $f_0$ into $M$ identical fractional operators yields $M$ operators, each defined as $f_0 / M$. By the positive homogeneity of the Hilbert-Schmidt norm, each fractional operator has capacity $\norm{f_0 / M}_{\calH} = \frac{1}{M} \norm{f_0}_{\calH}$.

Evaluating the capacity of this partitioned subset under the $L_p$ functional yields:
\begin{equation}
E_{\text{subset}} = \left( \sum_{m=1}^M \norm{\frac{f_0}{M}}_{\calH}^p \right)^{\frac{1}{p}} = \left( M \left( \frac{1}{M} \norm{f_0}_{\calH} \right)^p \right)^{\frac{1}{p}} = M^{\frac{1-p}{p}} \norm{f_0}_{\calH} \,.
\end{equation}

For the macroscopic capacity to remain invariant for any partition scale $M \ge 1$, we must have $E_{\text{subset}} = \norm{f_0}_{\calH}$. This equality holds if and only if $M^{\frac{1-p}{p}} = 1$. Since $M$ is arbitrary, the exponent must be zero:
\begin{equation}
\frac{1-p}{p} = 0 \implies p = 1 \,.
\end{equation}

Substituting $p=1$ reduces the $L_p$ norm to the $L_1$ sum of scalar norms, yielding $E(\Omega) = \sum_{i=1}^n \sum_{j=1}^{N_i} \norm{f_{i,j}}_{\calH_i}$.
\end{proof}

\section{Derivation of Physical BN Parameters}
\label{app:bn_inversion_details}

This appendix provides the complete derivation for recovering the physical BN parameters and raw weights from the effective parameters, expanding upon Section~\ref{sec:bn_invert}. 

Since the parent direction $\widehat{\uVect}$ lies within the 2D subspace spanned by the augmented children, there exist projection coefficients $c_1, c_2$ such that:
\begin{equation}
\wVect_{p, \text{in}}^{\text{eff}} = c_1 \wVect_{\text{in},i}^{\text{eff}} + c_2 \wVect_{\text{in},j}^{\text{eff}} \quad \text{and} \quad b_p = c_1 b_i + c_2 b_j \,.
\end{equation}
By the definition of the pre-activation signal $y = (\wVect_{\text{in}}^{\text{eff}})^T \xVect + b$, it follows that $y_p = c_1 y_i + c_2 y_j$. The physical BN parameters $\beta_p$ and $\gamma_p$ correspond to the mean and standard deviation of $y_p$. Taking the expectation and variance yields:
\begin{align}
\beta_p &= \E[y_p] = c_1 \E[y_i] + c_2 \E[y_j] = c_1 \beta_i + c_2 \beta_j \\
\gamma_p^2 &= \Var[y_p] = c_1^2 \gamma_i^2 + c_2^2 \gamma_j^2 + 2 c_1 c_2 \Cov(y_i,y_j) \,.
\end{align}
Substituting the closed-form covariance $\Cov(y_i , y_j) = |\gamma_i| |\gamma_j| \hat{\rho}_{ij}$ (where $\hat{\rho}_{ij}$ is defined in Equation~\ref{eq:rho_hat_maintext}) gives:
\begin{equation}
\label{eq:gamma_both_signs}
\gamma_p = \pm \sqrt{c_1^2 \gamma_i^2 + c_2^2 \gamma_j^2 + 2 c_1 c_2 |\gamma_i| |\gamma_j| \hat{\rho}_{ij}} \,.
\end{equation}
Next, we determine the raw weights $\wVect_{\text{in},p}^{\text{raw}}$ and dataset statistics $\mu_p, \sigma_p$. By definition (\ref{eq:eff_params}):
\begin{equation}
\label{eq:eff_params_p}
\wVect_{\text{in},p}^{\text{eff}} = \frac{\gamma_p}{\sqrt{\sigma_p^2 + \epsilon}} \wVect_{\text{in},p}^{\text{raw}} \qquad \text{and} \qquad b_p = \beta_p - \frac{\gamma_p}{\sqrt{\sigma_p^2 + \epsilon}} \mu_p \,.
\end{equation}
This system is under-constrained. To resolve the ambiguity, we anchor the variance such that $\sigma_p^2 = \max(0, \gamma_p^2 - \epsilon)$. As noted in the main text, we focus on the active regime $\gamma_p^2 \ge \epsilon$ and defer the edge case $\gamma_p^2 < \epsilon$ to Appendix~\ref{sec:numerical_stability_BN}. In the active regime, the physical scaling factor simplifies directly to its sign:
\begin{equation}
\frac{\gamma_p}{\sqrt{\sigma_p^2 + \epsilon}} = \frac{\gamma_p}{|\gamma_p|} = \sign(\gamma_p) \,.
\end{equation}
Substituting this back into (\ref{eq:eff_params_p}) yields:
\begin{equation}
\wVect_{\text{in},p}^{\text{eff}} = \sign(\gamma_p) \wVect_{\text{in},p}^{\text{raw}} \qquad \text{and} \qquad b_p = \beta_p - \sign(\gamma_p) \mu_p \,.
\end{equation}
Multiplying both equations by $\sign(\gamma_p)$ (and noting that $\sign(\gamma_p)^2 = 1$ for $\gamma_p \neq 0$) isolates the raw physical parameters:
\begin{equation}
\wVect_{\text{in},p}^{\text{raw}} = \sign(\gamma_p) \wVect_{\text{in},p}^{\text{eff}} \qquad \text{and} \qquad \mu_p = \sign(\gamma_p)(\beta_p - b_p) \,.
\end{equation}
Because $\gamma_p$ in (\ref{eq:gamma_both_signs}) can take either sign, we arbitrarily choose the positive root $\gamma_p > 0$ without loss of generality. Under this choice, $\sign(\gamma_p) = 1$, and the recovered parameters simplify to:
\begin{equation}
\wVect_{\text{in},p}^{\text{raw}} =  \wVect_{\text{in},p}^{\text{eff}} \qquad , \qquad \mu_p = c_1 \beta_i + c_2 \beta_j - b_p \qquad,\qquad \sigma_p = \gamma_p = \sqrt{c_1^2 \gamma_i^2 + c_2^2 \gamma_j^2 + 2 c_1 c_2 |\gamma_i| |\gamma_j| \hat{\rho}_{ij}} \,.
\end{equation}

To prove that the physical forward pass is invariant to the chosen sign of $\gamma_p$, we substitute the recovered parameters back into the BN inference equation. The true pre-activation signal $y_p$ is computed as:
\begin{align}
y_p &= \frac{\gamma_p}{\sqrt{\sigma_p^2 + \epsilon}} \left( (\wVect_{\text{in},p}^{\text{raw}})^T \xVect - \mu_p \right) + \beta_p \nonumber \\
    &= \sign(\gamma_p) \Big[ \big(\sign(\gamma_p)\wVect_{\text{in},p}^{\text{eff}}\big)^T \xVect - \sign(\gamma_p)(\beta_p - b_p) \Big] + \beta_p \nonumber \\
    &= \sign(\gamma_p)^2 \Big[ (\wVect_{\text{in},p}^{\text{eff}})^T \xVect - (\beta_p - b_p) \Big] + \beta_p \,.
\end{align}
Since $\sign(\gamma_p)^2 = 1$, the shift $\beta_p$ cancels out:
\begin{equation}
y_p = (\wVect_{\text{in},p}^{\text{eff}})^T \xVect - \beta_p + b_p + \beta_p = (\wVect_{\text{in},p}^{\text{eff}})^T \xVect + b_p \,.
\end{equation}
Thus, regardless of the polarity of $\gamma_p$, the deployed pre-activation robustly realizes the target effective geometry.

\section{Kernel Formulation}
\label{sec:kernel}
Expressing the Hilbert space inner products and norms via a kernel decouples finite-dimensional vector operations from infinite-dimensional function integrals. We define the correlation kernel $K(i, j)$ of two neurons as:
\begin{equation}
\label{eq:kernel}
K(i, j) \triangleq \E_{\xVect \sim P_{\calX}} [ \Psi(y_i) \Psi(y_j) ]
\end{equation}
Under this definition, the inner product in $\calH$ evaluates to:
\begin{align}
\label{eq:hilbert_inner_prod}
    \inner{f_i}{f_j}_{\calH} &= \inner{g_i \otimes \wVect_{\text{out},i}}{g_j \otimes \wVect_{\text{out},j}}_{\calH} \nonumber \\
    &= \inner{g_i}{g_j}_{\calH_{\text{in}}} \inner{\wVect_{\text{out},i}}{\wVect_{\text{out},j}}_{\R^c} \nonumber \\
    &= K(i,j) \inner{\wVect_{\text{out},i}}{\wVect_{\text{out},j}}_{\R^c}
\end{align}
Consequently, the capacity of a single neuron is formulated as:
\begin{equation}
    \norm{f_i}_{\calH} = \norm{\wVect_{\text{out},i}}_2 \sqrt{K(i, i)}
\end{equation}
where $K(i, i) = \E_{\xVect \sim P_{\calX}} [ \Psi(y_i)^2 ]$ represents the expected squared energy of the activation signal. 

Evaluating $K(i,j)$ requires integrating over a high-dimensional space. We derive closed-form analytical expressions for these integrals under the \ReLU activation.\footnote{This evaluation connects to the Arc-Cosine Kernel family of order $n=1$ \cite{cho2009kernel}. Because HOPE relies on the translation vector $\beta$ to capture the full statistical state of the signal, we provide the formal proof for the biased formulation here to ensure the paper remains self-contained.}

\subsection{Pre-Activation Distribution}

To evaluate the integrals analytically, we first determine the distribution of the pre-activation signal $y_i = (\wVect_{\text{in},i}^{\text{eff}})^T \xVect + b_i$ under the surrogate distribution $P_{\calX} = \calN(\hat{\muVect}_x, \hat{\Sigma}_x)$. 

Because $y_i$ is an affine transformation of a Gaussian vector, it is univariate Gaussian. Recall from Section~\ref{sec:neural_signal_distribution} that the surrogate mean is defined as $\hat{\muVect}_x = \Wmat_{\text{raw}}^+ \muVect_{\text{BN}}$. Because $\wVect_{\text{raw},i}^T \hat{\muVect}_x$ is the $i$-th element of $\Wmat_{\text{raw}} \hat{\muVect}_x$, we have $\Wmat_{\text{raw}} \hat{\muVect}_x = \Wmat_{\text{raw}} \Wmat_{\text{raw}}^+ \muVect_{\text{BN}}$. Since the empirical BN mean vector $\muVect_{\text{BN}}$ resides in the column space of $\Wmat_{\text{raw}}$, this projection simplifies to $\muVect_{\text{BN}}$. Thus, the expected value of the raw projection recovers the empirical mean: $\E[\wVect_{\text{raw},i}^T \xVect] = \mu_i$. Similarly, the surrogate covariance $\hat{\Sigma}_x$ is constrained to satisfy $\Var(\wVect_{\text{raw},i}^T \xVect) = \sigma_i^2$.

Using the effective parameters $\wVect_{\text{in},i}^{\text{eff}} = \frac{\gamma_i}{\sqrt{\sigma_i^2 + \epsilon}} \wVect_{\text{raw},i}$ and $b_i = \beta_i - \frac{\gamma_i \mu_i}{\sqrt{\sigma_i^2 + \epsilon}}$, the mean of $y_i$ evaluates to:
\begin{equation}
    \E[y_i] = \frac{\gamma_i}{\sqrt{\sigma_i^2 + \epsilon}} \E[\wVect_{\text{raw},i}^T \xVect] + b_i = \frac{\gamma_i \mu_i}{\sqrt{\sigma_i^2 + \epsilon}} + \beta_i - \frac{\gamma_i \mu_i}{\sqrt{\sigma_i^2 + \epsilon}} = \beta_i
\end{equation}
Similarly, its variance evaluates to:
\begin{equation}
    \Var(y_i) = \frac{\gamma_i^2}{\sigma_i^2 + \epsilon} \Var(\wVect_{\text{raw},i}^T \xVect) = \gamma_i^2 \left( \frac{\sigma_i^2}{\sigma_i^2 + \epsilon} \right)
\end{equation}
Assuming the numerical stability constant $\epsilon$ is negligible compared to the data variance $\epsilon \to 0$, the variance simplifies to $\gamma_i^2$. Thus, under $P_{\calX}$, the pre-activation density is:
\begin{equation}
\label{eq:pre-activation_density_appendix}
y_i \sim \calN(\beta_i, \gamma_i^2)
\end{equation}

\begin{alertbox}{Realizing Empirical Constraints in Practice}
Adapting HOPE to unnormalized networks requires a lightweight empirical calibration pass over a small, unlabeled data batch to measure the marginal pre-activation statistics $(\mu_i, \sigma_i^2)$. Defining the effective scaling as $\gamma_i = \sigma_i$ and the shift as $\beta_i = \mu_i + b_{\text{raw}, i}$ anchors the evaluation to the data manifold and recovers the decoupling required to evaluate the kernel analytically.\footnote{Networks using LayerNorm or GroupNorm still require this empirical calibration pass to determine the marginal channel statistics. Because these layers condition the signal variance, their calibration converges with remarkably few samples, and the resulting pre-activations adhere to the Gaussian surrogate assumptions.}
\end{alertbox}

\subsection{Self-Kernel}

\begin{proposition}[Dimensionality Reduction]
\label{prop:decouple_weights}
For a Gaussian pre-activation $y_i \sim \calN(\beta_i, \gamma_i^2)$ and any activation function $\Psi$, the expected energy $K(i,i) = \E[\Psi(y_i)^2]$ reduces to a 1D integral, independent of the input space dimensionality.
\end{proposition}
\begin{proof}
By definition, $K(i,i) = \E_{\xVect \sim P_{\calX}} [\Psi(y_i)^2]$. Because $y_i$ is a scalar random variable with density $p(y_i) = \calN(y_i; \beta_i, \gamma_i^2)$, the expectation evaluates as $\int_{-\infty}^\infty \Psi(y)^2 p(y) dy$, decoupling the computation from the ambient space $\calX$.
\end{proof}

\begin{theorem}[Closed-Form Self-Kernel for ReLU]
\label{theorem:self_kernel}
For $y_i \sim \calN(\beta_i, \gamma_i^2)$ with $|\gamma_i| > 0$, the expected squared energy for $\Psi(y_i) = \max(0, y_i)$ evaluates to:
\begin{equation}
    K(i, i) = (\gamma_i^2 + \beta_i^2)\normCDF\left(\frac{\beta_i}{|\gamma_i|}\right) + \beta_i|\gamma_i|\normPDF\left(\frac{\beta_i}{|\gamma_i|}\right)
\end{equation}
where $\normPDF$ and $\normCDF$ are the standard Normal PDF and CDF.
\end{theorem}
\begin{proof}
By Proposition~\ref{prop:decouple_weights}, $K(i, i) = \int_{0}^{\infty} y^2 p(y) dy$. Applying the substitution $z = \frac{y-\beta_i}{|\gamma_i|}$, we have $y = |\gamma_i| z + \beta_i$ and $dy = |\gamma_i| dz$. Setting the integration limit $c = \frac{\beta_i}{|\gamma_i|}$, we obtain:
\begin{equation*}
    K(i, i) = \int_{-c}^{\infty} (|\gamma_i| z + \beta_i)^2 \normPDF(z) dz = \beta_i^2 \int_{-c}^{\infty} \normPDF(z) dz + 2\beta_i|\gamma_i| \int_{-c}^{\infty} z\normPDF(z) dz + \gamma_i^2 \int_{-c}^{\infty} z^2\normPDF(z) dz
\end{equation*}
We evaluate each term using integration by parts and the identity $\normPDF'(z) = -z\normPDF(z)$:
\begin{itemize}
    \item $\int_{-c}^{\infty} \normPDF(z) dz = 1 - \normCDF(-c) = \normCDF(c)$
    \item $\int_{-c}^{\infty} z\normPDF(z) dz = [-\normPDF(z)]_{-c}^{\infty} = \normPDF(c)$
    \item $\int_{-c}^{\infty} z^2\normPDF(z) dz = [-z\normPDF(z)]_{-c}^{\infty} + \int_{-c}^{\infty} \normPDF(z) dz = \normCDF(c) - c\normPDF(c)$
\end{itemize}
Substituting these evaluations yields:
\begin{align*}
    K(i, i) &= \beta_i^2 \normCDF(c) + 2\beta_i|\gamma_i|\normPDF(c) + \gamma_i^2 \big( \normCDF(c) - c\normPDF(c) \big) \\
    &= (\gamma_i^2 + \beta_i^2)\normCDF(c) + \big( 2\beta_i|\gamma_i| - \gamma_i^2 c \big) \normPDF(c)
\end{align*}
Since $\gamma_i^2 c = \gamma_i^2 (\frac{\beta_i}{|\gamma_i|}) = |\gamma_i|\beta_i$, the coefficient for $\normPDF(c)$ simplifies to $\beta_i|\gamma_i|$. Substituting $c = \frac{\beta_i}{|\gamma_i|}$ completes the proof.
\end{proof}

\subsection{Cross-Kernel}
\label{sec:cross-kernel}

\subsubsection{The Local Pairwise Surrogate Distribution}
Let $\rho_{\text{eff}} = \frac{\inner{\wVect_{\text{in},i}^{\text{eff}}}{\wVect_{\text{in},j}^{\text{eff}}}}{\norm{\wVect_{\text{in},i}^{\text{eff}}}_2 \norm{\wVect_{\text{in},j}^{\text{eff}}}_2}$ be the cosine similarity between the effective input weights. The cross-kernel relies on the joint pre-activation distribution of $y_i$ and $y_j$. Because optimizing the global covariance $\hat{\Sigma}_x$ is computationally prohibitive, we restrict the maximum entropy formulation to the local $2 \times 2$ subspace spanned by the neuron pair.

\begin{proposition}[Pairwise Warped Correlation]
\label{prop:warped_corr}
Under a local pairwise maximum entropy surrogate, the correlation $\hat{\rho}_{ij}$ between $y_i$ and $y_j$ is given analytically by:
\begin{equation}
\label{eq:def_hat_rho}
    \hat{\rho}_{ij} = \frac{2\kappa}{1 + \sqrt{1 + 4\kappa^2}}
\end{equation}
where $\kappa$ is the blending constant uniquely defined by the input weight geometry and empirical standard deviations\footnote{To prevent undefined $0/0$ states, this evaluation is restricted to active features where the effective weight norm is non-zero $\norm{\wVect_{\text{in}}^{\text{eff}}}_2 > 0$.}:
\begin{equation}
    \kappa = \left( \frac{\rho_{\text{eff}}}{1 - \rho_{\text{eff}}^2} \right) \left( \frac{|\gamma_i|}{\norm{\wVect_{\text{in},i}^{\text{eff}}}_2} \right) \left( \frac{|\gamma_j|}{\norm{\wVect_{\text{in},j}^{\text{eff}}}_2} \right)
\end{equation}
\end{proposition}

\begin{proof}
Let $\Wmat = [\wVect_{\text{raw},i}, \wVect_{\text{raw},j}] \in \R^{n \times 2}$ be the raw weights and $\Gmat = \Wmat^T \Wmat$ be their Gram matrix. The maximum entropy distribution constrained by the variances of these two projections has a precision matrix $\hat{\Sigma}_x^{-1} = \Imat + \Wmat \Lambdamat \Wmat^T$, where $\Lambdamat = \diag(\lambda_i, \lambda_j)$ contains the Lagrange multipliers. The output covariance is $\Cmat = \Wmat^T \hat{\Sigma}_x \Wmat = \Wmat^T (\Imat + \Wmat \Lambdamat \Wmat^T)^{-1} \Wmat$. 

Applying the Woodbury matrix identity to the inner inverse yields:
\begin{equation*}
    (\Imat + \Wmat \Lambdamat \Wmat^T)^{-1} = \Imat - \Wmat (\Lambdamat^{-1} + \Wmat^T \Wmat)^{-1} \Wmat^T = \Imat - \Wmat (\Lambdamat^{-1} + \Gmat)^{-1} \Wmat^T
\end{equation*}
Substituting this back into the expression for $\Cmat$:
\begin{align*}
    \Cmat &= \Wmat^T \left[ \Imat - \Wmat (\Lambdamat^{-1} + \Gmat)^{-1} \Wmat^T \right] \Wmat = \Gmat - \Gmat (\Lambdamat^{-1} + \Gmat)^{-1} \Gmat
\end{align*}
We factor out $\Gmat$ to simplify the expression:
\begin{align*}
    \Cmat &= \Gmat \left[ \Imat - (\Lambdamat^{-1} + \Gmat)^{-1} \Gmat \right] \\
    &= \Gmat \left[ (\Lambdamat^{-1} + \Gmat)^{-1} (\Lambdamat^{-1} + \Gmat) - (\Lambdamat^{-1} + \Gmat)^{-1} \Gmat \right] \\
    &= \Gmat (\Lambdamat^{-1} + \Gmat)^{-1} \Lambdamat^{-1}
\end{align*}
Inverting both sides yields the localized precision matrix $\Cmat^{-1}$:
\begin{equation*}
    \Cmat^{-1} = \Lambdamat (\Lambdamat^{-1} + \Gmat) \Gmat^{-1} = \Lambdamat \Lambdamat^{-1} \Gmat^{-1} + \Lambdamat \Gmat \Gmat^{-1} = \Gmat^{-1} + \Lambdamat
\end{equation*}
Because $\Lambdamat$ is diagonal, its addition only perturbs the diagonal entries of $\Gmat^{-1}$. Therefore, the off-diagonal entries are identical: $[\Cmat^{-1}]_{12} = [\Gmat^{-1}]_{12}$.

Let $\rho_{\text{raw}}$ be the correlation in $\Gmat$ and $r_{\text{raw}}$ be the target correlation in $\Cmat$. Evaluating the inverse off-diagonal elements of these $2 \times 2$ matrices yields:
\begin{equation*}
    [\Gmat^{-1}]_{12} = \frac{-\rho_{\text{raw}}}{\norm{\wVect_{\text{raw},i}}_2 \norm{\wVect_{\text{raw},j}}_2 (1 - \rho_{\text{raw}}^2)} \quad,\quad [\Cmat^{-1}]_{12} = \frac{-r_{\text{raw}}}{\sigma_i \sigma_j (1 - r_{\text{raw}}^2)}
\end{equation*}
Equating them gives:
\begin{equation*}
    \frac{r_{\text{raw}}}{1 - r_{\text{raw}}^2} = \left( \frac{\rho_{\text{raw}}}{1 - \rho_{\text{raw}}^2} \right) \left( \frac{\sigma_i}{\norm{\wVect_{\text{raw},i}}_2} \right) \left( \frac{\sigma_j}{\norm{\wVect_{\text{raw},j}}_2} \right)
\end{equation*}
By the definition of the effective weights, $\wVect_{\text{in},i}^{\text{eff}} = \frac{\gamma_i}{\sqrt{\sigma_i^2+\epsilon}} \wVect_{\text{raw},i}$. The effective correlations relate to the raw correlations via the signs of these scaling parameters: $\rho_{\text{raw}} = \sign(\gamma_i \gamma_j) \rho_{\text{eff}}$ and $r_{\text{raw}} = \sign(\gamma_i \gamma_j) \hat{\rho}_{ij}$. Substituting these into the odd function $f(x) = \frac{x}{1-x^2}$ causes the sign terms to cancel. Taking the limit as $\epsilon \to 0$, we substitute $\frac{\sigma_i}{\norm{\wVect_{\text{raw},i}}_2} = \frac{|\gamma_i|}{\norm{\wVect_{\text{in},i}^{\text{eff}}}_2}$, which yields the blended constant $\kappa$:
\begin{equation*}
    \frac{\hat{\rho}_{ij}}{1 - \hat{\rho}_{ij}^2} = \kappa \implies \kappa \hat{\rho}_{ij}^2 + \hat{\rho}_{ij} - \kappa = 0
\end{equation*}
Solving for $\hat{\rho}_{ij}$ via the quadratic formula and selecting the root that satisfies the boundary condition $\lim_{\rho_{\text{eff}} \to 0} \hat{\rho}_{ij} = 0$, we obtain $\frac{-1 + \sqrt{1 + 4\kappa^2}}{2\kappa}$. Multiplying the numerator and denominator by the conjugate $(\sqrt{1 + 4\kappa^2} + 1)$ ensures numerical stability as $\kappa \to 0$, producing the final formulation $\frac{2\kappa}{1 + \sqrt{1 + 4\kappa^2}}$.
\end{proof}

Consequently, the joint pre-activation distribution under the surrogate space is:
\begin{equation}
\label{eq:joint_pre-activation}
    \begin{pmatrix} y_i \\ y_j \end{pmatrix} \sim \calN \left( \begin{pmatrix} \beta_i \\ \beta_j \end{pmatrix}, \begin{pmatrix} \gamma_i^2 & |\gamma_i| |\gamma_j| \hat{\rho}_{ij} \\ |\gamma_i| |\gamma_j| \hat{\rho}_{ij} & \gamma_j^2 \end{pmatrix} \right)
\end{equation}

Regardless of the activation $\Psi$, any valid Cross-Kernel $K(i, j) = \E[\Psi(y_i)\Psi(y_j)]$ must satisfy three properties:
\begin{enumerate}
\item \textbf{Diagonal Consistency:} If $\hat{\rho}_{ij}=1$ and the marginals are identical ($\beta_i=\beta_j$, $|\gamma_i|=|\gamma_j|$), the interaction recovers the self-kernel: $K(i, j) = K(i, i)$.
\item \textbf{Cauchy-Schwarz Compliance:} The magnitude is bounded: $\abs{K(i, j)} \leq \sqrt{K(i, i) K(j, j)}$.
\item \textbf{Weight-Space Correlation Dependency:} The interaction is monotonic with respect to $\hat{\rho}_{ij}$.
\end{enumerate}

\subsubsection{Exact Bivariate Cross-Kernel for Biased ReLUs}
\label{sec:exact_bivar}
Evaluating $K(i, j) = \E[\max(0, y_i) \max(0, y_j)]$ requires integrating over the joint positive orthant. Defining $c_i = \frac{\beta_i}{|\gamma_i|}$ and $c_j = \frac{\beta_j}{|\gamma_j|}$, the exact cross-kernel evaluates to the closed-form moments of a truncated bivariate normal:
\begin{align}
    K(i, j) = |\gamma_i \gamma_j| \Bigg[ &(c_i c_j + \hat{\rho}_{ij}) \Phi_2(c_i, c_j; \hat{\rho}_{ij}) + c_i \normPDF(c_j) \normCDF(c_{i|j}) \nonumber \\
    &+ c_j \normPDF(c_i) \normCDF(c_{j|i}) + (1-\hat{\rho}_{ij}^2) \phi_2(c_i, c_j; \hat{\rho}_{ij}) \Bigg]
\end{align}
where $\Phi_2$ and $\phi_2$ are the standard bivariate normal CDF and PDF evaluated at $(c_i, c_j)$ with correlation $\hat{\rho}_{ij}$, $\normPDF$ and $\normCDF$ are the standard univariate normal PDF and CDF, and the conditional integration boundaries are $c_{i|j} = \frac{c_i - \hat{\rho}_{ij} c_j}{\sqrt{1-\hat{\rho}_{ij}^2}}$ and $c_{j|i} = \frac{c_j - \hat{\rho}_{ij} c_i}{\sqrt{1-\hat{\rho}_{ij}^2}}$. By centering the evaluation on $\beta_i$ rather than the weight bias $b_i$, we correctly evaluate active feature detectors in the function space.

\subsubsection{Zero-Bias Approximation for Large-Scale Networks}
Because calculating $\Phi_2$ for all neuron pairs is computationally intensive, we approximate the interaction by assuming bias shifts are negligible $\beta_i, \beta_j \approx 0$. Under this assumption, $K(i, j)$ factors into the geometric mean of the capacities scaled by a normalized interaction function $\calI(\hat{\rho}_{ij})$:
\begin{equation}
    K(i, j) \approx \calI(\hat{\rho}_{ij}) \sqrt{K(i, i) K(j, j)}
\end{equation}
For the ReLU activation, applying the Arc-Cosine kernel of order $n=1$ yields the approximate cross-kernel:
\begin{equation}
\label{eq:approx_kernel}
    K(i, j) \approx \frac{1}{\pi} \left( \sqrt{1 - \hat{\rho}_{ij}^2} + (\pi - \arccos \hat{\rho}_{ij}) \hat{\rho}_{ij} \right) \sqrt{K(i, i) K(j, j)}
\end{equation}
This formulation complies with the Cauchy-Schwarz inequality $|K(i,j)| \leq \sqrt{K(i,i)K(j,j)}$ and ensures diagonal consistency ($K(i,j) = K(i,i)$ when $\hat{\rho}_{ij}=1$).

\section{Derivations for Block Eviction}
\label{sec:app_block_eviction}

This appendix provides the mathematical framework and extended derivations for the block eviction operation introduced in Section~\ref{sec:block_eviction}.

\subsection{Generalization and Execution Degradation in Depleted Blocks}
A dedicated macro-level operation is necessary because standard granular pruning cannot fully deplete a residual block due to architectural constraints. Because the residual connection computes $X + F(X)$, the output dimension of $F(X)$ (and thus the final weight tensor $W_3$) must match that of $X$. While granular pruning can safely compress internal neurons within $W_1$ and $W_2$, pruning the output filters of $W_3$ would cause a dimensional mismatch in the element-wise addition $X + F(X)$. Consequently, granular compression leaves the output channels of $W_3$ intact. Leaving a residual pathway active with heavily depleted $W_1$ and $W_2$ degrades both model generalization and execution efficiency:

\begin{itemize}
    \item \textbf{Model Generalization.} The output of the final BN layer evaluates to $F(X) = \gammaVect \odot \frac{W_3 H_2 - \muVect}{\sqrt{\sigmaVect^2 + \epsilon}} + \betaVect$. When $H_2 \to \zeroVect$, this reduces to $F(X) = \betaVect - \gammaVect \odot \frac{\muVect}{\sqrt{\sigmaVect^2 + \epsilon}}$. As defined in (\ref{eq:eff_params}), this is the effective bias vector $\bVect$. Once the block is depleted, $\bVect$ loses its normalization purpose and instead acts as an uncalibrated bias injected into the skip connection, yielding $Y = X + \bVect$. This forcefully shifts downstream feature maps out of their calibrated domain. Since the block terminates with a ReLU activation $Z = \ReLU(X + \bVect)$, a negative $\bVect$ causes catastrophic clipping and irreversible information loss.
    \item \textbf{Execution Efficiency.} Even if $H_2$ partially survives, the shape of $W_3$ remains locked at its ambient output size. Retaining this massive parameter tensor simply to process a negligible, low-rank subspace is computationally wasteful.
\end{itemize}

Block eviction resolves these issues by projecting $F(X) \to \zeroVect$, yielding a pure identity mapping $Y = X$. Unlike $Y = X + \bVect$, residual architectures are natively robust to pure identity mappings. For instance, initializing $\gamma=0$ in the final layer of each residual branch \cite{goyal2017accurate, he2016identity} ensures blocks begin training as $F(X) = \zeroVect$.

\subsection{Derivation of the Unified Macro Cost $\calJ_{\text{evict}}$}
\label{sec:app_macro_derivation_unified}
Recall that a state tuple in layer $l$ has the form $\Phi^{(l)} = (f_1^{(l)}, \dots, f_{N_l}^{(l)}) \in \calH_l^{N_l}$. Let $\calI = (I_1, \dots, I_{d_{\text{amb}}}) \in \calH_A^{d_{\text{amb}}}$ be the tuple of identity mappings comprising the skip connection, where $\calH_A$ is the Hilbert space of operators over the ambient dimension $\mathbb{R}^{d_{\text{amb}}} \to \mathbb{R}^{d_{\text{amb}}}$. The aggregate mapping of any tuple is the sum of its constituent operators:
\begin{equation}
\calF(\Phi^{(l)}) \triangleq \sum_{i=1}^{N_l} f_i^{(l)} \qquad,\qquad \calF(\calI) = \sum_{k=1}^{d_{\text{amb}}} I_k \,.
\end{equation}
The composite block operator $\calB \in \calH_A$ is formalized as the addition of the residual pathway $\calF_{\text{res}} \in \calH_A$ and the skip connection:
\begin{equation}
    \calB = \calF_{\text{res}}\big( \Phi^{(1)}, \Phi^{(2)} \big) + \calF(\calI) \,.
\end{equation}
We define eviction iteratively: we project internal layers to null operators until the block is fully depleted. Evaluating the projection of an internal layer $\Phi^{(l)}$ over $t \in [0,1]$ induces a continuous trajectory $\Phi^{(l)}(t) \to (\zeroVect, \dots, \zeroVect)$. Let $\calF_{\text{res}}(t)$ denote the sequential pathway where $\Phi^{(l)}(t) \to \zeroVect$ while other layers are held constant. The composite block trajectory is:
\begin{equation}
\calB(t) = \calF_{\text{res}}(t) + \calF(\calI) \,.
\end{equation}

From Section~\ref{sec:layer_transition_costs}, the capacity cost for a microscopic layer transition is:
\begin{equation}
    \label{eq:app_fundamental_ode_micro}
    \calJ_{\text{capacity}}(\Phi_a, \Phi_b) = \int_{0}^{1} -c(\Phi(t)) \frac{\frac{d}{dt} E(\Phi(t))}{E(\Phi(t))} dt \,.
\end{equation}
To generalize this to macroscopic blocks, we expand the state to a composite tuple $\Omega \triangleq (T_1, \dots, T_n)$, where $T_i \in \calH_i^{N_i}$. By Proposition~\ref{prop:macro_capacity}, extending $E(\cdot)$ under constraints of symmetry, separability, and homogeneity yields:
\begin{equation}
\label{eq:app_macro_capacity}
E(\Omega) = \sum_{i=1}^n E(T_i) = \sum_{i=1}^n \sum_{j=1}^{N_i} \|f_{i,j}\|_{\calH_i} \,.
\end{equation}
For consistency, the constituent tuples of $\Omega$ must interact strictly additively within $\calB$. Grouping sequentially composed layers (e.g., $\Phi^{(1)}$ and $\Phi^{(2)}$) into the same state creates a contradiction: projecting $\Phi^{(1)} \to \zeroVect$ collapses the entire sequential pathway $\calF_{\text{res}} \to \zeroVect$, leaving $\calB = \calF(\calI)$ with a true capacity of $E(\calI)$. However, $E(\Omega)$ would erroneously evaluate to $E(\Phi^{(2)}) + E(\calI)$. To avoid this, we must restrict the macroscopic state strictly to additive components: a single internal layer $l$ and the skip connection $\calI$. 

Furthermore, because the mapping $\Phi^{(l)}$ and the skip connection $\calI$ operate in different dimensional spaces ($d_{\text{bottleneck}}$ versus $d_{\text{amb}}$), direct operator addition $\Phi^{(l)} + \calI$ is mathematically undefined. We bypass this by defining the state as the composite pair $\Omega^{(l)} \triangleq (\Phi^{(l)}, \calI)$. The capacity cost for this macroscopic state is:  
\begin{equation}
    \label{eq:app_fundamental_ode_macro}
    \calJ_{\text{capacity}}(\Omega_a^{(l)}, \Omega_b^{(l)}) = \int_{0}^{1} -c(\Omega^{(l)}(t)) \frac{\frac{d}{dt} E(\Phi^{(l)}(t)) + \frac{d}{dt} E(\calI)}{E(\Phi^{(l)}(t)) + E(\calI)} dt \,.
\end{equation}
To mitigate layer-width bias, $c(\cdot)$ acts as a counting measure of compressible operators. Being additive over disjoint sets, $c(\Omega^{(l)}) = c(\Phi^{(l)}) + c(\calI)$. Since the skip connection $\calI$ is fixed, it has no compressible operators $c(\calI) = 0$. Thus, $c(\Omega^{(l)}(t)) = c(\Phi^{(l)}(t))$. Combined with the stationarity of the skip connection $\frac{d}{dt} E(\calI) = 0$, the integral simplifies to:
\begin{equation}
    \calJ_{\text{capacity}}(\Omega_a^{(l)}, \Omega_b^{(l)}) = \int_{0}^{1} -c(\Phi^{(l)}(t)) \frac{\frac{d}{dt} E(\Phi^{(l)}(t))}{E(\Phi^{(l)}(t)) + E(\calI)} dt \,.
\end{equation}
During the projection $\Phi_a^{(l)} \to \zeroVect$, $c(\Phi^{(l)}(t))$ drops from $N_{\text{active}}^{(l)}$ to $0$. Bounding this density by its maximum pre-action value $N_{\text{active}}^{(l)}$ allows a closed-form integration. Defining $E(t) \triangleq E(\Phi^{(l)}(t))$, $\dot{E}(t) \triangleq \frac{d}{dt}E(t)$, and $E_{\text{identity}} \triangleq E(\calI)$:
\begin{equation}
    \calJ_{\text{capacity}}(\Omega_a^{(l)}, \Omega_b^{(l)}) \leq N_{\text{active}}^{(l)} \int_{0}^{1} \frac{-\dot{E}(t)}{E(t) + E_{\text{identity}}} \, dt = N_{\text{active}}^{(l)} \ln\left( 1 + \frac{E_{\text{active}}^{(l)}}{E_{\text{identity}}} \right) \,.
\end{equation}
This logarithmic formulation penalizes capacity reduction continuously. However, compression executes finite capacity removals $\Delta E = E_{\text{active}}^{(l)}$ in discrete leaps. Evaluating macroscopic evictions logarithmically while governing granular micro-actions (Section~\ref{sec:layer_transition_costs}) via linear ratios creates an inconsistent optimization hierarchy. Driven by this logarithmic discount, a greedy optimizer would view massive architectural deletions as artificially cheap. To align the macro-eviction cost with the linear micro-transition cost and ensure optimization stability, we apply the standard inequality $\ln(1+x) \le x$ for $x \ge 0$. Substituting $x = E_{\text{active}}^{(l)} / E_{\text{identity}}$ establishes a strict linear upper bound:
\begin{equation}
    \calJ_{\text{layer}}(\Omega_a^{(l)}, \Omega_b^{(l)}) \triangleq N_{\text{active}}^{(l)} \frac{E_{\text{active}}^{(l)}}{E_{\text{identity}}} \,.
\end{equation}

\subsubsection{Parallel Survival Capacity $E_{\text{identity}}$}
Flattening the tensor $X$ into a vector $\xVect \in \R^{d_{\text{amb}}}$, we decompose it into $d_{\text{amb}}$ independent identity operators $\calI_k(\xVect) = \langle \eVect_k, \xVect \rangle \eVect_k = x_k \eVect_k$. Over the input distribution $P_{\calX}$, the Hilbert-Schmidt capacity evaluates to the Root Mean Square (RMS) energy. Given $\|\eVect_k\|_2 = 1$:
\begin{equation}
    \|\calI_k\|_{\calH} = \left( \E_{\xVect \sim P_{\calX}} [ \|x_k \eVect_k\|_2^2 ] \right)^{1/2} = \sqrt{ \E_{\xVect \sim P_{\calX}} [ x_k^2 ] } \,.
\end{equation}
Assuming $\xVect$ is conditioned by a preceding BN layer with scale $\gamma_k$ and shift $\beta_k$, the expected energy evaluates to $\E[x_k^2] = \Var(x_k) + (\E[x_k])^2 = \gamma_k^2 + \beta_k^2$. The aggregate capacity of the skip connection is therefore:
\begin{equation}
    E_{\text{identity}} = \sum_{k=1}^{d_{\text{amb}}} \sqrt{\gamma_k^2 + \beta_k^2} \,.
\end{equation}
For a stably normalized network $\gamma_k \approx 1, \beta_k \approx 0$, this naturally evaluates to the ambient physical dimension: $E_{\text{identity}} \approx d_{\text{amb}}$.

\subsubsection{Resolving the Extinction Divergence}
Block eviction projects a layer $l$ to $\zeroVect$. Since PH-1 activations satisfy $\Psi(\zeroVect) = \zeroVect$, this ensures the entire sequential residual pathway vanishes $\calF_{\text{res}} \to \zeroVect$, reducing the composite block mapping to $\calB = \calF(\calI)$. Axiom 2 (Connectivity Preservation) requires the transition cost to diverge $\calJ \to \infty$ as the state capacity approaches zero, penalizing graph disconnections. 

If evaluated solely on the isolated layer $\Phi^{(l)}$, the terminal capacity would evaluate to zero, yielding an infinite penalty that would permanently prevent the optimizer from selecting an eviction operation. However, because the post-eviction mapping of the block is $\calF(\calI)$, the network retains a minimum capacity $E_{\text{identity}} > 0$. Evaluating the integral over the macroscopic state $\Omega^{(l)} = (\Phi^{(l)}, \calI)$ naturally tracks this non-vanishing capacity. Since the differential change is driven entirely by the active layer $dE(\Omega) = dE(\Phi^{(l)})$, we have:
\begin{equation}
    \calJ_{\text{layer}}^{(l)} = \int_{E_{\text{active}}^{(l)}}^{0} \frac{-N_{\text{active}}^{(l)} \, dE}{E + E_{\text{identity}}} \,.
\end{equation}
As $E \to 0$, the denominator is bounded by $\lim_{E \to 0} E(\Omega) = E_{\text{identity}} > 0$. The extinction divergence is thus mitigated.

\subsection{Generalization to Non-Residual Architectures}
While $E_{\text{identity}}$ structurally resolves the division-by-zero limit for residual networks, architectures without additive skip connections lack this architectural anchor. Examples include sequential VGG-style blocks, Inception modules, DenseNets, or intermediate expansion layers inside Transformer FFNs (prior to the global residual addition). In these architectures, evicting a block leaves no parallel identity mapping. This means the surviving capacity reaches zero.

To evaluate block eviction in these domains without encountering an infinite penalty $\calJ \to \infty$, we replace the dynamic composite denominator with a static historical constant: the block's initial, pre-pruning capacity $E_{\text{init}}^{(l)}$. The individual layer-wise cost becomes:
\begin{equation}
    \calJ_{\text{layer}}(\Omega_a^{(l)}, \Omega_b^{(l)}) = N_{\text{active}}^{(l)} \frac{E_{\text{active}}^{(l)}}{E_{\text{init}}^{(l)}}
\end{equation}
For a block containing multiple internal layers, the total macroscopic eviction cost is the linear sum of these bounds: $\calJ_{\text{evict}} = \sum_l \calJ_{\text{layer}}^{(l)}$. 

From the perspective of pure functional analysis, $E_{\text{init}}$ acts as a heuristic. It violates the Markov property of continuous state transitions (because the system must maintain memory of a pristine architectural state that no longer exists) and departs from Axiom 2 by freezing the denominator rather than tracking an architectural asymptote. However, it serves as an optimal proxy for three reasons:

\begin{itemize}
    \item \textbf{Scale Invariance:} Using an arbitrary absolute constant $C$ would yield $\calJ \propto E_{\text{active}}^{(l)}/C$, violating scale neutrality. The ratio $E_{\text{active}}^{(l)}/E_{\text{init}}^{(l)} \in [0, 1]$ normalizes the capacity reduction against the layer's own baseline. This ensures blocks of varying widths are penalized fairly based on the relative fraction of capacity destroyed.
    \item \textbf{Avoiding Normalization Collapse:} Normalizing by the current active capacity $E_{\text{active}}^{(l)}$ would collapse the cost to $\calJ_{\text{layer}} = N_{\text{active}}^{(l)}$. This artificially discounts wide, pristine layers, driving the greedy optimizer to evict them prematurely instead of targeting heavily pruned, redundant layers.
    \item \textbf{Dynamic Cost Decay:} The progressive encoder selects actions that minimize the Distortion Rate $\text{DR} = \calJ / \Delta P$. Under a mean-field assumption, active capacity scales linearly with surviving width $E_{\text{active}} \propto N_{\text{active}}$. By locking the denominator to $E_{\text{init}}$, the structural cost decays quadratically as the layer is pruned $\calJ \propto N_{\text{active}} E_{\text{active}} \propto N_{\text{active}}^2$. Because the physical parameter yield also scales linearly $\Delta P \propto N_{\text{active}}$, the overall distortion rate decreases linearly: $\text{DR} \propto N_{\text{active}}$. This guarantees that the relative cost of evicting a block naturally decreases as granular micro-actions deplete it.
\end{itemize}

Thus, while $E_{\text{identity}}$ is the rigorously derived anchor for residual networks, $E_{\text{init}}$ provides the necessary constraints to enable automated, progressive eviction in non-residual architectures.

\section{Reproducibility Protocols for Cross-Domain Transfer}
\label{app:transfer_details}

This section details the experimental setup, algorithmic formulations, and hyperparameter grids necessary to reproduce the cross-domain transfer results presented in Section~\ref{sec:deft_transfer}.

\subsection{Task Construction and Data Partitioning}
Tasks are dynamically constructed from the datasets. We perform $4$ independent trials, with each trial comprised of $5$ tasks to yield $20$ distinct cross-domain transfer scenarios.
\begin{itemize}
    \item \textbf{Source Tasks (CIFAR-100):} To ensure the source task requires learning cohesive structural features, the task comprises $20$ classes sampled by selecting $4$ random superclasses and utilizing all $5$ of their constituent fine classes.
    \item \textbf{Target Tasks (SVHN):} The target tasks consist of $10$-class classification utilizing all SVHN digits.
\end{itemize}
Both CIFAR-100 and SVHN datasets are partitioned into an $80\%$ training split and a $20\%$ validation split. The original test sets are preserved entirely for final evaluation. Prior to training, images are cast to floating-point tensors and scaled to $[0,1]$ by dividing by $255$. No data augmentation is applied during training.

\subsection{Network Architecture}
All experiments utilize a VGG-style 8-layer sequential baseline, modernized with BN. To prevent the final classification head from dominating the network's parameter footprint, we replace the classic flattened Dense layers with Global Average Pooling (GAP). This ensures that the model's overall capacity, and consequently $\Delta P$ evaluated by our compression algorithm, remains localized to the convolutional filters rather than spatial dense transitions. This prevents artificially skewed Distortion Rates (DR). The network consists of the following blocks:
\begin{enumerate}
    \item \textbf{Block 1:} Conv2D (128 filters, 3x3) $\to$ BN $\to$ ReLU $\to$ MaxPool (2x2)
    \item \textbf{Block 2:} Conv2D (256 filters, 3x3) $\to$ BN $\to$ ReLU $\to$ MaxPool (2x2)
    \item \textbf{Block 3:} Conv2D (512 filters, 3x3) $\to$ BN $\to$ ReLU $\to$ Conv2D (512 filters, 3x3) $\to$ BN $\to$ ReLU $\to$ MaxPool (2x2)
    \item \textbf{Block 4:} Conv2D (512 filters, 3x3) $\to$ BN $\to$ ReLU $\to$ Conv2D (512 filters, 3x3) $\to$ BN $\to$ ReLU $\to$ MaxPool (2x2)
    \item \textbf{Transition:} GlobalAveragePooling2D
    \item \textbf{Bottleneck:} Dense (512 neurons) $\to$ BN $\to$ ReLU
    \item \textbf{Classification Head:} Dense (10 or 20 neurons) for multi-class logit output.
\end{enumerate}
Note that BN layers are positioned \textit{between} the affine transformations (Conv2D/Dense) and the ReLU non-linearities. To prevent redundancy and ensure capacity accounting, biases in the Conv2D and Dense layers preceding a BN layer are disabled. 

\subsection{Base Training Regimen}
Both pre-training and fine-tuning utilize SGD with standard heavy-ball momentum of $0.9$, a fixed batch size of $16$, and a Sparse Categorical Cross-Entropy (CE) loss function. 
\begin{itemize}
    \item \textbf{Source Pre-training:} The model is trained from a random initialization for $100$ epochs. To suppress mini-batch noise and ensure the parameters in the slack space converge to zero, the learning rate follows a Cosine Decay schedule starting at $\eta = 0.05$ and decaying to a near-zero $\alpha = 0.001$. Early stopping is deliberately disabled. To actively clear out the network's peripheral slack during this phase, an $L_2$ regularization penalty of $5 \times 10^{-4}$ is applied to all multi-dimensional kernels as well as all 1D BN scale parameters $\gamma$.
    \item \textbf{Target Fine-Tuning:} All algorithms fine-tune for $30$ epochs and use a \textit{fixed} learning rate rather than a decay schedule. During target adaptation, the $L_2$ penalty is restricted to tensors with a rank greater than 1 to leave the calibrated 1D BN scale parameters immune to weight decay.
\end{itemize}

\subsection{EWC Exact Empirical Fisher Calculation}
To properly lock foundational source features for the EWC baseline, the empirical diagonal Fisher Information Matrix (FIM) is evaluated over the \textbf{source training dataset} (bypassing validation sets). 

Note that the mathematical formulation calculates the expected squared gradient of the source CE loss at the \textit{per-example} level before batch averaging:
\begin{equation}
    \text{FIM}_i = \E_{\xVect \sim \calX_{\text{source\_train}}} \left[ \left( \frac{\partial \calL_{\text{CE}}}{\partial \thetaVect_i} \right)^2 \right]
\end{equation}
Many naive EWC implementations incorrectly square the averaged mini-batch gradient, which misrepresents the true Fisher Information Matrix. By computing the per-example squared gradients directly, our evaluation avoids this theoretical pitfall. During target domain optimization, the target CE loss is augmented by the EWC penalty:
\begin{equation}
    \calL_{\text{total}} = \calL_{\text{target\_CE}} + \frac{\lambda}{2} \sum_{i \in \text{Backbone}} \text{FIM}_i (\thetaVect_i - \thetaVect_i^*)^2
\end{equation}
where $\lambda$ is a global regularization hyperparameter, and $\thetaVect_i^*$ represents the model parameters of the optimal pre-trained source. Note that the summation is restricted to the backbone parameters; the newly initialized target classification head is excluded from the penalty.

\subsection{Hyperparameter Tuning and Final Evaluation}
To ensure a fair comparison, all algorithms undergo a hyperparameter pre-sweep prior to final evaluation. Hyperparameters are selected by maximizing the \textbf{H-Score}, defined as the Harmonic Mean of the validation accuracies on the target and source tasks:
\begin{equation}
    \text{H-Score} = \frac{2 \cdot \text{Acc}_{\text{val}}^{\text{target}} \cdot \text{Acc}_{\text{val}}^{\text{source}}}{\text{Acc}_{\text{val}}^{\text{target}} + \text{Acc}_{\text{val}}^{\text{source}}}
\end{equation}
This evaluation uses the \textit{Validation Sets} of both SVHN (target) and CIFAR-100 (source) to prevent data leakage. To accelerate the pre-sweep, the target training set is artificially capped at $1,000$ samples during the tuning phase. 

Rather than re-evaluating the grid for all $20$ distinct transfer scenarios, the algorithm isolates the very first task (Task 0) as a representative sample, sweeps the grids to find the globally optimal parameters, and locks them in for the entire benchmark run. The sweep evaluates grids specific to each methodology:
\begin{itemize}
    \item \textbf{DEFT:} A grid search across the target percentile $P \in \{60, 40, 30, 20\}$, and fixed fine-tuning learning rate $\eta \in \{0.04, 0.02, 0.01, 0.005, 0.001\}$.
    \item \textbf{EWC:} A grid search across the regularization strength $\lambda \in \{0.1, 0.5, 1.0, 5.0, 10.0, 20.0, 50.0\}$ and learning rate $\eta \in \{0.01, 0.005, 0.001, 0.0005, 0.0001\}$.
    \item \textbf{PEFT:} Unfreezes all 1D tensors (including layer biases and BN affine parameters) alongside the target head, while freezing all 2D/4D kernels. Optimized via a line search for $\eta \in \{0.01, 0.005, 0.0025, 0.001, 0.0005, 0.0001\}$.
    \item \textbf{Standard Full FT:} Unfreezes the entire network architecture. Optimized via a line search for $\eta \in \{0.02, 0.01, 0.005, 0.001, 0.0005, 0.0001\}$.
    \item \textbf{Standard Head-Only FT:} Freezes the entire backbone, optimizing only the final classification head. Optimized via a line search for $\eta \in \{0.04, 0.02, 0.01, 0.005, 0.001, 0.0005\}$.
\end{itemize}

Once optimal hyperparameters are secured, models are fine-tuned on the \textit{full} SVHN training set for $30$ epochs. Performance is tracked epoch-by-epoch using the SVHN Validation Set, and the algorithm identifies the best epoch yielding peak target validation accuracy. The final reported metrics in the main text are extracted at this best epoch by running the model against the unseen \textbf{SVHN Test Set} and \textbf{CIFAR-100 Test Set}.

\textbf{Source Retention Protocol:} To evaluate source retention, the original pre-trained source classification head is temporarily grafted back onto the fine-tuned backbone. Furthermore, a universal mask is applied to the grafted source head. This zeros out connections originating from upstream ``slack'' features $E > 0$ to ensure target-adapted weights do not corrupt the source logits. \textbf{Critically, the BN moving statistics $\mu,\sigma^2$ are reverted}; they are unlocked from the target-adapted backbone and restored to their original source states to guarantee a fair evaluation.

\section{Theoretical Guarantees of DEFT}
\label{app:theoretical_justification}

Continual learning seeks a parameter update $\Delta \thetaVect$ that minimizes target risk while bounding the degradation of the source representation. Rather than analyzing non-convex loss landscapes, we abstract the network into continuous-functional operators and evaluate layer-wise distortion. Under HOPE, the capacity of neuron $i$ is its Hilbert-Schmidt norm over $P_\calX$:
\begin{equation}
\label{eq:hope_capacity}
    \|f_i\|_{\calH} = \|\wVect_{\text{out},i}\|_2 \sqrt{K(i, i)} \quad \text{where} \quad K(i, i) = \E_{\xVect \sim P_{\calX}} \left[ \Psi(y_i(\xVect))^2 \right] \,.
\end{equation}

\textbf{The Setup: Core vs. Slack.}
DEFT partitions the network into two disjoint sets based on a capacity threshold $\tau$:
\begin{itemize}
    \item \textbf{Universal Core $\|f_i\|_{\calH} > \tau$:} Highly active neurons carrying source knowledge. We freeze these $E_i=0$ to prevent forgetting.
    \item \textbf{Slack Subset $\|f_j\|_{\calH} \le \tau$:} Weak, inactive neurons. We make these fully plastic $E_j=1$ to learn the target task.
\end{itemize}
To prevent the target-driven updates of the slack subset from corrupting the core, DEFT applies a \textbf{structural mask} at initialization $t=0$, severing all connections from upstream slack neurons to downstream core neurons.

We establish the stability and plasticity of this setup through four main guarantees:

\textbf{1. Bounded Initialization Shock (Theorem \ref{thm:static_bound}).} 
Severing the slack-to-core connections at $t=0$ introduces a static error. \textbf{As visualized in Figure~\ref{fig:deft_proofs}(a)}, the severed core-directed weights $\wVect_{\text{core},j}$ of a slack neuron $j$ form a sub-vector of its total outgoing weights $\wVect_{\text{out},j}$. Thus, $\|\wVect_{\text{core},j}\|_2 \le \|\wVect_{\text{out},j}\|_2$. Because we only sever connections from slack neurons (where $\|f_j\|_{\calH} \le \tau$), the initial distortion injected into the core is strictly bounded by $|\calN_{\text{slack}}^{(l)}| \tau^{(l)}$.

\textbf{2. Dynamic Decoupling (Theorem \ref{thm:dynamic_decoupling}).} 
During training $t>0$, slack neurons drift to learn the target task. Because the structural mask severs cross-connections at initialization $\wVect_{\text{core},j}^{(0)} = \zeroVect$ and the zero elasticity prevents gradient updates $\partial \wVect_{\text{core},j} / \partial t = \zeroVect$, these weights remain zero. \textbf{As illustrated in Figure~\ref{fig:deft_proofs}(b)}, a changing signal multiplied by zero is zero; thus, target parameter drift cannot penetrate the core.

\textbf{3. Freeing Space Safely (Proposition \ref{prop:epsilon_tradeoff}).} 
Freezing the core might leave insufficient parameter space for the target task. Deep networks often fragment a feature across $M$ correlated neurons. Statically freezing them incorrectly locks redundant volume. Instead, DEFT compresses these $M$ neurons into a single rank-1 parent operator. This preserves the feature while releasing the remaining $M-1$ children to the slack subset, freeing the parameter space $\mathbb{R}^{(M-1) \times c}$ for optimization. Because the optimal parent minimizes the Hilbert-Schmidt projection error, this structural distortion is upper-bounded by the sum of the children's initial capacities: $\delta_k \le \sum_{m=1}^M \|f_m\|_{\calH^{(l)}}$. (While we structurally guarantee this capacity release, we do not formally prove gradient descent finds a global optimum in this non-convex space).

\textbf{4. Unified Cumulative Bound (Corollary \ref{cor:unified_bound}).} 
Cutting connections and merging neurons across layers introduces multiple distortions. However, because Theorem \ref{thm:dynamic_decoupling} guarantees zero interference during training, the network's total error does not compound exponentially. By the triangle inequality in $\calH^{(l)}$, the global degradation in the function space is static and bounded by the linear sum of the severed slack connections and the merge projection errors.

\begin{figure}[htbp]
\centering
% Left: Vector Geometry
\begin{subfigure}{0.48\textwidth}
    \centering
    \begin{tikzpicture}[font=\sffamily, >=Stealth, scale=0.85, transform shape]
        % Shaded Projection Background
        \fill[coreColor!5] (0,0) -- (3,0) -- (3,2) -- cycle;
        
        % Axes
        \draw[<->, thick, gray!80] (0, 3) node[above, text=black!70] {To Slack} -- (0,0) -- (4, 0) node[right, text=black!70] {To Core};
        
        % Vectors
        \draw[->, ultra thick, coreColor, drop shadow={opacity=0.2}] (0,0) -- (3, 0) node[midway, below] {Cut Component};
        \draw[->, ultra thick, slackColor, drop shadow={opacity=0.2}] (0,0) -- (0, 2) node[midway, left, align=right] {Kept \\ Component};
        
        % Total Vector
        \draw[->, ultra thick, black!70, drop shadow={opacity=0.3}] (0,0) -- (3, 2) node[midway, sloped, above] {Total $\mathbf{w}_{\text{out}}$};
        
        % Dashed projections
        \draw[dashed, thick, gray!60] (3,0) -- (3,2);
        \draw[dashed, thick, gray!60] (0,2) -- (3,2);
        
        % Right angle
        \draw[gray!80, thick] (0.3, 0) -- (0.3, 0.3) -- (0, 0.3);
    \end{tikzpicture}
    \caption{Vector Geometry (Theorem \ref{thm:static_bound})}
\end{subfigure}\hfill
% Right: Dynamic Signal Blocking
\begin{subfigure}{0.48\textwidth}
    \centering
    \begin{tikzpicture}[font=\sffamily, >=Stealth, scale=0.6, transform shape]
        % Nodes with gradients and shadows
        \node[draw=black, thick, rounded corners, fill=slackColor!20, drop shadow={opacity=0.3, shadow xshift=2pt, shadow yshift=-2pt}, minimum height=3cm, text width=2.5cm, align=center] (drift) at (0,0) {Preceding Layer's\\Slack\\Drift Signal\\($\Delta a$)\\[1ex]\textit{\small(Highly chaotic)}};
        
        \node[draw=danger, ultra thick, circle, fill=white, drop shadow={opacity=0.2}, inner sep=5pt] (mult) at (4.5,0) {\Large $\times$};
        \node[above=0.2cm of mult, font=\bfseries, danger, align=center] {Cut Connection \\ $\color{black}\boldsymbol{w}_{\text{core}} = \boldsymbol{0}$};
        
        \node[draw=black, thick, rounded corners, fill=coreColor!20, drop shadow={opacity=0.3, shadow xshift=2pt, shadow yshift=-2pt}, minimum height=3cm, text width=2.5cm, align=center] (core) at (9,0) {Current Layer's\\Core Update\\[1ex]($\Delta \mathbf{s} = \mathbf{0}$)};

        % Paths
        \draw[->, ultra thick, decorate, decoration={snake, amplitude=1.2mm, segment length=3mm, post length=3mm}, slackColor] (drift) -- (mult) node[midway, above=3mm, font=\bfseries] {$\color{black}\Delta \neq 0$};
        \draw[->, ultra thick, coreColor] (mult) -- (core) node[midway, above, font=\bfseries] {$\color{black}0$};
    \end{tikzpicture}
    \caption{Dynamic Signal Blocking (Theorem \ref{thm:dynamic_decoupling})}
\end{subfigure}
\caption{Visual intuition for the theoretical bounds of DEFT. (a) Vector decomposition demonstrates the initialization shock is bounded by the slack capacities. (b) A changing signal multiplied by the zeroed structural mask ensures dynamic decoupling.}
\label{fig:deft_proofs}
\end{figure}
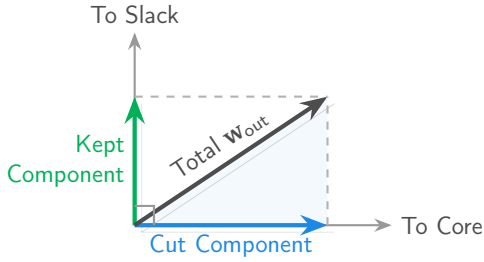
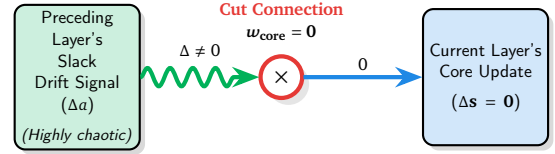

\subsection{Algorithmic Axioms and Partitioning of Neurons}

Here we formally define how DEFT partitions neurons into the peripheral slack and universal core. We also present a few axioms that merely mirror the mechanisms designated in the DEFT algorithm itself. These \textit{axioms do not impose any extra assumptions or restrictions on DEFT beyond the entirety of the algorithm itself}.

\begin{definition}[Neurons Partition]
\label{def:subset_partition}
For any intermediate layer $l$ containing a set of active neurons $\calI^{(l)}$, a continuous capacity threshold $\tau^{(l)} > 0$ divides the layer into two disjoint functional subsets:
\begin{align}
    \calN^{(l)}_{\text{core}}   &= \{i \in \calI^{(l)} \mid \|f_i\|_{\calH^{(l)}} > \tau^{(l)}\} \\
    \calN^{(l)}_{\text{slack}}  &= \{j \in \calI^{(l)} \mid \|f_j\|_{\calH^{(l)}} \le \tau^{(l)}\}
\end{align}
\end{definition}

\begin{axiom}[The Algorithmic Structural Mask]
\label{ax:structural_mask}
At initialization $t=0$, DEFT prevents cross-subset interference by severing the projections from the upstream slack subset to the downstream core. For any weight $w_{u,j}$ from an upstream slack neuron $j \in \calN^{(l)}_{\text{slack}}$ to a downstream core neuron $u \in \calN^{(l+1)}_{\text{core}}$, the mask $\Mmat$ enforces $w_{u,j} = 0$.
\end{axiom}

\begin{axiom}[The Gradient Elasticity]
\label{ax:gradient_scaling}
During target optimization $t>0$, the parameter update for any parameter vector $\thetaVect_k$ belonging to a destination neuron $k$ is governed by:
\begin{equation}
    \frac{\partial \thetaVect_k}{\partial t} = E_k \cdot \nabla_{\thetaVect_k} \calL_{\text{target}}
\end{equation}
where $E_k \in \{0, 1\}$ is the elasticity multiplier. Slack neurons are fully plastic $E_j = 1 \, \forall j \in \calN_{\text{slack}}$ and core neurons are strictly frozen $E_u = 0 \, \forall u \in \calN_{\text{core}}$.
\end{axiom}

\begin{axiom}[The Merged-Vessel Release]
\label{ax:merge_constraint}
When the continuous encoder collapses redundant neurons into a rank-1 parent operator, DEFT assigns full plasticity $E_j = 1$ to the structurally released child neurons.
\end{axiom}

\subsection{Layer-to-Layer Bounding Framework}
We now establish a rigorous bound on the structural distortion, decoupled entirely from target parameter drift.

\begin{assumption}[Bipartite Separability]
\label{assump:topology}
Let the pre-activation $y_u$ of a downstream core neuron $u \in \calN^{(l+1)}_{\text{core}}$ be a linear combination of upstream activations $g_k(\xVect) = \Psi(y_k(\xVect))$ for $k \in \calI^{(l)}$, governed by parameterized weights $w_{u,k}$ and unparameterized identity connections denoted by an adjacency indicator $\mathbb{I}_{u,k} \in \{0, 1\}$. We assume that for all upstream slack neurons $j \in \calN^{(l)}_{\text{slack}}$, the network identity routing is strictly zero:
\begin{equation}
    \mathbb{I}_{u,j} = 0 \quad \forall j \in \calN^{(l)}_{\text{slack}}, \forall u \in \calN^{(l+1)}_{\text{core}}
\end{equation}
This ensures that the mask strictly dictates inter-subset signal flow.
\end{assumption}

\paragraph*{Remark on Feed-Forward Architectures:} For purely sequential architectures (e.g., standard MLPs or VGG), $\mathbb{I}_{u,k} \equiv 0$ globally. This assumption only becomes non-trivial in architectures with unparameterized residual connections (e.g., ResNets), where an identity mapping could bypass the mask if the architecture cross-wired the subsets.

\begin{theorem}[The Static Initialization Shock Bound]
\label{thm:static_bound}
Let $g_k(\xVect) \triangleq \Psi(y_k(\xVect))$, and let $\sVect_{\text{core}}(\xVect) \triangleq \sum_{k \in \calI^{(l)}} \wVect_{\text{core},k} \, g_k(\xVect)$ be the input to the downstream core. The initialization mask $\Mmat$ introduces a distortion $\Delta \sVect_{\text{core}}^{\text{init}}(\xVect) \triangleq \sum_{j \in \calN_{\text{slack}}^{(l)}} \wVect_{\text{core},j} g_j(\xVect)$ bounded over $P_\calX$ by:
\begin{equation}
    \norm{\Delta \sVect_{\text{core}}^{\text{init}}}_{\calH^{(l)}} \le |\calN_{\text{slack}}^{(l)}| \, \tau^{(l)}
\end{equation}
\end{theorem}

\begin{proof}
By Assumption \ref{assump:topology}, $\sVect_{\text{core}}(\xVect)$ decomposes strictly over the parameterized upstream subsets:
\begin{equation}
    \sVect_{\text{core}}(\xVect) = \sum_{i \in \calN_{\text{core}}^{(l)}} g_i(\xVect) \, \wVect_{\text{core},i} + \sum_{j \in \calN_{\text{slack}}^{(l)}} g_j(\xVect) \, \wVect_{\text{core},j}
\end{equation}
where $\wVect_{\text{core},k}$ is the sub-vector of structural weights connecting upstream neuron $k$ to the downstream core. By Axiom \ref{ax:structural_mask}, DEFT enforces $\wVect_{\text{core},j} = \zeroVect \, \forall j \in \calN_{\text{slack}}^{(l)}$. The resulting functional distortion evaluates to the severed projections:
\begin{equation}
    \Delta \sVect_{\text{core}}^{\text{init}}(\xVect) = \sum_{j \in \calN_{\text{slack}}^{(l)}} g_j(\xVect) \, \wVect_{\text{core},j}
\end{equation}
Applying the triangle inequality in $\calH^{(l)}$ yields:
\begin{equation}
    \norm{\Delta \sVect_{\text{core}}^{\text{init}}}_{\calH^{(l)}} \le \sum_{j \in \calN_{\text{slack}}^{(l)}} \norm{\wVect_{\text{core},j}}_2 \norm{g_j}_{\calH_{\text{in}}^{(l)}}
\end{equation}
By Definition \ref{eq:hope_capacity}, $\norm{g_j}_{\calH_{\text{in}}^{(l)}} = \sqrt{K(j,j)}$. Furthermore, because $\wVect_{\text{out},j} = \wVect_{\text{core},j} \oplus \wVect_{\text{slack},j}$, its Euclidean norm bounds its sub-vectors: $\norm{\wVect_{\text{core},j}}_2 \le \norm{\wVect_{\text{out},j}}_2$. Substituting these properties gives:
\begin{equation}
    \norm{\Delta \sVect_{\text{core}}^{\text{init}}}_{\calH^{(l)}} \le \sum_{j \in \calN_{\text{slack}}^{(l)}} \norm{\wVect_{\text{out},j}}_2 \sqrt{K(j,j)} = \sum_{j \in \calN_{\text{slack}}^{(l)}} \|f_j\|_{\calH^{(l)}}
\end{equation}
By Definition \ref{def:subset_partition}, the capacity of every slack neuron satisfies $\|f_j\|_{\calH^{(l)}} \le \tau^{(l)}$. Therefore, the functional distortion is strictly bounded by:
\begin{equation}
    \norm{\Delta \sVect_{\text{core}}^{\text{init}}}_{\calH^{(l)}} \le \sum_{j \in \calN_{\text{slack}}^{(l)}} \tau^{(l)} = \tau^{(l)} \, |\calN_{\text{slack}}^{(l)}|
\end{equation}
\end{proof}

\subsubsection{Dynamic Decoupling}
The dynamic decoupling of the core emerges as a direct consequence of Axiom \ref{ax:gradient_scaling}.

\begin{theorem}[Dynamic Decoupling]
\label{thm:dynamic_decoupling}
During target fine-tuning $t>0$, the operators $f_i^{(l)}$ of all neurons within the core experience zero dynamic deviation:
\begin{equation}
\forall \, l \,,\, i \in \calN^{(l)}_{\text{core}} \,;\, \norm{f_i^{(l), (t)} - f_i^{(l), (0)}}_{\calH^{(l)}} = 0
\end{equation}
\end{theorem}

\begin{proof}
We proceed by structural induction over the network layers.

\textbf{Base Case:} The input $\xVect \sim P_\calX$ is a static anchor unmodified by optimization $\Delta \xVect = \zeroVect$.

\textbf{Inductive Step:} Assume that for layer $l$, the core functional mappings exhibit zero drift: $g_i^{(t)}(\xVect) = g_i^{(0)}(\xVect)$ for all $i \in \calN_{\text{core}}^{(l)}$. 

The dynamic deviation of the signal injected into the downstream core of layer $l+1$ at time $t$ is defined by $\Delta \sVect_{\text{core}}(\xVect) = \sVect_{\text{core}}^{(t)}(\xVect) - \sVect_{\text{core}}^{(0)}(\xVect)$. Expanding this difference over the core and slack subsets yields:
\begin{equation}
\Delta \sVect_{\text{core}}(\xVect) = \sum_{i \in \calN_{\text{core}}^{(l)}} \left( \wVect_{\text{core},i}^{(t)} g_i^{(t)}(\xVect) - \wVect_{\text{core},i}^{(0)} g_i^{(0)}(\xVect) \right) + \sum_{j \in \calN_{\text{slack}}^{(l)}} \left( \wVect_{\text{core},j}^{(t)} g_j^{(t)}(\xVect) - \wVect_{\text{core},j}^{(0)} g_j^{(0)}(\xVect) \right)
\end{equation}

For the first summation, Axiom \ref{ax:gradient_scaling} freezes downstream core weights $E_u = 0 \implies \wVect_{\text{core},i}^{(t)} = \wVect_{\text{core},i}^{(0)}$, and the inductive hypothesis guarantees $g_i^{(t)}(\xVect) = g_i^{(0)}(\xVect)$. The first summation is therefore zero. 

For the second summation, Axiom \ref{ax:structural_mask} severs cross-subset weights at initialization $\wVect_{\text{core},j}^{(0)} = \zeroVect$. Because $E_u = 0$, these structural weights receive zero gradient updates and remain zero $\wVect_{\text{core},j}^{(t)} = \zeroVect$. Thus, the second summation evaluates to zero irrespective of the slack drift $g_j^{(t)}$.

Consequently, $\Delta \sVect_{\text{core}}(\xVect) \equiv \zeroVect$. Since the downstream core parameters $\thetaVect_u \in \{\wVect_{\text{in},u}, b_u, \gamma_u, \beta_u\}$ are also frozen $E_u = 0$, the downstream core mappings experience zero drift $g_u^{(t)}(\xVect) = g_u^{(0)}(\xVect)$. By induction, the target learning is decoupled from the core representation across all layers.
\end{proof}

\paragraph*{Remark on Covariate Shift:}
Theorem \ref{thm:dynamic_decoupling} guarantees the immutability of the core operators relative to the source distribution anchor $P_\calX$. In practice, the network processes novel target data $\calD_T$ during fine-tuning. Consequently, empirical activations flowing through the core will inevitably shift due to standard covariate shift, not parameter degradation. DEFT ensures that any shift in downstream representations is driven entirely by the target data itself, free from the compounding noise of parameter degradation. This preserves foundational knowledge as an uncorrupted lens for processing new domains.

\begin{corollary}[Cumulative Masking-Induced Bound]
\label{cor:mask_bound}
Assuming a masking-based reduction without feature merging, the global structural distortion $D_{\text{global}}^{\text{mask}}$ evaluates as the cumulative sum of normalized layer-wise distortions. Normalizing the initialization shock at each layer by the downstream core capacity $E_{\text{core}}^{(l+1)} \triangleq \sum_{u \in \calN_{\text{core}}^{(l+1)}} \|f_u\|_{\calH^{(l+1)}}$, the total distortion is bounded by:
\begin{equation}
    D_{\text{global}}^{\text{mask}} \triangleq \sum_{l=1}^{L-1} \frac{\norm{\Delta \sVect_{\text{core}}^{\text{init}, (l+1)}}_{\calH^{(l)}}}{E_{\text{core}}^{(l+1)}} \le \sum_{l=1}^{L-1} \left( \frac{|\calN_{\text{slack}}^{(l)}| \tau^{(l)}}{E_{\text{core}}^{(l+1)}} \right)
\end{equation}
This confirms that source task degradation is strictly bounded by the chosen capacity thresholds, circumventing the exponential scaling issues of global Lipschitz constants.
\end{corollary}

\subsection{Dynamic Resolution of Redundancy via Bounded Trade-off}
A static capacity threshold evaluates functional volume in isolation, which ignores the \textit{redundancy trap} where networks fragment a feature across multiple correlated neurons. To compress this redundant volume and avoid incorrectly locking it, DEFT employs HOPE's \texttt{MERGE} operation.

\begin{proposition}[Bounded Distortion of Merging]
\label{prop:epsilon_tradeoff}
Consolidating $M$ correlated neurons into a rank-1 parent neuron guarantees the release of target optimization volume while bounding the distortion of the source representation.
\end{proposition}

\begin{proof}
Suppose $M$ correlated neurons $\{f_m\}_{m=1}^M$ are assigned to the core $E_m = 0$. By generating an optimal parent $f_p \in \calH^{(l)}$, DEFT collapses $M$ copies into $1$ and frees $M-1$ child vessels. By Axiom \ref{ax:merge_constraint}, DEFT assigns these children full plasticity $E_j = 1$, successfully releasing the parameter space $\mathbb{R}^{(M-1) \times c}$ for optimization.

The distortion $\delta_k$ introduced by this operation is the composite projection error between the parent $f_p$ and its children $\{f_m\}_{m=1}^M$. Because the optimal parent minimizes this distance in $\calH^{(l)}$, the error is upper-bounded by the sub-optimal projection to the null operator $f_p \equiv \zeroVect$:
\begin{equation}
    \delta_k \triangleq \min_{f_p} \left( \sum_{m=1}^M \|f_m - f_p\|_{\calH^{(l)}}^2 \right)^{\frac{1}{2}} \le \left( \sum_{m=1}^M \|f_m - \zeroVect\|_{\calH^{(l)}}^2 \right)^{\frac{1}{2}}
\end{equation}
Applying the standard $\ell_p$-norm inequality $\|\xVect\|_2 \le \|\xVect\|_1$, we bound this sub-optimal projection by the linear sum of the child capacities:
\begin{equation}
    \delta_k \le \sum_{m=1}^M \|f_m\|_{\calH^{(l)}}
\end{equation}
Thus, the algorithm safely releases massive target parameter volume while the distortion remains strictly bounded by the initial capacities.
\end{proof}

\begin{corollary}[The Unified Cumulative Bound]
\label{cor:unified_bound}
By the triangle inequality in $\calH^{(l)}$, the total perturbation injected into the core of layer $l+1$ is bounded by the sum of the missing slack projections (Theorem \ref{thm:static_bound}) and the structural projection errors from $K^{(l)}$ merge operations. Because dynamic interference evaluates to zero via structural induction (Theorem \ref{thm:dynamic_decoupling}), the global normalized functional distortion is bounded as:
\begin{equation}
    D_{\text{global}}^{\text{total}} \triangleq \sum_{l=1}^{L-1} \frac{\norm{\Delta \sVect_{\text{core}}^{(l+1)}}_{\calH^{(l)}}}{E_{\text{core}}^{(l+1)}} \le \sum_{l=1}^{L-1} \left( \frac{|\calN_{\text{slack}}^{(l)}| \tau^{(l)} + \sum_{k=1}^{K^{(l)}} \delta_{k}}{E_{\text{core}}^{(l+1)}} \right)
\end{equation}
where $\delta_k$ is the bounded projection error of the $k$-th merge operation. Therefore, DEFT releases target parameter volume while bounding cumulative source degradation to an algorithmically verifiable constant.
\end{corollary}

\section{Algorithms}

\begin{algorithm}[tb]
\caption{HOPE Progressive Encoding Loop}
\label{alg:hope_encoder}
\begin{algorithmic}[1]
\Require Pre-trained model $\calM$, Target density $\rho_{\text{target}}$
\State \textbf{Initialize} \textproc{CostManager}()

\Statex \Comment{\textbf{PHASE 1: $\calO(N^2)$ Initialization}}
\For{each layer $L \in \calM$}
    \State \textproc{InitializeCaches}($L$) \Comment{Computes initial capacities and $\calO(N^2)$ geometry}
    \State Anchor initial uncompressed capacity $E_0^{(L)}$
\EndFor

\Statex \Comment{\textbf{PHASE 2: True $\calO(1)$ Greedy Scan}}
\While{$\text{Density}(\calM) > \rho_{\text{target}}$}
    \State $a^* \gets \text{null}$, $\text{DR}_{\min} \gets \infty$
    
    \Statex \quad \Comment{\textit{Evaluate all cached actions using live residual capacities}}
    \For{each active layer $L$}
        \State $E_{\text{rem}} \gets \textproc{GetResidualCapacity}(L)$ \Comment{Dynamic tracking of layer shrinkage}
        \State $N_{\text{live}} \gets L.\text{ActiveCount}()$ \Comment{Live layer width for unbiased normalization}
        
        \For{each candidate $c \in L.\text{ActiveActions}()$} \Comment{Prunes, Merges, Evicts}
            \State $\Delta P_{\text{static}} \gets \textproc{GetStaticPayoff}(c)$ \Comment{Action-specific static payoff}
            \State $\calJ_{\text{cost}} \gets \textproc{ComputeDistortion}(c, E_{\text{rem}}, N_{\text{live}})$ \Comment{$\calO(1)$ scalar arithmetic}
            \State $\text{DR} \gets \calJ_{\text{cost}} / \Delta P_{\text{static}}$
            
            \If{$\text{DR} < \text{DR}_{\min}$}
                \State $\text{DR}_{\min} \gets \text{DR}$
                \State $a^* \gets c$
            \EndIf
        \EndFor
    \EndFor
    
    \Statex \quad \Comment{\textit{Execution and Local Synthesis}}
    \State \textproc{ExecutePhysicalReduction}($\calM$, $a^*$)
    \State \textproc{UpdateResidualCapacity}($a^*.L$, $a^*.\calJ_{\text{cost}}$)
    
    \Statex \quad \Comment{\textit{$\calO(1)$ Local Recalculation (Strict intra-layer isolation)}}
    \If{$a^*.\text{type} == \text{MERGE}$}
        \State $f_p \gets a^*.\text{parent}$
        \State $\text{Cache}.\text{UpdatePruneCapacity}(f_p)$ \Comment{Update solitary capacity for new parent}
        \For{each $f_{\text{neighbor}} \in a^*.L.\text{ActiveNeurons}()$}
            \State $\text{geo}_{\text{new}} \gets \textproc{PrecomputePairGeometry}(f_p, f_{\text{neighbor}})$
            \State $\text{Cache}.\text{Insert}((f_p, f_{\text{neighbor}}), \text{geo}_{\text{new}})$
        \EndFor
    \EndIf
    \Statex \quad \Comment{ZERO downstream recalculation required due to the static $\Delta P$ approximation}
\EndWhile

\State \Return $\calM$
\end{algorithmic}
\end{algorithm}

\begin{algorithm}[tb]
\caption{HOPE Subroutine: \textproc{UpdateResidualCapacity}}
\label{alg:update_capacity}
\begin{algorithmic}[1]
\Require Executed compression action $a^*$ (contains action type, target layer(s) $L$, and targeted neurons)
\State $\epsilon \gets 10^{-12}$ \Comment{Numerical stability floor for capacity tracking}

\If{$a^*.\text{type} == \text{PRUNE}$}
    \State $L \gets a^*.L$
    \State $f_{\text{vic}} \gets a^*.\text{victim}$
    
    \State $\Delta E \gets \text{Cache}.\text{GetPruneCapacity}(f_{\text{vic}})$ 
    \State $L.E_{\text{rem}} \gets \max(L.E_{\text{rem}} - \Delta E, \epsilon)$
    \State $L.\text{RemoveNeuron}(f_{\text{vic}})$
    
\ElsIf{$a^*.\text{type} == \text{MERGE}$}
    \State $L \gets a^*.L$
    \State $f_i, f_j \gets a^*.\text{children}$
    \State $f_p \gets a^*.\text{parent}$
    
    \Statex \quad \Comment{Net flux: capacity of the extinguished children minus the new parent}
    \State $\Delta E \gets \text{Cache}.\text{GetPruneCapacity}(f_i) + \text{Cache}.\text{GetPruneCapacity}(f_j) - \textproc{ComputeCapacity}(f_p)$
    
    \State $L.E_{\text{rem}} \gets \max(L.E_{\text{rem}} - \Delta E, \epsilon)$
    \State $L.\text{RemoveNeuron}(f_j)$ \Comment{Child $j$ is purged; Child $i$ serves as the vessel for $f_p$}

\ElsIf{$a^*.\text{type} == \text{EVICT}$}
    \Statex \quad \Comment{Eviction targets all internal reduction layers within the residual pathway}
    \For{each internal layer $L \in a^*.\text{block}$}
        \State $L.E_{\text{rem}} \gets \epsilon$ \Comment{Pathway capacity forcibly collapsed}
        \State $L.\text{ClearAllActiveNeurons}()$ \Comment{All internal operators are projected to $\zeroVect$}
    \EndFor
\EndIf

\end{algorithmic}
\end{algorithm}

\begin{algorithm}[tb]
\caption{Dispersed Elastic Fine-Tuning (DEFT)}
\label{alg:deft_pipeline}
\begin{algorithmic}[1]
\Require Pre-trained source parameters $\theta_{\text{src}}$, Source validation data $\calD_s^{\text{val}}$, Target data $\calD_t$, Percentile threshold $P$

\Statex \Comment{\textbf{PHASE 1: Capacity Evaluation \& Core/Slack Partitioning}}
\State \textbf{Merge Redundancies:} Use HOPE to compress highly correlated neurons into parent operators, freeing child vessels.
\State \textbf{Compute Capacities:} Calculate the functional capacity $\|f_i\|_{\calH}$ for all active neurons.
\State \textbf{Thresholding:} Determine global capacity threshold $\tau$ corresponding to the $P$-th percentile.
\For{each neuron $i$ in the network}
    \If{$\|f_i\|_{\calH} > \tau$ \textbf{and} neuron $i$ is active}
        \State $E_i \gets 0$ \Comment{\textbf{Universal Core:} Freeze to protect source knowledge}
    \Else
        \State $E_i \gets 1$ \Comment{\textbf{Plastic Slack:} Freed vessels and weak neurons learn target}
    \EndIf
\EndFor

\Statex \Comment{\textbf{PHASE 2: Consistency at Initialization (The Structural Mask)}}
\State Initialize target weights $\theta_0 \gets \theta_{\text{src}}$.
\For{each weight $w_{u,j}$ connecting an upstream neuron $j$ to a downstream neuron $u$}
    \If{$E_j = 1$ \textbf{and} $E_u = 0$}
        \State $M_{u,j} \gets 0$ \Comment{Sever cross-connections to prevent Slack drift from entering Core}
    \Else
        \State $M_{u,j} \gets 1$
    \EndIf
\EndFor
\State \textbf{Apply Mask:} $\theta_0 \gets \Mmat \odot \theta_0$ 

\Statex \Comment{\textbf{PHASE 3: Elastic Target Fine-Tuning (Dynamic Decoupling)}}
\State $\theta_t \gets \theta_0$, $\text{best\_h\_score} \gets -1$
\For{$\text{epoch} \in \{1, \dots, \text{FINETUNE\_EPOCHS}\}$}
    \For{mini-batch $(\xVect, \yVect) \in \calD_t^{\text{train}}$}
        \State $\gVect \gets \nabla_{\theta_t} \calL_{\text{target}}(\theta_t; \xVect, \yVect)$ \Comment{Compute raw target gradients}
        \State $\gVect_{\text{mod}} \gets \Emat \odot \gVect$ \Comment{Nullify gradients for frozen Core ($E_i=0$)}
        \State $\theta_t \gets \textproc{OptimizerStep}(\theta_t, \gVect_{\text{mod}})$ \Comment{Only plastic Slack updates}
    \EndFor
    
    \Statex \quad \textit{// Dual-Domain Evaluation (Harmonic Score Optimization)}
    \State $\text{Acc}_{\text{tgt}} \gets \textproc{Evaluate}(\theta_t, \calD_t^{\text{val}})$
    \State $\theta_{\text{src\_eval}} \gets \textproc{MaskSlackDrift}(\theta_t, \theta_{\text{src}}, \Emat)$ \Comment{Restore pristine Core / Mask out Slack drift}
    \State $\text{Acc}_{\text{src}} \gets \textproc{Evaluate}(\theta_{\text{src\_eval}}, \calD_s^{\text{val}})$
    \State $\text{H-Score} \gets \frac{2 \cdot \text{Acc}_{\text{tgt}} \cdot \text{Acc}_{\text{src}}}{\text{Acc}_{\text{tgt}} + \text{Acc}_{\text{src}}}$ 
    
    \If{$\text{H-Score} > \text{best\_h\_score}$}
        \State $\text{best\_h\_score} \gets \text{H-Score}$
        \State $\theta_{\text{best}} \gets \theta_t$
    \EndIf
\EndFor

\State \Return $\theta_{\text{best}}$
\end{algorithmic}
\end{algorithm}

\end{document}